\documentclass[preprint,nopreprintline,10pt]{elsarticle}

\usepackage{amsmath,amssymb,amsfonts,amsthm,mathtools}
\usepackage{booktabs}
\usepackage{tabularx}
\usepackage{array}
\usepackage{graphicx}
\usepackage[font={footnotesize,stretch=1.00},labelfont=bf]{caption}
\usepackage{subcaption}
\usepackage{xr-hyper}
\usepackage{hyperref}
\usepackage[T1]{fontenc}
\usepackage{lmodern}
\usepackage[expansion=false]{microtype}
\usepackage{algorithm}
\usepackage{algorithmic}
\usepackage{multirow}
\usepackage{float}
\usepackage{enumitem}
\usepackage{bm}
\usepackage{setspace}
\usepackage{etoolbox}
\usepackage[margin=1.85cm]{geometry}
\AtBeginEnvironment{tabular}{\singlespacing}
\AtBeginEnvironment{tabularx}{\singlespacing}
\AtBeginEnvironment{algorithmic}{\footnotesize\singlespacing}
\renewcommand{\arraystretch}{0.90}

\setlist{nosep,topsep=0.5pt,partopsep=0pt,itemsep=0pt,parsep=0pt}
\makeatletter
\renewcommand\section{\@startsection{section}{1}{\z@}%
  {3.5pt \@plus 1.0pt \@minus 1.0pt}%
  {1.5pt \@plus 0.5pt}%
  {\normalfont\normalsize\bfseries}}
\renewcommand\subsection{\@startsection{subsection}{2}{\z@}%
  {2.5pt \@plus 1.0pt \@minus 0.5pt}%
  {1.0pt \@plus 0.5pt}%
  {\normalfont\normalsize\itshape}}
\renewcommand\paragraph{\@startsection{paragraph}{4}{\z@}%
  {1.5pt \@plus 0.5pt \@minus 0.5pt}%
  {-0.5em}%
  {\normalfont\normalsize\bfseries}}
\makeatother
\AtBeginDocument{%
  \setlength{\abovedisplayskip}{2.5pt plus 1.0pt minus 1.0pt}%
  \setlength{\belowdisplayskip}{2.5pt plus 1.0pt minus 1.0pt}%
  \setlength{\abovedisplayshortskip}{1.5pt plus 1.0pt minus 1.0pt}%
  \setlength{\belowdisplayshortskip}{1.5pt plus 1.0pt minus 1.0pt}%
  \setlength{\parskip}{0pt}%
}
\makeatletter
\def\ps@pprintTitle{%
  \let\@oddhead\@empty
  \let\@evenhead\@empty
  \let\@oddfoot\@empty
  \let\@evenfoot\@oddfoot}
\makeatother

\newtheorem{theorem}{Theorem}

\newtheorem{proposition}{Proposition}

\newcommand{\SO}{\mathrm{SO}}
\newcommand{\OO}{\mathrm{O}}
\newcommand{\SU}{\mathrm{SU}}
\newcommand{\SE}{\mathrm{SE}}
\newcommand{\GL}{\mathrm{GL}}
\newcommand{\SL}{\mathrm{SL}}
\newcommand{\AGL}{\mathrm{AGL}}
\newcommand{\PSL}{\mathrm{PSL}}
\newcommand{\Heis}{\mathrm{Heis}}
\newcommand{\R}{\mathbb{R}}
\newcommand{\CC}{\mathbb{C}}
\newcommand{\FF}{\mathbb{F}}
\newcommand{\ZZ}{\mathbb{Z}}
\newcommand{\HH}{\mathbb{H}}

\newcommand{\tr}{\mathrm{tr}}
\newcommand{\ad}{\mathrm{ad}}
\newcommand{\Syl}{\mathrm{Syl}}
\newcommand{\Aut}{\mathrm{Aut}}
\newcommand{\Arf}{\mathrm{Arf}}
\newcommand{\Unif}{\mathrm{Unif}}
\newcommand{\KL}{D_{\mathrm{KL}}}
\newcommand{\GELU}{\mathrm{GELU}}
\newcommand{\RMSNorm}{\mathrm{RMSNorm}}

\begin{document}

\begin{frontmatter}

\title{Non-Commutative Lie-Algebraic State-Space Models: Finite-Division Variety Laws, Semidirect Cascades, and Gradient Saddle Escape}

\author[inst1,inst2]{Zeyu Jia\corref{cor1}}
\ead{jiazeyu@tju.edu.cn}

\affiliation[inst1]{organization={School of Biomedical Engineering and Technology, Tianjin Medical University},
            city={Tianjin},
            postcode={300070},
            country={China}}
\affiliation[inst2]{organization={Medical School, Tianjin University},
            city={Tianjin},
            postcode={300072},
            country={China}}

\cortext[cor1]{Corresponding author.}

\begin{abstract}
Structured state-space sequence models (SSMs) achieve linear-time inference and logarithmic-time parallel training via associative prefix scans, yet homogeneous diagonal transitions generate solvable semigroups that cannot track non-solvable group word problems or continuous 3D rigid-body manifolds. We present Non-Commutative Lie-Algebraic State-Space Models (Lie-SSMs), parameterizing recurrent updates on Lie algebra tangent spaces via the matrix exponential map. First, using an o-minimal cylindrical flag collar, we establish an unconditional single-layer finite-division criterion for contracting $\SO(3)$ and $\SU(2)$ heads, proving odd-prime Sylow abelianity, odd-$p$-core centrality, and exponent-30 commutator obstructions that exclude $S_5$ in $\SO(3)^m, \SU(2)^m,$ and $\SO(4)^m$ while identifying the sharp dimension threshold $d^*(S_5)=d^*(A_6)=5$ in $\SO(5)$. Second, we prove that Depth-2 semidirect and Schreier $2$-cocycle cascades resolve $S_5$, $2I\cong\SL_2(\FF_5)$, and $\Heis(\FF_p)$ in $\mathcal{O}(\log T)$ scan depth, whereas unnormalized contracting coupling collapses exponentially without Directional RMSNorm. Third, we characterize the $\Theta(\varepsilon^{c-1})$ Chen--Magnus--Hall--Witt commutator saddle and $1/|[G,G]|$ coset-centroid dilution, proving generic first-order symmetry breaking at $\varepsilon=0$ and teacher-forced local Polyak--\L{}ojasiewicz convergence under dense prefix supervision and factored readouts ($99.90\%/99.40\%$ on $S_5$ at $T=128$; $100.00\%$ on $D_{16}$). Across $\SE_2(3)$ Left-Jacobian odometry, Hillis--Steele/Blelloch GPU scans ($48.24\times$ forward, $20.27\times$/$10.28\times$ full-layer speedup at $T=2048$), and three wearable 3D inertial benchmarks across two cohorts (UCI HAR, UCI HAPT-PT, and 17-subject REALDISP), Lie-$\SO(3)$ models reach $95.69\%\pm0.14\%$ on HAR, $88.44\%\pm0.83\%$ episode logit-vote ($83.97\%\pm1.20\%$ central-3 window in 30-subject 5-fold CV; $87.16\%\pm0.64\%$ 9-subject split) on HAPT-PT, and $84.23\%\pm1.00\%$ on REALDISP.
\end{abstract}

\begin{keyword}
State-space models \sep Non-commutative recurrence \sep Lie groups and Lie algebras \sep Krohn-Rhodes decomposition \sep Semidirect cascades \sep Gradient saddle escape \sep Parallel prefix scan
\end{keyword}

\end{frontmatter}

\section{Introduction}
\label{sec:intro}

Sequential state-space models (SSMs) provide linear-time inference and logarithmic-time parallel training in S4~\citep{gu2021s4}, Mamba~\citep{gu2023mamba}, and Mamba-2~\citep{dao2024mamba2}, and have been deployed across pattern recognition in GrassNet~\citep{zhao2026grassnet}, DBMambaPose~\citep{wang2026dbmambapose}, EIMa~\citep{liu2026eima}, and ResiDet~\citep{sun2026residet}. Diagonal formulations such as S4~\citep{gu2021s4}, S5~\citep{smith2022simplified}, and LRU~\citep{orvieto2023resurrecting} evolve via
\begin{equation}
\label{eq:diagonal_ssm}
h_t = \Lambda_t h_{t-1} + B_t u_t, \qquad y_t = C_t h_t + D_t u_t,
\end{equation}
with diagonal $\Lambda_t \in \CC^{d \times d}$, enabling $\mathcal{O}(d)$ steps and $\mathcal{O}(\log T)$ parallel prefix scans~\citep{blelloch1990prefix}.

However, transition commutativity enforces an algebraic barrier rooted in the Krohn-Rhodes decomposition~\citep{krohn1965algebraic}, Barrington's theorem~\citep{barrington1989bounded}, and formal-language bounds for SSMs~\citep{merrill2024expressive,sarrof2024expressive,shakerinava}: diagonal transitions generate solvable semigroups and cannot track non-solvable groups such as $A_5$ (order 60) or $S_5$ (order 120)~\citep{merrill2024expressive,sarrof2024expressive,shakerinava}. Beyond discrete algebra, 3D action recognition requires geometric integration on $\SO(3)$ and $\SE(3)$~\citep{hall2015lie,hairer2006geometric}, supporting Lie-group Laplace descriptors~\citep{liao2025lie}, domain adaptation~\citep{liu2026lie}, Riemannian subspace learning~\citep{yin2023riemannian}, skeleton recognition~\citep{zhou2026neuropath,qiu2024skeleton,dai2023skeleton}, and wearable inertial activity, gait, and postural transition recognition~\citep{altun2010inertial,ngo2015inertial,reyes2016transition}. Recent non-commutative extensions use negative eigenvalues~\citep{grazzi2025unlocking}, selective dense state-space matrices for automaton simulation in SD-SSM~\citep{terzic2025sdssm}, Householder reflection products in DeltaProduct~\citep{deltaproduct}, complex delta rules in Complex KDA~\citep{siems2026complex}, and depth-vs.-width Lie-algebraic/Magnus analyses of group word problems and rotation tracking~\citep{heo2026depth}:
\begin{equation}
\label{eq:deltaproduct_intro}
H(v) = I_d - 2 \frac{v v^\top}{\|v\|_2^2}, \qquad M_t = \prod_{k=1}^{n_h} H(v_{t,k}) \in \OO(d).
\end{equation}
By the Cartan-Dieudonn\'e theorem, proper rotations in $\SO(d)$ require $2\lfloor d/2 \rfloor$ reflections; at $n_h = 1$, $\det(H(v)) = -1$ alternates orientation by $(-1)^t$, while planar Givens networks (RotRNN~\citep{rotrnn}) lack non-abelian closure without dense mixing.

We investigate Non-Commutative Lie-Algebraic State-Space Models (Lie-SSMs), parameterizing transitions on $\mathfrak{g} = \mathrm{Lie}(G)$ via the matrix exponential map:
\begin{equation}
\label{eq:lie_ssm_intro}
M_t = \exp\bigl(\Delta t_t \cdot A(u_t)\bigr) \in G, \qquad X_t = X_{t-1} M_t \quad (\text{or } h_{t,m} = \rho_{t,m} M_{t,m} h_{t-1,m}),
\end{equation}
and resolve four fundamental questions on (i) single-layer finite-division varieties under contractions $\rho_{a,h} \in [0, 1]$, (ii) higher-rank dimension and head-count laws, (iii) Depth-2 cascades with Directional RMSNorm, and (iv) commutator saddle escape from random initialization. Our contributions are:
\begin{enumerate}[leftmargin=2em,itemsep=0pt]
    \item \textbf{Unconditional Finite-Division Criterion and Variety Obstructions (Theorem~\ref{thm:master1_division})}: Using cylindrical o-minimal cell decomposition on $M \times (0, 1]^{|\Sigma|}$, we prove for contracting heads ($\SO(3)^m$ or $\SU(2)^m$, $\rho_{a,h} \in [0, 1]$) that single-layer all-word tracking holds iff $G$ divides $\SO(3)^m$ or $\SU(2)^m$ ($\iff \inf_{N \ge 1} \Delta_{\mathrm{norm}}(N) > 0$). We classify perfect targets ($A_5^r$, $(2I)^r/Z$ with $120r$ GELU width) and prove three Depth-1 obstructions: odd-prime Sylow abelianity ($p \ge 3$, ruling out $\Heis(\FF_p)$, with an explicit $p=3$ counterexample in $A_4 \le \SO(3)$ separating $ab$ from $ba$ at $N=2$ while failing central commutator/exponent relations), odd-$p$-core centrality $[u^2, x] = e$ on $O_p(G)$ (ruling out $F_{20}, G_{36}, G_{72}$), and the exponent-30 law $[u^{30}, v^{30}] = e$ in $\SO(3)^m$ (ruling out $S_5$ at $N_{\mathrm{alg}} = 120$).
    \item \textbf{Higher-Rank Dimension Thresholds and Head-Count Laws (Theorem~\ref{thm:master2_dimension_headcount})}: Via $\Pi: \SU(2) \times \SU(2) \twoheadrightarrow \SO(4)$, we prove that $\SO(4)^m$ remains obstructed on $S_5, \Heis(\FF_p), F_{20}$ at Depth 1, whereas monolithic solvable $2T \cong \SL_2(\FF_3)$ and $2O$ divide $\SU(2), \SO(4)$ ($d^* = 4, m=1$) yet require $m = \infty$ in $\SO(3)^m$. We establish $d^* = 5$ ($\SO(5)$) for $S_5, A_6$, exact dimension laws for $F_{pk} = C_p \rtimes_k C_k$ ($d_{\mathrm{div}}^* = d_{\mathrm{rep}}^*$) and $\Heis(\FF_p)$ ($2p$), faithful representation thresholds $d_{\mathrm{rep},\OO}^*, d_{\mathrm{rep},\SO}^*$ and extraspecial central-extension division bounds for irreducible Frobenius groups $\FF_p^r \rtimes_\alpha H$ ($G_{72,\mathrm{cyc}}$ at $d_{\mathrm{rep},\SO}^* = 9$ vs.\ Hessian $G_{72,\mathrm{quat}}$ at $d_{\mathrm{rep},\SO}^* = 8$ and $4 \le d_{\mathrm{div},\OO}^* \le d_{\mathrm{div},\SO}^* \le 6 < 8$ via $\Sigma(72 \times 3) \le \SU(3) \hookrightarrow \SO(6)$, vs.\ reducible $G_{18} = (C_3 \times C_3) \rtimes_{-I_2} C_2$ at $d^* = 4$), and exact head-count / $\FF_2$ Arf laws.
    \item \textbf{Semidirect/Cohomology Cascades, Directional RMSNorm, and Stiefel Polar Gating (Theorem~\ref{thm:master3_cascades_gating})}: We prove that Depth-2 cascades resolve split $G = N \rtimes_\varphi Q$ ($S_5, \Heis(\FF_p)$) and non-split $H^2_\varphi(Q, N)$ ($2I, 2T, 2O, Q_{16}$), whereas unnormalized contractions $\rho^{(1)}_{a,h} < 1$ collapse Layer 2 without headwise Directional RMSNorm $\hat{x}^{(1)}_{w,h} = x^{(1)}_{w,h}/\|x^{(1)}_{w,h}\|_2$, and Stiefel Polar orthogonalization $Q(S) = S(S^\top S)^{-1/2}$ eliminates the $c \ge 3$ simplex singularity $\mu_{\mathrm{sing}}(c) = \frac{1}{c-1} < 1$.
    \item \textbf{Commutator Saddle-Escape Theory and Multi-Cohort Empirical Validation (Theorem~\ref{thm:master4_optimization}, Section~\ref{sec:experiments})}: We prove that pair-difference gradients separating $\gamma_c(G)/\gamma_{c+1}(G)$ vanish as $\|\mathcal{P}_{\gamma_c}\nabla_\theta \mathcal{L}(w_c, w'_c; \varepsilon)\|_2 = \mathcal{O}(\varepsilon^{c-1})$ while terminal-only gradients suffer $1/|[G,G]|$ coset-centroid dilution and $O(T \rho_{\mathrm{ab}}^{T-1})$ decay, establish generic first-order symmetry breaking at $\varepsilon=0$ under dense prefix supervision plus factored readouts and teacher-forced local Polyak--\L{}ojasiewicz convergence ($\theta_{\max} < \pi$), and validate 5-seed $2 \times 2$ factorials on $S_5$ ($99.90\%/99.40\%$ at $T=128$) and $D_{16}$ ($100.00\%/99.90\%$), 6-DoF $\SE_2(3)$ Left-Jacobian odometry~\citep{barrau2017invariant}, Hillis--Steele/Blelloch GPU/CPU scans ($48.24\times$ forward, $20.06\times$ Fwd+Bwd, $20.27\times$/$10.28\times$ full-layer speedup at $T=2048$), and three wearable 3D inertial benchmarks across two cohorts (UCI HAR~\citep{anguita2013public}, UCI HAPT-PT~\citep{reyes2016transition}, and 3-IMU UCI REALDISP~\citep{banos2014dealing}).
\end{enumerate}

Section~\ref{sec:related_work} reviews related work; Section~\ref{sec:formulation} details our architectures and parallel scans (Figure~\ref{fig:arch_overview}); Section~\ref{sec:theory} states Propositions~\ref{prop:covering_vs_fibration}--\ref{prop:klein} and Theorems~\ref{thm:master1_division}--\ref{thm:master4_optimization} (proofs in Supplementary Material, \ref{app:proof_master1}--\ref{app:proof_master4}); Section~\ref{sec:experiments} reports experiments; and Sections~\ref{sec:discussion}--\ref{sec:conclusion} discuss limitations and conclude.

\section{Related Work}
\label{sec:related_work}

\subsection{Structured and Diagonal State-Space Models}
Linear SSMs evolve $x(t) \in \CC^d$ via $\dot{x}(t) = A x(t) + B u(t)$~\citep{gu2021s4}, motivating diagonal $A$ in S4~\citep{gu2021s4}, MIMO S5~\citep{smith2022simplified}, and complex-diagonal LRUs~\citep{orvieto2023resurrecting}. Selective SSMs such as Mamba~\citep{gu2023mamba} and Mamba-2~\citep{dao2024mamba2} introduced input-dependent step sizes $\Delta t_t$ and decay gates $\rho_t \in (0, 1]$, extended to graphs~\citep{zhao2026grassnet} and 3D point clouds~\citep{sun2026residet}; however, their transition matrices remain diagonal and commutative.

\subsection{Algebraic Expressivity, Krohn-Rhodes Theory, and Definable Automata}
Barrington's theorem~\citep{barrington1989bounded} characterizes $\mathrm{NC}^1$ via $A_5$, the Krohn-Rhodes theorem~\citep{krohn1965algebraic} decomposes finite transformation monoids into wreath cascades of simple dividing groups and $U_2 = \{\mathrm{id}, \mathrm{set}_0, \mathrm{set}_1\}$, and Gromov's theorem~\citep{gromov1981groups} links polynomial growth to virtually nilpotent groups. Diagonal SSMs and fixed-precision transformers are restricted to $\mathrm{TC}^0$ and solvable monoids~\citep{merrill2024expressive,sarrof2024expressive,shakerinava}. Prior topological/quantum group automata bounds~\citep{jeandel2005topological,brodsky} assumed unnormalized output margins, which fail under contractions $\rho_{a,h} \in [0, 1]$. Theorem~\ref{thm:master1_division} resolves this unconditionally via o-minimal cell decomposition~\citep{vandendries1998tame,chambert2025ominimality}, o-minimal neural classification~\citep{kratsios2026every,kinn2026fast}, abstract 3D rotation recognition~\citep{mcenerney2025abstract}, real semialgebraic geometry~\citep{bochnak1998real}, and Gelander's finite-supplement theorem~\citep{gelander}, while our subspace gating bounds generalize the Welch bound~\citep{welch1974lower}.

\subsection{Orthogonal, Unitary, and Lie-Group Recurrent Networks}
Unitary/orthogonal RNNs~\citep{arjovsky2016unitary} stabilize BPTT gradients via Cayley transforms, Riemannian exponential maps~\citep{lezcano2019cheap}, and Lie-algebra convolutions~\citep{dehmamy2021automatic}, alongside Lie-group Laplace descriptors~\citep{liao2025lie}, domain adaptation~\citep{liu2026lie}, and Riemannian subspace learning~\citep{yin2023riemannian}. We resolve three open questions: characterizing the finite-group variety of contracting $\SO(d)$ and $\SU(2)$ heads, diagnosing Magnus~\citep{magnus1954exponential} and Hall--Witt~\citep{hall1959theory,serre2006lie} commutator saddles, and coupling Lie transitions with $\mathcal{O}(\log T)$ parallel prefix scans~\citep{blelloch1990prefix}.

\subsection{Householder Reflection Products, Selective Dense Transitions, and Lie-Algebraic Depth Analyses}
Grazzi et al.~\citep{grazzi2025unlocking} unlocked parity via negative eigenvalues, Terzi\'c et al.~\citep{terzic2025sdssm} introduced Selective Dense State-Space Models (SD-SSM) with parallel scans for automaton simulation, DeltaProduct~\citep{deltaproduct} chained $n_h$ generalized Householder updates $\prod_{k=1}^{n_h} (I_d - \beta_{t,k} v_{t,k} v_{t,k}^\top / \|v_{t,k}\|_2^2)$ ($\beta_{t,k} \in (0, 2]$) with $L \le 3\text{--}4$ layer $S_5$ bounds, Complex KDA~\citep{siems2026complex} analyzed complex delta-rule group tracking under exact finite precision, Heo et al.~\citep{heo2026depth} studied Lie-algebraic depth and Magnus expansions on $S_5$ and rotation tracking, Howe~\citep{howe2026automaton} showed that in Householder linear RNNs $h_t = A(x_t) h_{t-1} + B(x_t)$ the additive pathway $B(x_t)$ acts as a parasitic gradient attractor that stalls multiplicative state tracking, RotRNN~\citep{rotrnn} used 2D rotations, and expRNN~\citep{lezcano2019cheap} parameterized static orthogonal recurrences $h_t = \sigma(R h_{t-1} + U x_t + b)$ via the Lie exponential map. Comparing these concurrent studies clarifies our scope: (1) relative to SD-SSM~\citep{terzic2025sdssm}, which uses unconstrained dense transitions $A_t \in \R^{d \times d}$ with additive $B_t x_t$, associative scanning of non-commutative transitions is established prior art, whereas our work restricts transitions to homogeneous compact Lie manifolds ($\SO(d), \SU(2), \SE_2(3)$) via exponential maps and proves exact finite-group division-variety and saddle-escape theorems; (2) relative to Heo et al.~\citep{heo2026depth}, our distinct advances are the o-minimal cylindrical cell-decomposition proof for contracting $\rho_{a,h} \in [0, 1]$ heads and Depth-1 obstructions ($[u^2, x] = e$ on $O_p(G)$, $[u^{30}, v^{30}] = e$ in $\SO(3)^m$, $2T/2O$ spinor separation, and the $\Sigma(72 \times 3) \le \SU(3) \hookrightarrow \SO(6)$ division-vs.-representation gap on $G_{72,\mathrm{quat}}$; Theorems~\ref{thm:master1_division}--\ref{thm:master2_dimension_headcount}), the $\FF_2$ Arf and coordinate-kernel head-count laws (Theorem~\ref{thm:master2_dimension_headcount}(c)(iii)), and the Directional-RMSNorm cascade collapse dichotomy (Theorem~\ref{thm:master3_cascades_gating}(b)); and (3) relative to Howe~\citep{howe2026automaton}, whose additive-$B(x_t)$ parasitic attractor explains why inhomogeneous linear RNNs and our additive \texttt{expRNN-RadBlk} control stall on group tracking, our homogeneous multiplicative Lie-SSM ($B_t \equiv 0$, $h_{t,m} = \rho_{t,m}\exp(\Delta t_{t,m} A_m(u_t))h_{t-1,m}$) removes the additive pathway by construction and characterizes the remaining multiplicative obstructions---the $\Theta(\varepsilon^{c-1})$ Chen--Magnus--Hall--Witt commutator saddle, $1/|[G,G]|$ coset-centroid dilution, and Fourier--Markov $O(T \rho_{\mathrm{ab}}^{T-1})$ terminal decay---validated via 5-seed $2 \times 2$ factorials against fixed ($\beta=2$) and gated ($\beta \in (0,2)$) \texttt{DeltaProduct-SO(6)} ($n_h \in \{4, 6\}$) and \texttt{expRNN-RadBlk} (Theorem~\ref{thm:master4_optimization}, Table~\ref{tab:s5_random_init_closure}).

\section{Lie-Algebraic State-Space Architectures \& Associative Parallel Scan}
\label{sec:formulation}

\subsection{Single-Layer Lie-\texorpdfstring{$\SO(d)$}{SO(d)}, Spinor \texorpdfstring{$\SU(2)$}{SU(2)}, and \texorpdfstring{$\SE(3)$}{SE(3)} State-Space Models}
\label{sec:formulation_single}
Let $G \subset \GL(d, \R)$ be a matrix Lie group with Lie algebra $\mathfrak{g} = T_{I_d} G$ and bracket $[A, B] = AB - BA$. For $\SO(d) = \{R \in \R^{d \times d} : R^\top R = I_d, \det(R) = +1\}$, $\mathfrak{so}(d) = \{A \in \R^{d \times d} : A^\top = -A\}$ has dimension $d(d-1)/2$. Given inputs $(u_1, \dots, u_T)$ ($u_t \in \R^{d_{\mathrm{in}}}$), an $M$-head single-layer Lie-$\SO(d)$-SSM projects $u_t$ to skew-symmetric generators and step sizes:
\begin{equation}
\label{eq:lie_tangent_proj}
A_{t,m} = \mathcal{P}_{\mathfrak{so}(d)}\bigl(W_{u,m} u_t + b_{u,m}\bigr) \in \mathfrak{so}(d), \qquad \Delta t_{t,m} = \mathrm{softplus}\bigl(w_{\Delta,m}^\top u_t + b_{\Delta,m}\bigr) > 0,
\end{equation}
where $\mathcal{P}_{\mathfrak{so}(d)}(Z) = \frac{1}{2}(Z - Z^\top)$, and updates via the matrix exponential:
\begin{equation}
\label{eq:lie_exp_step}
R_{t,m} = \exp\bigl(\Delta t_{t,m} A_{t,m}\bigr) \in \SO(d).
\end{equation}
For $d=3$ with $\Omega_{t,m} = \Delta t_{t,m} \omega_{t,m} \in \R^3$ and hat map $\hat{\Omega} \in \mathfrak{so}(3)$ ($\hat{\Omega} v = \Omega \times v$), Rodrigues' formula gives
\begin{equation}
\label{eq:rodrigues}
\exp(\hat{\Omega}) = I_3 + \frac{\sin \theta}{\theta} \hat{\Omega} + \frac{1 - \cos \theta}{\theta^2} \hat{\Omega}^2, \qquad \theta = \|\Omega\|_2,
\end{equation}
stabilized by a 3rd-degree Taylor polynomial (4th-order even coefficients) for $\theta < \theta_{\mathrm{Taylor}} = 10^{-4}$ (Algorithm~\ref{alg:lie_ssm_forward}). States evolve on $X_{t,m} = X_{t-1,m} R_{t,m} \in \SO(d)$ or as contracting vectors $h_{t,m} = \rho_{t,m} R_{t,m} h_{t-1,m} \in \R^d$ ($\rho_{t,m} \in (0, 1], h_{0,m} \in S^{d-1}$); raw-spinor $\SU(2)$ heads evolve on $\CC^2 \simeq \R^4$ via $y_{t,m} = \rho_{t,m} U_{t,m} y_{t-1,m}$ ($U_{t,m} \in \SU(2)$). On $\SE(3) = \R^3 \rtimes \SO(3)$, twists $\xi_t = (v_t^\top, \omega_t^\top)^\top \in \R^6$ map via the Left-Jacobian exponential:
\begin{equation}
\label{eq:se3_exp}
\exp(\hat{\xi}_t) = \begin{pmatrix} \exp(\hat{\omega}_t) & V(\omega_t) v_t \\ \mathbf{0}^\top & 1 \end{pmatrix} \in \SE(3), \qquad V(\omega) = I_3 + \frac{1 - \cos \theta}{\theta^2} \hat{\omega} + \frac{\theta - \sin \theta}{\theta^3} \hat{\omega}^2.
\end{equation}

\begin{algorithm}[htbp]
\caption{Forward Evaluation of Single-Layer Multi-Head Lie-$\SO(3)$-SSM (Exact Taylor-Stabilized $\mathrm{sinc}$ Rodrigues Map Matching \texttt{so3\_exp\_map})}
\label{alg:lie_ssm_forward}
\begin{algorithmic}[1]
\REQUIRE Input sequence $(u_1, \dots, u_T)$, tangent weights $\{W_{u,m}, b_{u,m}, w_{\Delta,m}, b_{\Delta,m}\}_{m=1}^M$, initial states $\{h_{0,m} \in S^2\}_{m=1}^M$, readout network $\Phi$, Taylor threshold $\theta_{\mathrm{Taylor}} = 10^{-4}$.
\ENSURE Output logits $(y_1, \dots, y_T)$ and manifold states $(h_1, \dots, h_T) \in (S^2)^M$.
\FOR{$t = 1$ \TO $T$ \textbf{and} $m = 1$ \TO $M$}
    \STATE $\omega_{t,m} \gets W_{u,m} u_t + b_{u,m} \in \R^3$, \quad $\Delta t_{t,m} \gets \mathrm{softplus}(w_{\Delta,m}^\top u_t + b_{\Delta,m})$
    \STATE $\Omega_{t,m} \gets \Delta t_{t,m} \omega_{t,m}$, \quad $\theta_{t,m} \gets \|\Omega_{t,m}\|_2$, \quad $\hat{\Omega}_{t,m} \gets [\Omega_{t,m}]_\times \in \mathfrak{so}(3)$
    \STATE $c_{1,t,m} \gets \begin{cases} 1 - \theta_{t,m}^2 / 6 + \theta_{t,m}^4 / 120, & \theta_{t,m} < 10^{-4} \\ (\sin \theta_{t,m}) / \theta_{t,m}, & \theta_{t,m} \ge 10^{-4} \end{cases}$, \quad $c_{2,t,m} \gets \begin{cases} 1/2 - \theta_{t,m}^2 / 24 + \theta_{t,m}^4 / 720, & \theta_{t,m} < 10^{-4} \\ (1 - \cos \theta_{t,m}) / \theta_{t,m}^2, & \theta_{t,m} \ge 10^{-4} \end{cases}$
    \STATE $R_{t,m} \gets I_3 + c_{1,t,m} \hat{\Omega}_{t,m} + c_{2,t,m} \hat{\Omega}_{t,m}^2 \in \SO(3)$ \COMMENT{Exact $\exp(\hat{\Omega}_{t,m})$ with $\|R_{t,m}^\top R_{t,m} - I_3\|_F < 10^{-15}$ in \texttt{float64}}
    \STATE $h_{t,m} \gets R_{t,m} h_{t-1,m}$ \COMMENT{or $X_{t,m} \gets X_{t-1,m} R_{t,m}$ in matrix mode}
\ENDFOR
\FOR{$t = 1$ \TO $T$}
    \STATE $y_t \gets \Phi(h_{t,1}, \dots, h_{t,M})$
\ENDFOR
\RETURN $(y_1, \dots, y_T), (h_1, \dots, h_T)$
\end{algorithmic}
\end{algorithm}

\subsection{Associative Parallel Prefix Scan on Non-Commutative Matrix Lie Groups}
\label{sec:formulation_scan}
Although Lie transitions do not commute ($M_t M_{t'} \ne M_{t'} M_t$), matrix multiplication $(M_1 M_2) M_3 = M_1 (M_2 M_3)$ and affine composition $(A_2, b_2) \circ (A_1, b_1) = (A_2 A_1, A_2 b_1 + b_2)$ are associative. All prefix products $X_t = \prod_{\tau=1}^t M_\tau$ over $T = 2^K$ can therefore be evaluated via two logarithmic tree-scan schedules: (i) the \textbf{work-efficient Blelloch up-sweep/down-sweep scan}~\citep{blelloch1990prefix} (Algorithm~\ref{alg:lie_prefix_scan}), requiring $2 \log_2 T$ kernel launches, $2(T - 1) = \mathcal{O}(T)$ batched $d \times d$ matrix multiplications, and $\mathcal{O}(T d^2)$ autograd states; and (ii) the \textbf{step-optimal Hillis--Steele doubling scan}~\citep{hillis1986data}, updating $P_t^{(k+1)} \gets P_{t-2^k}^{(k)} \cdot P_t^{(k)}$ in $\log_2 T$ launches with $\mathcal{O}(T \log_2 T)$ matrix multiplications and $\mathcal{O}(T \log_2 T \, d^2)$ autograd states. As benchmarked in Section~\ref{sec:exp_hardware} (Table~\ref{tab:scan_benchmark}), Hillis--Steele minimizes launch overhead at short horizons ($T \le 256, M=1$), whereas Blelloch's $\mathcal{O}(T)$ work efficiency overtakes Hillis--Steele on CPU ($8.41\times$ vs.\ $3.51\times$ at $T=2048$) and at $T \ge 1024$ on GPU ($48.24\times$ primitive forward, $20.06\times$ scan Fwd+Bwd, and $20.27\times$ full-layer training speedup at $T=2048, M=1$ with $30.0\%$ lower full-layer peak VRAM [$144.83\text{ MB}$ vs.\ $207.02\text{ MB}$; $-32.0\%$ scan-primitive VRAM], and $10.28\times$ full-layer speedup at $M=16$ with $31.5\%$ lower full-layer peak VRAM [$2097.86\text{ MB}$ vs.\ $3061.99\text{ MB}$; $-33.7\%$ scan-primitive VRAM]).

\begin{algorithm}[htbp]
\caption{Parallel Associative Prefix Scan on Matrix Lie Groups}
\label{alg:lie_prefix_scan}
\begin{algorithmic}[1]
\REQUIRE Sequence of Lie group transition matrices $(M_1, \dots, M_T) \in G^T$ with $T = 2^K$.
\ENSURE Prefix products $X_t = \prod_{\tau=1}^t M_\tau \in G$ for all $t \in \{1, \dots, T\}$.
\STATE Allocate buffer $S[1 \dots T] \gets [M_1, \dots, M_T]$.
\FOR{$d = 0$ \TO $K - 1$}
    \FOR{$k = 0, 2^{d+1}, \dots, T - 1$ \textbf{in parallel}}
        \STATE $S[k + 2^{d+1}] \gets S[k + 2^d] \cdot S[k + 2^{d+1}]$ \COMMENT{Phase 1: Up-Sweep Reduction Tree}
    \ENDFOR
\ENDFOR
\STATE $S[T] \gets I_d$
\FOR{$d = K - 1, K - 2, \dots, 0$}
    \FOR{$k = 0, 2^{d+1}, \dots, T - 1$ \textbf{in parallel}}
        \STATE $T_{\mathrm{tmp}} \gets S[k + 2^d]$; \quad $S[k + 2^d] \gets S[k + 2^{d+1}]$
        \STATE $S[k + 2^{d+1}] \gets S[k + 2^{d+1}] \cdot T_{\mathrm{tmp}}$ \COMMENT{Phase 2: Down-Sweep (Left-to-Right Order)}
    \ENDFOR
\ENDFOR
\FOR{$t = 1$ \TO $T$ \textbf{in parallel}}
    \STATE $X_t \gets S[t] \cdot M_t$
\ENDFOR
\RETURN $(X_1, \dots, X_T)$
\end{algorithmic}
\end{algorithm}

\subsection{Depth-2 Semidirect and Schreier \texorpdfstring{$2$}{2}-Cocycle Cascades with Directional RMSNorm}
\label{sec:formulation_cascade}
Because single-layer heads receive only $u_t$, their joint transition monoid is restricted to subdirect products (Theorem~\ref{thm:master1_division}). To track semidirect products $G = N \rtimes_\varphi Q$ ($S_5 \cong A_5 \rtimes_\varphi C_2$, $\Heis(\FF_p) \cong C_p^2 \rtimes_\varphi C_p$) or non-split extensions $1 \to N \to G \to Q \to 1$ with Schreier $2$-cocycle $[\omega] \in H^2_\varphi(Q, N)$ ($2I \cong \SL_2(\FF_5), Q_{16}$) using $\SO(3)$ heads, we construct a \emph{Depth-2 Directional-RMSNorm Lie-SSM Cascade}:
\begin{enumerate}[label=(\roman*),leftmargin=2em,itemsep=0pt]
    \item \textbf{Layer 1 (Quotient $Q$ Tracker)}: Layer 1 evolves $M_1$ conformal $\SO(3)$ heads driven by $u_t$:
    \begin{equation}
    \label{eq:cascade_layer1}
    h^{(1)}_{t,m} = \rho^{(1)}_{t,m} R^{(1)}_m(u_t) h^{(1)}_{t-1,m}, \qquad R^{(1)}_m(u_t) = \exp\bigl(\Omega^{(1)}_m(u_t)\bigr) \in \SO(3), \quad \rho^{(1)}_{t,m} \in (0, 1].
    \end{equation}
    \item \textbf{Inter-Layer Directional RMSNorm Projection}: Rather than passing unnormalized $h^{(1)}_{t-1,m}$ (which decays as $(\rho^{(1)})^t \to 0$ and collapses downstream routing, Theorem~\ref{thm:master3_cascades_gating}(b)), we apply parameter-free headwise spherical normalization at every step $t$:
    \begin{equation}
    \label{eq:dir_rmsnorm}
    \hat{h}^{(1)}_{t,m} = \RMSNorm_{\mathrm{dir}}\bigl(h^{(1)}_{t,m}\bigr) := \frac{h^{(1)}_{t,m}}{\|h^{(1)}_{t,m}\|_2} \in S^2, \qquad \RMSNorm_{\mathrm{dir},\epsilon}\bigl(h^{(1)}_{t,m}\bigr) := \frac{h^{(1)}_{t,m}}{\|h^{(1)}_{t,m}\|_2 + \epsilon} \quad (\epsilon = 10^{-8}).
    \end{equation}
    In pure-rotation prefix scans or step-wise normalized recurrences where $\rho^{(1)}_{\tau,m} > 0$ is factored out by homogeneity, $\|h^{(1)}_{t,m}\|_2 = 1$ for all $t \ge 1$ with constant error $\frac{\epsilon}{1+\epsilon} \approx 10^{-8}$ across $T \to \infty$. If unrescaled contracting states $h^{(1)}_{t,m} = (\prod_{\tau=1}^t \rho^{(1)}_{\tau,m} R^{(1)}_m(u_\tau)) h^{(1)}_{0,m}$ are normalized post-scan via $\RMSNorm_{\mathrm{dir},\epsilon}$, the error $1 - \|\RMSNorm_{\mathrm{dir},\epsilon}(h^{(1)}_{t,m})\|_2 = \frac{\epsilon}{\|h^{(1)}_{t,m}\|_2 + \epsilon}$ is governed by the output-norm threshold horizon $N_{\epsilon,\mathrm{norm}}(\rho_{\min}^{(1)}, \delta) := \bigl\lfloor \frac{\log(\|h^{(1)}_{0,m}\|_2 (1-\delta) / (\epsilon \delta))}{\log(1/\rho_{\min}^{(1)})} \bigr\rfloor$ and error-tolerance horizon $N_{\epsilon,\mathrm{err}}(\rho_{\min}^{(1)}, \eta) := \bigl\lfloor \frac{\log(\|h^{(1)}_{0,m}\|_2 \eta / ((1-\eta) \epsilon))}{\log(1/\rho_{\min}^{(1)})} \bigr\rfloor = N_{\epsilon,\mathrm{norm}}(\rho_{\min}^{(1)}, 1-\eta)$: at $\rho_1 = 0.65$ and $\epsilon = 10^{-8}$, $N_{\epsilon,\mathrm{norm}}(0.65, 0.99) = N_{\epsilon,\mathrm{err}}(0.65, 0.01) = 32$ so the output norm remains $> 0.9903$ across our subsampled inertial horizon $t \le T_{\mathrm{sub}} = 32$ (and crosses $0.5$ at $N_{\epsilon,\mathrm{norm}}(0.65, 0.5) = 42$).
    \item \textbf{Layer 2 (Normal-Core $N$ Tracker with Lie-Algebra Routing)}: The normalized state $\hat{h}^{(1)}_{t-1} \in (S^2)^{M_1}$ routes the Layer-2 generator $\Omega^{(2)}_m(\hat{h}^{(1)}_{t-1}, u_t) \in \mathfrak{so}(3)$ to twist the normal-core update via $\varphi(q_{t-1})(n_{u_t}) \omega(q_{t-1}, q_{u_t})$:
    \begin{equation}
    \label{eq:cascade_layer2}
    h^{(2)}_{t,m} = \exp\Bigl(\Omega^{(2)}_m\bigl(\hat{h}^{(1)}_{t-1}, u_t\bigr)\Bigr) h^{(2)}_{t-1,m} \in S^2 \qquad (m = 1, \dots, M_2).
    \end{equation}
    In end-to-end training from random initialization (e.g., on $S_5 \cong A_5 \rtimes_\varphi C_2$), Layer 1 computes $p_{t-1} = \mathrm{Softmax}(W_{\mathrm{route}} \hat{h}^{(1)}_{t-1} / \tau_{\mathrm{route}}) \in \Delta^{|Q|-1}$ ($\tau_{\mathrm{route}} = 0.30$) and forms $\Omega^{(2)}_m(\hat{h}^{(1)}_{t-1}, u_t) = \sum_{q \in Q} p_{t-1,q} \Omega^{(2)}_{m,q}(u_t) \in \mathfrak{so}(3)$. Because Layer 1 is autonomous, both layers execute in two sequential $\mathcal{O}(\log T)$ parallel prefix scans.
\end{enumerate}

\subsection{Cross-Head Mixing \texorpdfstring{$\SO(3m)$}{SO(3m)} and Differentiable Stiefel Polar Subspace Gating}
\label{sec:formulation_stiefel}
Interleaving $m$ block-diagonal $\SO(3)$ heads with orthogonal Givens mixing $W_{\mathrm{mix}} \in \SO(3m)$ coupling adjacent boundaries $(3h, 3h+1)$ expands $\bigoplus_{h=1}^m \mathfrak{so}(3)$ to $\mathfrak{so}(3m)$. To track partial-reset monoids $M_{\mathrm{part}}(G_1, \dots, G_c) = \prod_{j=1}^c (G_j \cup \{\mathrm{reset}_g : g \in G_j\})$ on $\R^d$ ($d \ge D := \sum_{j=1}^c d_j$) without oblique leakage or the $c \ge 3$ simplex singularity $\mu_{\mathrm{sing}}(c) = \frac{1}{c-1}$, we parameterize $S \in \mathcal{M}_* := \{S \in \R^{d \times D} : \mathrm{rank}(S) = D, \sigma_{\min}(S) > 0\}$ via the Stiefel Polar projection:
\begin{equation}
\label{eq:stiefel_polar_gate}
Q(S) = S(S^\top S)^{-1/2} = [Q_1(S), \dots, Q_c(S)], \qquad \Lambda_a(S) = \Bigl(I_d - (1 - \gamma) \sum_{j \in J_{\mathrm{reset}}(a)} Q_j(S) Q_j(S)^\top\Bigr) R_a(S),
\end{equation}
where $Q(S)$ ($\epsilon_{\mathrm{pol}} = 0$) enforces $Q_j(S)^\top Q_{j'}(S) = \delta_{j,j'} I_{d_j}$ on $\mathcal{M}_*$, while the $\epsilon_{\mathrm{pol}}$-damped map $Q_{\epsilon_{\mathrm{pol}}}(S) := S(S^\top S + \epsilon_{\mathrm{pol}} I_D)^{-1/2}$ ($\epsilon_{\mathrm{pol}} = 10^{-6}$) bounds $\|Q_{\epsilon_{\mathrm{pol}}}(S)^\top Q_{\epsilon_{\mathrm{pol}}}(S) - I_D\|_{\mathrm{op}} = \frac{\epsilon_{\mathrm{pol}}}{\sigma_{\min}(S)^2 + \epsilon_{\mathrm{pol}}}$.

\subsection{Stationary Dense Prefix Supervision and Factored Subdirect Readouts}
\label{sec:formulation_dense_factored}
As proven in Theorem~\ref{thm:master4_optimization}, terminal-only cross-entropy $\mathcal{L}_{\mathrm{term}} = \mathcal{L}_{\mathrm{CE}}(\mathrm{logits}_T, y_T)$ suffers exponential Fourier--Markov decay $O(T \rho_{\mathrm{ab}}^{T-1}) \to 0$ at small random initialization, while an unstructured $|G|$-way readout suffers $1/|[G, G]|$ coset-centroid variance dilution ($1/60$ on $S_5$, $1/4$ on $D_{16}$). We overcome both via two principles:
\begin{enumerate}[label=(\roman*),leftmargin=2em,itemsep=0pt]
    \item \textbf{Stationary Autoregressive Dense Per-Step Prefix Supervision}: Supervising all prefix states $(h_1, \dots, h_T)$ with discount normalization $Z_T(\beta) := \sum_{t=1}^T \beta^{t-1}$:
    \begin{equation}
    \label{eq:dense_prefix_loss}
    \mathcal{L}_{\mathrm{dense}}^{(\beta)}(w) = \frac{1}{Z_T(\beta)} \sum_{t=1}^T \beta^{t-1} \mathcal{L}_{\mathrm{CE}}\bigl(\mathrm{logits}_t, \eta(w_{1:t})\bigr) \qquad (\beta \in (0, 1]),
    \end{equation}
    breaks first-order abelianization symmetry at $\theta = 0$ via short-prefix terms ($\Omega(1)$ gradient for $\beta \in (0, 1)$ and $\Theta(1/T)$ for $\beta = 1$).
    \item \textbf{Factored Abelianization-Plus-Core Readout}: For $G$ with quotient $\pi_{\mathrm{ab}}: G \twoheadrightarrow G^{\mathrm{ab}} = G/[G, G]$ and normal core $N = [G, G]$, each $g \in G$ corresponds bijectively to $(n(g), q(g)) \in N \times G^{\mathrm{ab}}$, and we decompose $|G|$-way logits additively into $\ell^{\mathrm{ab}}_t \in \R^{|G^{\mathrm{ab}}|}$ and a zero-sum core vector $\ell^{\mathrm{core}}_t \in \R^{|N|}$:
    \begin{equation}
    \label{eq:factored_readout}
    \mathrm{logits}_{|G|}(g)_t = \ell^{\mathrm{ab}}_t[q(g)] + \ell^{\mathrm{core}}_t[n(g)], \qquad \sum_{n \in N} \ell^{\mathrm{core}}_t[n] = 0.
    \end{equation}
    The Softmax partition function factors identically into $|G^{\mathrm{ab}}|$-way and $|N|$-way factors:
    \begin{equation}
    \label{eq:partition_factorization}
    Z_t = \sum_{g \in G} \exp\bigl(\ell^{\mathrm{ab}}_t[q(g)] + \ell^{\mathrm{core}}_t[n(g)]\bigr) = \Biggl(\sum_{q \in G^{\mathrm{ab}}} e^{\ell^{\mathrm{ab}}_t[q]}\Biggr) \Biggl(\sum_{n \in N} e^{\ell^{\mathrm{core}}_t[n]}\Biggr),
    \end{equation}
    so $\mathcal{L}_{\mathrm{CE}}(\mathrm{logits}_{|G|, t}, g_t) = \mathcal{L}_{\mathrm{CE}}(\ell^{\mathrm{ab}}_t, q(g_t)) + \mathcal{L}_{\mathrm{CE}}(\ell^{\mathrm{core}}_t, n(g_t))$ delivers an undiluted $\Theta(1)$ abelianization gradient. On $S_5 = A_5 \rtimes_\varphi C_2$, linear projections suffice for $\ell^{\mathrm{ab}}_t \in \R^2$ and $\ell^{\mathrm{core}}_t \in \R^{60}$ ($62$ vs.\ $120$ logits). On $D_{16} \le \SO(3)$, even and odd rotations $C_4, rC_4$ lie on coplanar squares sharing centroid $(0,0,z_0)^\top$, obstructing linear $V_4$ readouts by Radon's theorem (Theorem~\ref{thm:master1_division}(c)(iii)); parameterizing $\ell^{\mathrm{ab}}_t \in \R^4$ and zero-sum $\ell^{\mathrm{core}}_t \in \R^4$ as 2-layer 64-unit GELU MLPs ($8$ logits, $8{,}632$ vs.\ $10{,}176$ params) resolves the Radon parity floor while enforcing Eq.~\eqref{eq:partition_factorization}.
\end{enumerate}

\begin{figure*}[t!]
\centering
\includegraphics[width=0.96\textwidth]{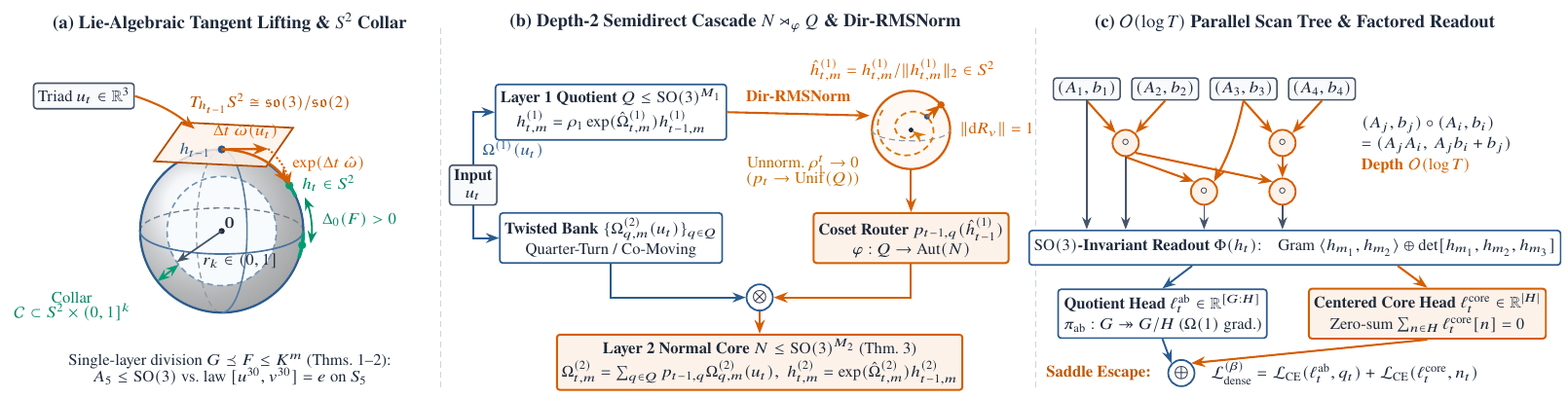}
\caption{Unified architectural and geometric overview of Non-Commutative Lie-Algebraic State-Space Models (Lie-SSMs). \textbf{(a)} Lie-algebraic tangent-space lifting ($T_{h_{t-1}}S^2 \cong \mathfrak{so}(3)/\mathfrak{so}(2)$, via the rank-$2$ orbital map $\hat{\omega} \mapsto \hat{\omega} h_{t-1}$ with axial stabilizer $\mathfrak{so}(2)$) of input triads $u_t \in \R^3$, geodesic exponential map $\exp(\Delta t\,\hat{\omega}(u_t))$ onto $S^2$, o-minimal cylindrical flag collar $\mathcal{C} \subset S^2 \times (0, 1]^k$ with spherical cap separation $\Delta_0(F) > 0$, and the single-layer finite-division variety criterion $G \preceq F \le K^m$ contrasting perfect $A_5 \le \SO(3)$ against $[u^{30}, v^{30}] = e$ on $S_5$ (Theorems~\ref{thm:master1_division}--\ref{thm:master2_dimension_headcount}). \textbf{(b)} Depth-2 semidirect and Schreier $2$-cocycle cascade ($G \cong N \rtimes_\varphi Q$), contrasting unnormalized spiral collapse ($\|h^{(1)}_t\|_2 = \rho_1^t \to 0$) against parameter-free headwise Directional RMSNorm ($\hat{h}^{(1)}_{t,m} = h^{(1)}_{t,m} / \|h^{(1)}_{t,m}\|_2 \in S^2$, Theorem~\ref{thm:master3_cascades_gating}). \textbf{(c)} $\mathcal{O}(\log T)$ binary associative prefix scan tree over affine-group pairs $(A_j, b_j) \circ (A_i, b_i) = (A_j A_i, A_j b_i + b_j)$, $\SO(3)$-invariant Gram/chirality readout features $\Phi(h_t)$, and $[G:H] \times |H|$ zero-sum coset-centered factored readout ($\ell_t^{\mathrm{ab}} \oplus \ell_t^{\mathrm{core}}$) with dense prefix supervision $\mathcal{L}_{\mathrm{dense}}^{(\beta)}$ escaping the $\Theta(\varepsilon^{c-1})$ commutator saddle (Theorem~\ref{thm:master4_optimization}).}
\label{fig:arch_overview}
\end{figure*}

\section{Unified Theoretical Foundations: Variety Laws, Cascades, and Optimization Dynamics}
\label{sec:theory}

\subsection{Topological and Representation Foundations}
\label{sec:theory_foundations}

\begin{proposition}[Universal Covering vs.\ Hopf Fibration, Commutative Obstruction, and Dyck-1 Precision Horizon]
\label{prop:covering_vs_fibration}
\leavevmode
\begin{enumerate}[label=(\alph*),leftmargin=2em,itemsep=0pt]
    \item \textbf{Covering vs.\ Fibration}: $\SU(2) \cong S^3 \subset \R^4$, and adjoint conjugation $\pi(U) x = U x U^\dagger$ on $\mathrm{Im}(\HH) \cong \R^3$ is a $2$-to-$1$ universal covering homomorphism onto $\SO(3) \cong S^3 / \{\pm 1\} \cong \R P^3$ ($\ker(\pi) = \{\pm I_2\}$, $\dim \SO(3) = 3$, $\pi_1(\SO(3)) \cong \ZZ_2$), topologically distinct from the Hopf fibration $S^1 \hookrightarrow S^3 \xrightarrow{p} S^2$ ($\dim S^2 = 2$, $\pi_1(S^2) = 0$).
    \item \textbf{Single-Layer Commutative Transition Obstruction}: In any recurrence $h_t = \Lambda_t h_{t-1} + B u_t$ with commuting transitions $\Lambda_t \Lambda_{t'} = \Lambda_{t'} \Lambda_t$, $\Lambda_{1:T} = \prod_{t=1}^T \Lambda_{\pi(t)}$ is permutation-invariant ($\pi \in S_T$) and cannot separate $g_1 g_2 \ne g_2 g_1$ through its transition operator for any $d$.
    \item \textbf{Dyck-1 Prefix Non-Negativity and Floating-Point Horizon}: A string $w \in \{+1, -1\}^T$ belongs to Dyck-1 iff $c_t = \sum_{\tau=1}^t x_\tau \ge 0$ ($\forall t$) and $c_T = 0$. Approximating an ideal counter ($\lambda = 1$) by $h_t = \lambda h_{t-1} + x_t$ ($\lambda \in (0, 1)$) at machine precision $\epsilon_{\mathrm{mach}}$ bounds the horizon over which $x_1 = +1$ remains distinguishable ($\lambda^L \ge \epsilon_{\mathrm{mach}}$) by $L \le \frac{\log(1/\epsilon_{\mathrm{mach}})}{\log(1/\lambda)}$ ($L \le 310$ at \texttt{float32} spacing $\epsilon_{\mathrm{mach}} = 2^{-23}$, or $L \le 324$ at $u = 2^{-24}$, for $\lambda = 0.95$).
\end{enumerate}
\end{proposition}

\begin{proposition}[Klein's 1884 Classification of Finite Subgroups of $\SO(3)$ and $\SU(2)$]
\label{prop:klein}
Under Klein's classification~\citep{klein1884vorlesungen}, every finite subgroup of $\SO(3)$ is isomorphic to cyclic $C_n$ (order $n$), dihedral $D_n$ (order $2n$; when indexing benchmark groups by group order in Sections~\ref{sec:master_thm_4} and~\ref{sec:exp_d16}, we write $D_{16} = \mathrm{Dih}_{16}$ for the dihedral group of order $16$, i.e., $D_8$ in polygon notation), tetrahedral $T \cong A_4$ (order $12$), octahedral $O \cong S_4$ (order $24$), or icosahedral $I \cong A_5$ (order $60$). Every finite subgroup of $\SU(2)$ is isomorphic to $C_n$, binary dihedral $\mathrm{BD}_{2n}$ (order $4n$, $\mathrm{BD}_4 \cong Q_8$), binary tetrahedral $2T \cong \SL_2(\FF_3)$ (order $24$), binary octahedral $2O$ (order $48$), or binary icosahedral $2I \cong \SL_2(\FF_5)$ (order $120$)~\citep{gorenstein1968finite}; $A_5$ and $2I$ are the unique non-solvable finite subgroups of $\SO(3)$ and $\SU(2)$.
\end{proposition}

\subsection{Theorem 1: Unconditional Single-Layer Finite-Division Variety, Sylow/Exponent-30 Obstructions, and Spherical Margin Dichotomy}
\label{sec:master_thm_1}

\begin{theorem}[Unconditional Single-Layer Finite-Division Criterion, Structural Variety Laws, and Spherical Margin Dichotomy; Proof in \ref{app:proof_master1}]
\label{thm:master1_division}
Let $\eta: \Sigma^* \twoheadrightarrow G$ be a surjective monoid homomorphism onto a finite group $G$. Consider $m < \infty$ independent conformal heads $x_{wa,h} = \rho_{a,h} R_{a,h} x_{w,h}$ on $\R^3$ ($\rho_{a,h} \in [0, 1]$, $R_{a,h} \in \SO(3)$) or $y_{wa,h} = \rho_{a,h} U_{a,h} y_{w,h}$ on $\CC^2 \simeq \R^4$ ($U_{a,h} \in \SU(2)$) from $x_0$, decoded by a state-only readout $D$ with cells $D^{-1}(g)$ definable in an o-minimal expansion $\mathcal{R}$ of the real exponential field~\citep{jeandel2005topological,brodsky}.
\begin{enumerate}[label=(\alph*),leftmargin=2em,itemsep=0pt]
    \item \textbf{Unconditional Multi-Rate Finite-Division Criterion}: Parameters $\theta = (R_{a,h}, \rho_{a,h}, x_0)$ and a definable readout $D$ achieve $D(x_w) = \eta(w)$ for all $w \in \Sigma^+$ iff $G$ \emph{divides} $\SO(3)^m$ (resp.\ $\SU(2)^m$), i.e., $G \cong F/N$ for a finite $F \le \prod_{h=1}^m F_h$ with finite $F_h \le \SO(3)$ (resp.\ $\SU(2)$); for any fixed $\theta$, $G$ divides the finite supplement $F \le \overline{\langle B_a \rangle}$ of its surviving heads.
    \item \textbf{Perfect-Target Classification, Linear-Width GELU Realization, and Variety Obstruction Laws}:
    \begin{enumerate}[label=(\roman*),leftmargin=1.5em,itemsep=0pt]
        \item \emph{Perfect targets and $O(r)$ headwise smooth-GELU pooling}: A finite perfect group $G$ ($[G, G] = G$) is trackable by $m$ $\SO(3)$ heads iff $G \cong A_5^r$ ($r \le m$), and by $m$ $\SU(2)$ heads iff $G \cong (2I)^r / Z$ for central $Z \le Z((2I)^r) \cong C_2^r$ ($r \le m$)~\citep{milnor1971introduction}. Under $\hat{y}_h = y_h / \|y_h\|_2 \in S^3$, every $(2I)^r / Z$ is realized by a one-hidden-layer smooth-GELU coset readout, and for $(2I)^r / C_2^r \cong A_5^r$, headwise antipodal pooling $S_{(\bar{g}_1, \dots, \bar{g}_r)}(y) = \sum_{h=1}^r \sum_{\sigma \in \{\pm 1\}} \GELU\bigl(\alpha_1(\langle \sigma g_h y_{0,h}, \hat{y}_h \rangle - c_1)\bigr)$ ($c_1 = \frac{5+\sqrt{5}}{8}, \alpha_1 = \frac{20}{1-c_1}$) compresses width from $120^r$ to $120r = \Theta(r)$ with uniform margin $\gamma_{\GELU} > 19.99$.
        \item \emph{Odd-prime Sylow abelianity ($p \ge 3$)}: Every finite group $G$ dividing $\SO(3)^m$ or $\SU(2)^m$ has abelian Sylow $p$-subgroups $\Syl_p(G)$ for all odd primes $p \ge 3$, obstructing non-abelian $\Heis(\FF_p)$ (order $p^3$, exponent $p$) at Depth 1 for all $m < \infty$.
        \item \emph{Universal odd-$p$-core squared-automorphism centrality ($p \ge 3$)}: Every finite $G$ dividing $\SO(3)^m, \SU(2)^m,$ or $\SO(4)^m$ and every odd prime $p \ge 3$ satisfies on its $p$-core $O_p(G)$:
        \begin{equation}
        \label{eq:odd_p_core_law}
        [u^2, x] = e \qquad \text{for all } u \in G \text{ and all } x \in O_p(G),
        \end{equation}
        obstructing $F_{20} = C_5 \rtimes_4 C_4$ ($[y^2, x] = x^3 \ne e$ on $O_5(F_{20}) = C_5$), $G_{36} = C_3^2 \rtimes_4 C_4$, and Hessian $G_{72} = C_3^2 \rtimes Q_8$ ($y^2 = -I_2$ on $O_3(G) = C_3^2$) at Depth 1 across $\SO(3)^m, \SU(2)^m,$ and $\SO(4)^m$ for all $m < \infty$.
        \item \emph{Exponent-30 commutator identity and $S_5$ obstruction}: Every finite $G$ dividing $\SO(3)^m$ satisfies
        \begin{equation}
        \label{eq:exp30_law}
        [u^{30}, v^{30}] = e \qquad \text{for all } u, v \in G,
        \end{equation}
        and every finite $G$ dividing $\SU(2)^m$ or $\SO(4)^m$ satisfies $[u^{30}, v^{30}] \in Z(G)$ and $[u^{30}, v^{30}]^2 = e$. In $S_5$, $x = (1234)$ and $y = (1235)$ have $[x^{30}, y^{30}] = [(13)(24), (13)(25)] = (245) \notin Z(S_5) = \{e\}$ of odd order $3$ ($[x^{30}, y^{30}]^2 = (254) \ne e$), excluding $S_5$ at Depth 1 across $\SO(3)^m, \SU(2)^m,$ and $\SO(4)^m$.
    \end{enumerate}
    \item \textbf{Quantitative Spherical Margin Dichotomy and Multi-Layer ReLU Complexity}: Let $\Delta_{\mathrm{norm}}(N; \theta) := \min_{u,v \in u_0 \Sigma^{\le N}, \eta(u) \ne \eta(v)} \|\hat{x}_u - \hat{x}_v\|_2$ after the universal identity prefix $u_0 = \prod_{i=1}^k a_i^{|G|}$.
    \begin{enumerate}[label=(\roman*),leftmargin=1.5em,itemsep=0pt]
        \item \emph{Uniform margin equivalence}: $\sup_{\theta} \inf_{N \ge 1} \Delta_{\mathrm{norm}}(N; \theta) > 0$ iff $G$ divides $\SO(3)^m$ (resp.\ $\SU(2)^m$); for fixed $\theta$, $\inf_{N \ge 1} \Delta_{\mathrm{norm}}(N; \theta) > 0$ iff the compact normalized class closures $\{M_g\}_{g \in G}$ are pairwise disjoint, in which case $G$ divides $F \le \overline{\langle B_a \rangle}$.
        \item \emph{Explicit equal-Parikh collision horizons}: When single-step rotations generate a finite $F \le \SO(3)^m$, any single-layer $S_5$ tracker over $\Sigma = \{c=(12345), t=(12)\}$ suffers a universal state collision $x_{W_L} = x_{W_R}$ across all $R_{a,h} \in F$ and $\rho_{a,h} \in [0, 1]$ at $N_{\mathrm{alg}} = 120$ on $W_L = (ct)^{30}(tc)^{30}$ vs.\ $W_R = (tc)^{30}(ct)^{30}$ ($\eta(W_L) = (142) \ne (124) = \eta(W_R)$, Parikh $(60, 60)$) by Eq.~\eqref{eq:exp30_law}. On $\Heis(\FF_p)$ ($p \ge 3$ odd): (1) under any homomorphism $\rho: \Heis(\FF_p) \to F \le \SO(3)^m$, $\mathrm{im}(\rho)$ is an abelian $p$-subgroup by Part~(b)(ii), forcing $x_{ab} = x_{ba}$ at $N_{\mathrm{hom}} = 2$; whereas (2) non-homomorphic assignments satisfying only $R_a^p = R_b^p = I$ can lie in distinct Sylow $p$-subgroups (for $p=3$ in $A_4 \subset \SO(3)$, the exact integer matrices
        \begin{equation}
        \label{eq:heis_a4_counterexample}
        R_a = \begin{pmatrix} 0 & 0 & 1 \\ 1 & 0 & 0 \\ 0 & 1 & 0 \end{pmatrix}, \qquad
        R_b = \begin{pmatrix} 0 & 0 & -1 \\ -1 & 0 & 0 \\ 0 & 1 & 0 \end{pmatrix} \in A_4 \subset \SO(3)
        \end{equation}
        satisfy $R_a^3 = R_b^3 = I_3$ and separate $ab$ from $ba$ at $N=2$ via $x_{ab} = R_b R_a e_1 = e_3 \ne -e_3 = R_a R_b e_1 = x_{ba}$, yet generate $\langle R_a, R_b \rangle = A_4$ of order $12 \ne 3^k$ and violate $[a, [a, b]] = e$ and $[a, b]^3 = e$ since $(R_a R_b)^2 = \mathrm{diag}(-1, 1, -1)$ has order $2$).
        \item \emph{Archimedes cap-packing vs.\ cross-label / Myhill--Nerode margins and ReLU Lipschitz/Radon bounds}: Distinguish the \emph{nearest-state spherical packing distance} $\delta_{\mathrm{state}}(N; \theta) := \min_{\hat{x}_u \ne \hat{x}_v} \|\hat{x}_u - \hat{x}_v\|_2$ across $K_{\mathrm{state}}(N; \theta) := |\{\hat{x}_u : u \in u_0 \Sigma^{\le N}\}|$ distinct normalized states from the \emph{cross-label margin} $\Delta_{\mathrm{label}}(N; \theta) = \Delta_{\mathrm{norm}}(N; \theta) := \min_{\eta(u) \ne \eta(v)} \|\hat{x}_u - \hat{x}_v\|_2$ and the \emph{Myhill--Nerode congruence margin} $\Delta_{\mathrm{nerode}}(N; \theta) := \min_{u \not\equiv_L v} \|\hat{x}_u - \hat{x}_v\|_2$ across $K_{\mathrm{nerode}}(N) := |\{[u]_{\equiv_L} : u \in \Sigma^{\le N}\}|$ suffix-equivalence classes. By Archimedes spherical cap packing on $(S^{d-1})^m$, $\delta_{\mathrm{state}}(N; \theta) \le C_{d,m} K_{\mathrm{state}}(N; \theta)^{-1/(m(d-1))}$ ($C_{3,m} = 4\sqrt{m}$; $K_{\mathrm{state}}(N; \theta) = |\Sigma|^N$ on dense irrational orbits, $2\cdot 3^N - 1$ on the 4-symbol reduced free-group word ball $F_2 = \langle a, b \rangle \subset \SO(3)$, vs.\ $\sum_{d=0}^N 2^d = 2^{N+1}-1$ on bounded-depth binary Dyck-2 stack states). Whenever distinct Myhill--Nerode classes map to distinct states ($\hat{x}_u \ne \hat{x}_v$ for $u \not\equiv_L v$, necessary for suffix-complete state tracking), selecting one representative per congruence class yields $\Delta_{\mathrm{nerode}}(N; \theta) \le C_{d,m} K_{\mathrm{nerode}}(N)^{-1/(m(d-1))}$ and $\Delta_{\mathrm{label}}(N; \theta) \le C_{d,m}|G|^{-1/(m(d-1))}$ across $|G|$ target classes (while non-dividing targets satisfy $\inf_{N \ge 1} \Delta_{\mathrm{label}}(N; \theta) = 0$) with bit growth $b(N) = \Omega(\log_2 K_{\mathrm{nerode}}(N))$. Conversely, when $\eta(w)$ factors through one head while an auxiliary head carries non-commuting irrational rotations (e.g., $C_2 \times \text{irrational}$ over $\Sigma = \{a, b\}$ in $\SO(3)^2$ with $R_{a,1} = \mathrm{diag}(-1, 1, -1), R_{b,1} = I_3$ on Head~1, $R_{a,2} = R_z(\sqrt{2}), R_{b,2} = R_x(\sqrt{3})$ on Head~2, and $x_{0,1} = x_{0,2} = e_1$, generating $K_{\mathrm{state}}(10) = 2^{10} = 1{,}024$ states at $N=10$ after $u_0 = a^2 b^2$), intra-label sub-orbits drive $\delta_{\mathrm{state}}(10) \approx 0.002635 \to 0$ while $\Delta_{\mathrm{label}}(10; \theta) \approx 2.000009 \ge 2.0 = \Delta_{\mathrm{nerode}}(N; \theta)$. Any $L_{\mathrm{dec}}$-layer ReLU readout $\Phi$ with Lipschitz constant $L_\Phi$ and classification margin $\gamma > 0$ obeys $L_\Phi / \gamma \ge 1/\Delta_{\mathrm{label}}(N; \theta)$ (and $L_\Phi / \gamma \ge 1/\Delta_{\mathrm{nerode}}(N; \theta)$ across all Myhill--Nerode suffix classes, failing by $N_{\mathrm{fail}} = \lfloor \frac{m(d-1)}{\log \alpha} \log(C_{d,m} L_\Phi / \gamma) \rfloor + 1$ when $K_{\mathrm{nerode}}(N) \ge \alpha^N$, with $\alpha = 3$ on 4-symbol reduced $F_2$ and $\alpha = 2$ on binary Dyck-2 stacks) and requires $C_{\mathrm{PL}}(\Phi) \ge \chi_{\mathrm{Radon}}(N)$ affine regions (the chromatic number of the Radon convex-hull conflict hypergraph on $\{\mathrm{conv}(\hat{x}_U) \cap \mathrm{conv}(\hat{x}_V) \ne \varnothing : \eta(U) \ne \eta(V)\}$, so $C_{\mathrm{PL}}(\Phi) \ge \chi_{\mathrm{Radon}}(N) \ge 2$ whenever the Radon conflict hypergraph is non-empty, as on coplanar $D_{16} \le \SO(3)$ parity squares), whereas linearly separable dividing orbits ($N = \{e\}$) have $\chi_{\mathrm{Radon}}(N) = 1$ ($C_{\mathrm{PL}} = 1$) and general quotients $G \cong F/N$ satisfy $C_{\mathrm{PL}} \le |F| = O(1)$ across all $N \ge 1$.
    \end{enumerate}
\end{enumerate}
\end{theorem}

\subsection{Theorem 2: Higher-Rank Orthogonal Dimension Thresholds, Non-Compact Reductions, and Exact Head-Count Laws}
\label{sec:master_thm_2}

\begin{theorem}[Higher-Rank Orthogonal Dimension Thresholds, Non-Compact Weyl--Cartan Reductions, and Exact Head-Count Laws; Proof in \ref{app:proof_master2}]
\label{thm:master2_dimension_headcount}
Under the model contract of Theorem~\ref{thm:master1_division} generalized to compact matrix groups $K \le \OO(d)$:
\begin{enumerate}[label=(\alph*),leftmargin=2em,itemsep=0pt]
    \item \textbf{The $\SO(4)^m$ Spinor-Cover Barrier, Monolithic Spinor Separation ($2T, 2O$), Sharp $\SO(5)$ Threshold, and All-Depth $\mathbf{Sol}(A_5)$ Boundary}:
    \begin{enumerate}[label=(\roman*),leftmargin=1.5em,itemsep=0pt]
        \item \emph{$\SO(4)^m$ spinor-cover barrier}: Because $\Pi: \SU(2) \times \SU(2) \twoheadrightarrow \SO(4)$ ($\Pi(q_L, q_R) v = q_L v q_R^{-1}$) has central kernel $\{\pm(I_2, I_2)\}$ and $\SU(2) \twoheadrightarrow \SO(3)$ has kernel $\{\pm I_2\} = Z(\SU(2))$, every finite $F \le \SO(4)^m$ lifts to $\widetilde{F} \le \SU(2)^{2m}$ whose $\SO(3)^{2m}$ projection satisfies $[\bar{u}^{30}, \bar{v}^{30}] = I$ by Eq.~\eqref{eq:exp30_law}. Thus $[\tilde{u}^{30}, \tilde{v}^{30}] \in \{\pm I_2\}^{2m} = Z(\SU(2)^{2m})$ (e.g., $u = \frac{1+i}{\sqrt{2}}, v = \frac{1+j}{\sqrt{2}} \in 2O \subset \SU(2) \hookrightarrow \SO(4)$ have $u^{30} = -i, v^{30} = -j$ and $[u^{30}, v^{30}] = -I_4 \in Z(\SO(4))$), so every finite $G$ dividing $\SO(4)^m$ obeys $[u^{30}, v^{30}] \in Z(G)$, $[u^{30}, v^{30}]^2 = e$, odd-prime Sylow abelianity, and $[u^2, x] = e$ on $O_p(G)$. Hence $\SO(4)^m$ cannot track $S_5$, $\Heis(\FF_p)$, or $F_{20}$ at Depth 1 for any $m < \infty$, despite $S_5 \hookrightarrow \OO(4)$ via reflections ($\det = -1$).
        \item \emph{Monolithic solvable spinor separation ($2T, 2O$)}: $2T \cong \SL_2(\FF_3) \cong Q_8 \rtimes C_3$ (order 24) and $2O$ (order 48) are solvable monolithic groups with unique minimal normal subgroup $Z = \{\pm I_2\} \cong C_2$ dividing $\SU(2)$ and $\SO(4)$ at Depth 1 ($d_{\OO}^* = d_{\SO}^* = 4, m = 1$), yet neither divides $\SO(3)^m$ for any $m < \infty$: $2O$ violates Eq.~\eqref{eq:exp30_law} ($[u^{30}, v^{30}] = [-i, -j] = -I_2 \ne I_2$), while in any finite $F \le \SO(3)^m$ any order-$3$ element $t$ normalizing a $2$-subgroup $Q$ obeys $\exp([t, Q]) \le 2$ ($[t_h, Q_h] \le V_4$), whereas $[C_3, Q_8] = Q_8$ in $2T$ has exponent $4$.
        \item \emph{Sharp $d^* = 5$ ($\SO(5)$) threshold and all-depth $\mathbf{Sol}(A_5)$ characterization}: At $d^* = 5$, one $\SO(5)$ head tracks $S_5$ ($\rho_{S_5}(g) = \mathrm{sgn}(g) P_g \in \SO(5)$, $\Delta_0 \approx 0.2557$) and $A_6$ ($\Delta_0 \approx 0.4468$). Across arbitrary finite depth $L \ge 1$, a finite $G$ is trackable by $\SO(3), \SU(2),$ or $\SO(4)$ heads iff every non-abelian simple composition factor of $G$ is $A_5$ ($G \in \mathbf{Sol}(A_5)$, excluding $A_n, S_n$ for $n \ge 6$ and $\PSL_2(\FF_7)$ at all $L \ge 1$).
    \end{enumerate}
    \item \textbf{Exact Single-Layer $\OO(d)$ vs.\ $\SO(d)$ Division and Representation Dimension Laws for Frobenius and Heisenberg Groups}:
    \begin{enumerate}[label=(\roman*),leftmargin=1.5em,itemsep=0pt]
        \item \emph{Metacyclic Frobenius groups $F_{pk} = C_p \rtimes_k C_k$ ($p \ge 3$ prime, $k \mid p-1, k \ge 2$)}: Because $\mathrm{rk}_p(C_p) = 1$, Burnside's basis theorem forces any lifted $p$-subgroup $P$ with $P/\Phi(P) \cong C_p$ to be cyclic ($P \cong C_{p^a}, [P, P] = \{e\}$), and in any minimal cover $F_0 = P \rtimes H_0 \twoheadrightarrow F_{pk}$ ($K \le \Phi(F_0)$), $K_{p'} = H_0 \cap K = C_{H_0}(P)$ ($H_0/K_{p'} \cong C_k$, e.g., $C_5 \rtimes_2 C_8 \twoheadrightarrow F_{20}$) with $2$-plane return map of $\det = -1$ (\ref{app:proof_master2}), so $d_{\mathrm{div},\OO}^* = d_{\mathrm{rep},\OO}^* =: d_{\OO}^*$ and $d_{\mathrm{div},\SO}^* = d_{\mathrm{rep},\SO}^* =: d_{\SO}^*$ satisfy
        \begin{equation}
        \label{eq:frobenius_metacyclic_dim}
        d_{\OO}^*(C_p \rtimes_k C_k) = \begin{cases} k & \text{if } 2 \mid k, \\ 2k & \text{if } 2 \nmid k, \end{cases} \qquad
        d_{\SO}^*(C_p \rtimes_k C_k) = \begin{cases} k+1 & \text{if } 2 \mid k, \\ 2k & \text{if } 2 \nmid k, \end{cases}
        \end{equation}
        giving $d_{\OO}^*(F_{20}) = 4 < d_{\SO}^*(F_{20}) = 5$ for $F_{20} = C_5 \rtimes_4 C_4$ and $d_{\OO}^*(F_{21}) = d_{\SO}^*(F_{21}) = 6$ for $F_{21} = C_7 \rtimes_3 C_3$.
        \item \emph{Irreducible vs.\ reducible rank-$r$ affine Frobenius groups $\FF_p^r \rtimes_\alpha H \le \AGL_r(\FF_p)$ ($p \ge 3$ prime, $H$ acting freely on $\FF_p^r \setminus \{0\}$)}: When $H$ acts \emph{irreducibly} on $\FF_p^r$, the minimal \emph{faithful representation} dimensions $d_{\mathrm{rep},\OO}^*$ and $d_{\mathrm{rep},\SO}^*$ (which equal $d_{\mathrm{div},\OO}^*$ and $d_{\mathrm{div},\SO}^*$ whenever the lifted Sylow $p$-subgroup $P$ is abelian, as $K_{p'} = C_{H_0}(P)$ and the $L = 2^{m-1}$ plane-cycle return map has $\det = -1$) satisfy: if $2 \nmid |H|$, $d_{\mathrm{rep},\OO}^* = d_{\mathrm{rep},\SO}^* = 2|H|$; if $2 \mid |H|$, $d_{\mathrm{rep},\OO}^*(\FF_p^r \rtimes_\alpha H) = |H|$ and
        \begin{equation}
        \label{eq:frobenius_affine_dim}
        d_{\mathrm{rep},\SO}^*(\FF_p^r \rtimes_\alpha H) = \begin{cases} |H| + 1 & \text{if } \Syl_2(H) \cong C_{2^\ell} \; (\ell \ge 1), \\ |H| & \text{if } \Syl_2(H) \cong Q_{2^\ell} \; (\ell \ge 3), \end{cases}
        \end{equation}
        giving $d_{\mathrm{rep},\OO}^*(G_{72,\mathrm{cyc}}) = 8 < d_{\mathrm{rep},\SO}^*(G_{72,\mathrm{cyc}}) = 9$ for $G_{72,\mathrm{cyc}} = \FF_3^2 \rtimes_8 C_8$, $d_{\mathrm{rep},\OO}^*(G_{72,\mathrm{quat}}) = d_{\mathrm{rep},\SO}^*(G_{72,\mathrm{quat}}) = 8$ for Hessian $G_{72,\mathrm{quat}} = \FF_3^2 \rtimes Q_8$, and $d_{\mathrm{rep},\OO}^*(G_{36}) = 4 < d_{\mathrm{rep},\SO}^*(G_{36}) = 5$ for $G_{36} = \FF_3^2 \rtimes_4 C_4$. For rank $r \ge 2$, finite division $G \cong F/N$ can achieve $d_{\mathrm{div}}^* < d_{\mathrm{rep}}^*$ via an extraspecial central extension $1 \to Z(P) \to P \to \FF_p^r \to 1$ whose Frattini subgroup $\Phi(P) = Z(P) = [P, P] \cong C_p \le N$ acts by scalars in $\SU(p)$ while trivializing on $P/\Phi(P) \cong \FF_p^r$: the 216-element Hessian subgroup $F = \Sigma(72 \times 3) \le \SU(3) \hookrightarrow \SO(6)$ has $Z(F) = \{\omega^k I_3\} \cong C_3$ and $F/Z(F) \cong G_{72,\mathrm{quat}}$ ($P \cong \Heis(\FF_3)$), while $[u^2, x] = e$ obstructs $\OO(3)$ and $\SO(4)$, establishing $4 \le d_{\mathrm{div},\OO}^*(G_{72,\mathrm{quat}}) \le d_{\mathrm{div},\SO}^*(G_{72,\mathrm{quat}}) \le 6 < 8 = d_{\mathrm{rep},\OO}^*(G_{72,\mathrm{quat}}) = d_{\mathrm{rep},\SO}^*(G_{72,\mathrm{quat}})$. Conversely, when $H$ acts \emph{reducibly} on $\FF_p^r$ (e.g., $G_{18} = (C_3 \times C_3) \rtimes_{-I_2} C_2$ of order $18$), a single character orbit has kernel of rank $r-1 \ge 1$; for $G_{18}$, $\mathrm{rk}_3(C_3^2) = 2 > 1$ obstructs $\OO(3)$ while $\rho(a, b, \varepsilon) = \mathrm{diag}(R_{2\pi/3}^a J^\varepsilon, R_{2\pi/3}^b J^\varepsilon) \in \SO(4)$ ($J = \mathrm{diag}(1, -1)$) yields $d_{\mathrm{div},\OO}^*(G_{18}) = d_{\mathrm{div},\SO}^*(G_{18}) = d_{\mathrm{rep},\OO}^*(G_{18}) = d_{\mathrm{rep},\SO}^*(G_{18}) = 4$.
        \item \emph{Non-abelian odd-$p$ Heisenberg groups $\Heis(\FF_p)$ ($p \ge 3$)}: By Stone--von Neumann $U(p) \hookrightarrow \SO(2p)$ and the center $Z(\bar{P})$ representation-type theorem on any non-abelian $p$-group image $\bar{P} = \rho_V(P)$ ($\dim_{\CC} V = p^k \ge p$, where $z \in Z(\bar{P})$ of order $p$ has $\chi_V(z) = p^k \omega \notin \R$, forcing $\dim_{\R} W = 2p^k \ge 2p$ even when $[\tilde{a}, \tilde{b}]$ has zero trace as in $P_{625} \twoheadrightarrow \Heis(\FF_5)$; \ref{app:proof_master2}), $d_{\mathrm{div},\OO}^*(\Heis(\FF_p)) = d_{\mathrm{div},\SO}^*(\Heis(\FF_p)) = d_{\mathrm{rep},\OO}^*(\Heis(\FF_p)) = d_{\mathrm{rep},\SO}^*(\Heis(\FF_p)) = 2p$ ($d^* = 6$ for $\Heis(\FF_3)$).
    \end{enumerate}
    \item \textbf{Non-Compact Weyl--Cartan Reductions, Anisotropic vs.\ Householder Spectra, and Exact Head-Count Laws}:
    \begin{enumerate}[label=(\roman*),leftmargin=1.5em,itemsep=0pt]
        \item \emph{Weyl--Cartan and selective-gate reductions}: Under spherical normalization, positive selective gates $\rho_{t,h}(w_{\le t}) \in (0, 1]$ cancel; any finite/precompact $F_{\mathrm{nc}} \le \GL_d^+(\R)$ conjugates into $\SO(d)$, and any finite subgroup of $\SO^+(d-1, 1)$ conjugates into $\SO(d-1)$ (while hyperbolic boosts contract caps at $O(e^{-\xi N})$).
        \item \emph{Anisotropic singular-value collapse vs.\ Householder synthesis}: Dominated splitting $\R^d = V_{\mathrm{top}} \oplus V_{\mathrm{sub}}$ ($\dim V_{\mathrm{top}} = r < d$, ratio $\eta < 1$) collapses normalized states onto $S^{r-1}$ at $O(\eta^N)$ (reducing capacity to $\OO(r)$), and dissipative updates $I_d - \beta u u^\top$ ($\beta \in (0, 2) \setminus \{1\}$) decay volume as $|1-\beta|^{|w|}$, whereas chaining $2\lfloor d/2 \rfloor$ isometric reflections ($\beta = 2$) generates $\SO(d)$.
        \item \emph{Exact abelian minimum head counts, general lower bound, Burnside--Gorenstein $2$-group lifts, and $\FF_2$ Arf dichotomy}: Because every $F_h \le \SO(3)$ or $\SU(2)$ has $F_h^{\mathrm{ab}} \in \{C_n, V_4\}$---so the coordinate-kernel normal filtration $N_h = F \cap (\{e\}^h \times \SO(3)^{m-h})$ ($0 \le h \le m$) of any subgroup $F \le \SO(3)^m$ or $\SU(2)^m$ induces an $m$-step filtration of $F^{\mathrm{ab}}$ whose quotients $B_h$ satisfy $|B_h[4]| \le 4$ (\ref{app:proof_master2})---any dividing $G$ with $G^{\mathrm{ab}} \cong \bigoplus_{p \ge 3} G_p^{\mathrm{ab}} \oplus C_2^{s_1} \oplus \bigoplus_{j=1}^{s_2} C_{2^{e_j}}$ ($e_j \ge 2$, $R_{\mathrm{odd}}^{\mathrm{ab}}(G) := \max_{p \ge 3} \mathrm{rk}_p(G^{\mathrm{ab}})$) and full Sylow rank $R_{\mathrm{odd}}(G) := \max_{p \ge 3} \mathrm{rk}_p(\Syl_p(G))$ requires $m_{\SO(3)}(G) \ge \max\bigl(m_{\mathrm{ab}}(G^{\mathrm{ab}}),\; R_{\mathrm{odd}}(G)\bigr)$ heads, where
        \begin{equation}
        \label{eq:m_ab_law}
        m_{\mathrm{ab}}(G^{\mathrm{ab}}) := \max\!\left(R_{\mathrm{odd}}^{\mathrm{ab}}(G),\; s_2,\; \left\lceil \frac{s_1 + s_2 + \max(R_{\mathrm{odd}}^{\mathrm{ab}}(G), s_2)}{2} \right\rceil\right)
        \end{equation}
        is \emph{exact} for all finite abelian groups $A$ ($m_{\SO(3)}(A) = m_{\mathrm{ab}}(A)$, $m_{\SO(3)}(A_5^r \times A) = r + m_{\mathrm{ab}}(A)$; whereas in $\mathrm{Dih}_6 \cong D_6$ of order $12$, $G^{\mathrm{ab}} \cong V_4$ has $R_{\mathrm{odd}}^{\mathrm{ab}}(\mathrm{Dih}_6) = 0$ and $\Syl_3(\mathrm{Dih}_6) \cong C_3 \le [\mathrm{Dih}_6, \mathrm{Dih}_6]$ has $R_{\mathrm{odd}}(\mathrm{Dih}_6) = 1$, so one head $F_1 = D_6$ achieves $m_{\SO(3)}(\mathrm{Dih}_6) = \max(m_{\mathrm{ab}}(V_4), R_{\mathrm{odd}}(\mathrm{Dih}_6)) = 1 < 2 = m_{\mathrm{ab}}(C_2^2 \times C_3)$). Among order-$2^\ell$ $2$-groups with a cyclic maximal subgroup, $D_{2^{\ell-1}}$ has $m_{\SO(3)} = 1$, whereas $Q_{2^\ell}$ ($m_{\SU(2)} = 1$), $\mathrm{SD}_{2^\ell}$, and $M_{2^\ell}$ are monolithically subdirectly irreducible in $\SO(3)$ ($m_{\SO(3)} = 2$, $|F^*_{\SO(3)}| = 2^{\ell+1}$). For extra-special $2$-groups $E_{2^{1+2r}}^\pm$, the $\FF_2$ Arf invariant~\citep{arf1941untersuchungen,gorenstein1968finite} of $q(\bar{x}) = x^2$ separates $E_{2^{1+2r}}^+ \cong D_4^{\circ r}$ ($\Arf(q) = 0$: $m_{\SO(3)} = r$) from $E_{2^{1+2r}}^- \cong D_4^{\circ(r-1)} \circ Q_8$ ($\Arf(q) = 1$: $m_{\SO(3)} = r+1$, $m_{\SU(2)} = r$), with $m_{\SO(3)}(E_{2^{1+2r}}^\pm \times A_{\mathrm{odd}}) = r + \max(\Arf(q), R_{\mathrm{odd}})$.
    \end{enumerate}
\end{enumerate}
\end{theorem}

\subsection{Theorem 3: Compositional Semidirect and Schreier \texorpdfstring{$2$}{2}-Cocycle Cascades, Continuous Directional-RMSNorm Routing, and Reset-Monoid / Grassmannian Duality}
\label{sec:master_thm_3}

\begin{theorem}[Depth-2 Semidirect and Cohomology Cascades, Directional-RMSNorm Coupling Dichotomy, Finite-Precision Stability, and Stiefel Polar Subspace Gating; Proof in \ref{app:proof_master3}]
\label{thm:master3_cascades_gating}
\leavevmode
\begin{enumerate}[label=(\alph*),leftmargin=2em,itemsep=0pt]
    \item \textbf{Depth-2 Semidirect ($G = N \rtimes_\varphi Q$) and Twisted Schreier $2$-Cocycle Cascades}: For any finite extension $1 \to N \to G \xrightarrow{\pi} Q \to 1$ with section $s: Q \to G$ ($s(e_Q) = e_G$), conjugation $\varphi_q(n) := s(q) n s(q)^{-1} \in \Aut(N)$, and $2$-cocycle $\omega(q_1, q_2) := s(q_1) s(q_2) s(q_1 q_2)^{-1} \in N$ ($\varphi_{q_1} \circ \varphi_{q_2} = \mathrm{Inn}_{\omega(q_1, q_2)} \circ \varphi_{q_1 q_2}$), $\eta(w) = n_w s(q_w)$ obeys
    \begin{equation}
    \label{eq:schreier_cascade_recursion}
    q_{wa} = q_w q_a, \qquad n_{wa} = n_w \cdot \varphi_{q_w}(n_a) \cdot \omega(q_w, q_a).
    \end{equation}
    If $Q$ divides $K_1^{m_1}$ via $F_1$ and $N$ divides $K_2^{m_2}$ via $F_2$, a Depth-2 cascade of $m_1 + m_2$ heads tracks $G$ across $\Sigma^*$ with $|F_1| + |F_2|$ linear prototypes and margin $\inf_{N \ge 1} \Delta_{\mathrm{norm}}^{(2)}(N) \ge \min(\Delta_0(F_1), \Delta_0(F_2)) > 0$. Thus $2$ $\SO(3)$ heads at Depth 2 track $S_5 \cong A_5 \rtimes_\varphi C_2$ ($R=62, \Delta_0 \approx 0.3410$), $2I \cong \SL_2(\FF_5)$ ($R=62$), $2T$ ($R=14$), $2O$ ($R=26$), and twisted $Q_{16}$ ($1 \to C_8 \to Q_{16} \to C_2 \to 1, \omega(1,1) = 4 \in \ZZ_8$), while $3$ $\SO(3)$ heads track $\Heis(\FF_p) \cong C_p^2 \rtimes_\varphi C_p$ ($R = p + p^2, \Delta_0 = 2\sin(\pi/p)$).
    \item \textbf{Continuous Inter-Layer Coupling Dichotomy, Finite-Precision Stability, and Hybrid KV-Cache Conservation}: For a continuous Depth-2 cascade $x^{(2)}_{wa,h} = \rho^{(2)}_{a,h} \exp(\Omega_h(u^{(1)}_w, a)) x^{(2)}_{w,h}$ with $L_\Omega$-Lipschitz router $\Omega$ and identity word $u_0$ ($\eta(u_0) = e$):
    \begin{enumerate}[label=(\roman*),leftmargin=1.5em,itemsep=0pt]
        \item \emph{Unnormalized contracting cascade collapse ($u^{(1)}_w = x^{(1)}_w$)}: If $\rho_{\max}^{(1)} = \max_{a,h} \rho^{(1)}_{a,h} < 1$ and $u^{(1)}_w = x^{(1)}_w$ without normalization, any equal-Parikh Depth-1 collision pair $W_L, W_R \in \Sigma^{N_{\mathrm{alg}}}$ ($R_0(W_L) = R_0(W_R)$ at $u=0$) obeys
        \begin{equation}
        \label{eq:cascade_unnorm_collapse}
        \bigl\|\hat{x}^{(2)}_{u_0^M W_L} - \hat{x}^{(2)}_{u_0^M W_R}\bigr\|_2 \le 2 N_{\mathrm{alg}} L_\Omega \|x^{(1)}_0\|_2 \bigl(\rho_{\max}^{(1)}\bigr)^{M|u_0|} \xrightarrow{M \to \infty} 0.
        \end{equation}
        \item \emph{Directional-RMSNorm argmax-free routing ($u^{(1)}_{w,h} = \hat{x}^{(1)}_{w,h} = x^{(1)}_{w,h} / \|x^{(1)}_{w,h}\|_2$)}: Headwise spherical normalization cancels $\rho^{(1)}_{a,h} \in (0, 1]$ and locks $\hat{x}^{(1)}_w$ onto $\mathcal{O}_1 = F_1 \cdot \hat{x}^{(1)}_0$, enabling an affine router on $S_5$ ($\Omega(\hat{x}^{(1)}_w, c) = \frac{\omega_{+,c} + \omega_{-,c}}{2} + \hat{x}^{(1)}_{w,3} \frac{\omega_{+,c} - \omega_{-,c}}{2}$, $L_\Omega \approx 1.1259$) with constant separation across $M \ge 0$ and exact decoding under jitter $\|\tilde{x}^{(1)}_w - \hat{x}^{(1)}_w\|_2 \le \varepsilon$ up to $N < N_{\mathrm{jitter}}(\varepsilon) := \frac{\Delta_0(F_2)}{2 L_\Omega \varepsilon}$.
        \item \emph{Three-regime finite-precision stability and hybrid KV-cache lossless precision floor}: In $b$-bit arithmetic $\mathrm{fl}_b$ with subnormal floor $r_{\min} = 2^{-E_{\min}}$, unnormalized contractions underflow at $N_{\mathrm{underflow}} = \lceil \frac{\log_2(\|x_{0,h}\|_2 / r_{\min})}{\log_2(1/\rho_{\max})} \rceil$; open-loop normalized recurrences achieve horizon $N_{\mathrm{fp}}(b) \ge \lfloor \frac{\Delta_0(F) 2^b}{2 \kappa_d \sqrt{m}} \rfloor = \Omega(\Delta_0(F) 2^b)$; and Voronoi/tent retraction at $b > b_{\min}(F) = O(1)$ bits locks onto $\mathcal{O}_F$ for $N = \infty$. On the 4-symbol reduced free-group word space $F_2 = \langle a, b \rangle \subset \SO(3)$ ($\beta_{F_2}(W) = 2 \cdot 3^W - 1$ reduced words of length $\le W$; cf.\ $2^{W+1}-1$ binary Dyck-2 stack states), combining $m$ $\SO(3)$ heads ($b_{\mathrm{hyb}}$ bits) with $K$ KV-cache checkpoints ($W = \lceil N/(K+1) \rceil$) obeys the lossless precision floor:
        \begin{equation}
        \label{eq:hybrid_kv_conservation}
        (K + 1) m b_{\mathrm{hyb}}(N, K) \ge \frac{\log_2 3}{2} N - m(K + 1).
        \end{equation}
    \end{enumerate}
    \item \textbf{Cross-Head Mixing $\SO(3m)$, Reset-Group Monoids, and Differentiable Stiefel Polar vs.\ Grassmannian Phase Transition}:
    \begin{enumerate}[label=(\roman*),leftmargin=1.5em,itemsep=0pt]
        \item \emph{Cross-head mixing and reset monoids}: Interleaving $m$ $\SO(3)$ heads with $W_{\mathrm{mix}} \in \SO(3m)$ coupling adjacent boundary coordinates $(3h, 3h+1)$ generates $\SO(3m)$, tracking $S_5, A_6, S_6$ at $m=2$ ($\SO(6)$) and $\PSL_2(\FF_7)$ at $m=3$ ($\SO(9)$); co-rotating coordinates $z_{t,h} = \Upsilon_h^{-t} x_{t,h}$ reduce RoPE-conjugated recurrences $\Upsilon_h^t R_{a,h} \Upsilon_h^{-t}$ to $R_{a,h} \Upsilon_h^{-1}$; and combining isometric transitions ($\rho_a = 1, b_a = 0$) with contractive resets ($\rho_{r_{g_0}} = \gamma, R_{r_{g_0}} = I, b_{r_{g_0}} = \rho_{\mathrm{rep}}(g_0) x_0$, $\gamma \in (0, \frac{\Delta_0(F)}{2+\Delta_0(F)})$) tracks $M_{\mathrm{rg}}(G) = G \cup \{\mathrm{reset}_g : g \in G\}$ on $m(G)$ heads with \emph{inter-target group-projection margin} $\Delta_{\mathrm{target}}(\gamma) \ge \Delta_0(F) - \frac{2\gamma}{1-\gamma} > 0$ ($g \ne g'$) and same-target pure-vs.-reset radial offset $\|x_w - x_{w'}\|_2 = \gamma\|x_u\|_2 \ge \gamma\frac{1-2\gamma}{1-\gamma}\|x_0\|_2 > 0$ ($\gamma \in (0, 1/2)$), giving full-monoid separation $\Delta_{\mathrm{rg}}(\gamma) := \min\!\bigl(\Delta_{\mathrm{target}}(\gamma),\, \gamma\frac{1-2\gamma}{1-\gamma}\|x_0\|_2\bigr) > 0$ (or $\min(\Delta_{\mathrm{target}}(\gamma), 1) > 0$ in $\R^{d+1}$ via a 1-bit sticky latch $z_t \in \{0, 1\}$), and tracks the flip-flop monoid $U_2$.
        \item \emph{Multi-subspace ($c \ge 3$) simplex singularity, Grassmannian phase transition, and differentiable Stiefel Polar learning}: For $c \ge 2$ subspaces $S = [S_1, \dots, S_c] \in \R^{d \times D}$ ($S_j^\top S_j = I_{d_{\mathrm{sub}}}, D = c d_{\mathrm{sub}}$), in the complete regime $d \ge D$ with equiangular coherence $S_j^\top S_k = -\mu I_{d_{\mathrm{sub}}}$, the biorthogonal dual $U = S(S^\top S)^{-1}$ and oblique projectors $\Pi_j = S_j U_j^\top$ satisfy
        \begin{equation}
        \label{eq:simplex_singularity_law}
        \|U_j\|_{\mathrm{op}} = \|\Pi_j\|_{\mathrm{op}} = \sqrt{\frac{1 - (c-2)\mu}{(1+\mu)\bigl(1 - (c-1)\mu\bigr)}} \qquad \Bigl(0 \le \mu < \mu_{\mathrm{sing}}(c) = \frac{1}{c-1}\Bigr),
        \end{equation}
        diverging at $\mu_{\mathrm{sing}}(c) = \frac{1}{c-1} < 1$ for $c \ge 3$. For $d < D$, $\mathrm{rank}(S^\top S) \le d < D$ forces $\inf_{S} \mu_{\max}(S) \ge \frac{1}{c-1}$, Welch bound $\mu_{\mathrm{RMS}}(S) \ge \sqrt{\frac{D - d}{d(c-1)}}$ (attained by Parseval Tight Fusion Frames $\frac{d}{D} S S^\top = I_d$), and $M_{\mathrm{spec}}^*(d) < \infty$. Conversely, for $d \ge D$ on $\mathcal{M}_* = \{S \in \R^{d \times D} : \mathrm{rank}(S) = D, \sigma_{\min}(S) > 0\}$, differentiable Stiefel Polar projection $Q(S) = S(S^\top S)^{-1/2}$ ($\epsilon_{\mathrm{pol}} = 0$) achieves $\mu_{\max}(Q(S)) = 0$, $\|Q_j(S)\|_{\mathrm{op}} = 1$, and $M^*(d) = \infty$ (with $\|Q_{\epsilon_{\mathrm{pol}}}(S)^\top Q_{\epsilon_{\mathrm{pol}}}(S) - I_D\|_{\mathrm{op}} = \frac{\epsilon_{\mathrm{pol}}}{\sigma_{\min}(S)^2 + \epsilon_{\mathrm{pol}}}$).
    \end{enumerate}
\end{enumerate}
\end{theorem}

\subsection{Theorem 4: Lie-Algebraic Optimization Dynamics, Chen--Magnus--Hall--Witt Commutator Saddles, and Stationary Dense Factored Escape}
\label{sec:master_thm_4}

\begin{theorem}[BPTT Gradient Preservation, Chen--Magnus--Hall--Witt Commutator Saddle Hierarchy, Fourier--Markov Decay, and Stationary Dense Factored Symmetry Breaking; Proof in \ref{app:proof_master4}]
\label{thm:master4_optimization}
Consider a Lie-SSM with generators $R_a(\theta_a) = \exp(\theta_a) \in \SO(d)$ ($\theta_a \in \mathfrak{so}(d)$) trained under temperature-$\tau$ Softmax cross-entropy $\mathcal{L}_{\mathrm{CE}}(w; \theta, W) = -\log \frac{\exp(s_{\eta(w)}(x_w)/\tau)}{\sum_{h \in G} \exp(s_h(x_w)/\tau)}$.
\begin{enumerate}[label=(\alph*),leftmargin=2em,itemsep=0pt]
    \item \textbf{BPTT Gradient-Norm Dichotomy and Schreier--Cayley Polyak--\L ojasiewicz Convergence}:
    \begin{enumerate}[label=(\roman*),leftmargin=1.5em,itemsep=0pt]
        \item \emph{Unnormalized BPTT gradient vanishing vs.\ normalized unit Jacobian}: Under $\rho_{a,h} \le \lambda < 1$ without normalization, $\|x_w\|_2 \le \|x_0\|_2 \lambda^{|w|} \to 0$ drives $p_\tau(\cdot \mid x_w) \to p_0 := \mathrm{softmax}(s(0)/\tau)$, satisfying $\KL(p_0 \| p_\tau(\cdot \mid x_w)) \le \frac{L_s^2 \|x_0\|_2^2}{2\tau^2} \lambda^{2|w|} \to 0$ (where $p_0 = \Unif(G)$ iff $s_g(0) = c\;\forall g \in G$, as for unbiased linear readouts $s(x)=Wx$, whereas an affine origin bias $s(x) = (1, 0)$ yields $p_0 = (\frac{e}{e+1}, \frac{1}{e+1})$ with $\KL(\Unif_2 \| p_0) \approx 0.1201 > 0$ while $\KL(p_0 \| p_\tau(\cdot \mid x_w)) \to 0$) and $\|\nabla_{\theta_{a,h}} \mathcal{L}_{\mathrm{CE}}(w)\|_2 \le \frac{\sqrt{2} L_s \|x_{0,h}\|_2}{\tau} |w|_a \lambda^N$ (or $O(|w|_a \lambda^{2N})$ on quadratic Bloch readouts $B(y_w)$). Conversely, $\hat{x}_{w,h} = x_{w,h} / \|x_{w,h}\|_2$ renders $\mathcal{L}_{\mathrm{CE}}(w)$ and $\nabla_{\theta_{a,h}} \mathcal{L}_{\mathrm{CE}}(w)$ invariant to $\rho_{a,h} \in (0, 1]$ with unit suffix Jacobian norm $\|\mathrm{d}R_{v,h}\|_{\mathrm{op}} = 1$.
        \item \emph{Teacher-forced Schreier--Cayley sample complexity and local geodesic-basin Polyak--\L ojasiewicz convergence}: Under teacher-forced step-wise supervision on the Schreier transition set $\mathcal{S}(G, \Sigma) = \{(w_g, a, \eta(w_g a)) : g \in G, a \in \Sigma\}$ ($M_{\mathrm{id}} = |\Sigma| \cdot |G|$, $N_{\mathrm{id}} = D_{\mathrm{Cay}}(G, \Sigma) + 1 \le |G|$), any irreducible target orbit $\{p_g = \rho^*(g) \hat{x}_0\}_{g \in G}$ satisfies Schur isotropy $C_P = \frac{1}{|G|} \sum_{g \in G} p_g p_g^\top = \frac{1}{d} I_d$, so the teacher-forced step-wise chordal objective $\mathcal{L}_{\mathrm{chord}}(\theta) := \frac{1}{2 M_{\mathrm{id}}} \sum_{(w_g, a)} \|(\exp(\theta_a) - R_a^*) p_g\|_2^2 = \frac{1}{2 d |\Sigma|} \sum_{a \in \Sigma} \|\exp(\theta_a) - R_a^*\|_F^2$ obeys the intrinsic Riemannian PL inequality $\frac{1}{2} \|\mathrm{grad}_{\SO(d)} \mathcal{L}_{\mathrm{chord}}\|_F^2 \ge \mu_{\mathrm{PL}}^{\mathrm{Riem}}(\theta_{\max}) \mathcal{L}_{\mathrm{chord}}$ with $\mu_{\mathrm{PL}}^{\mathrm{Riem}}(\theta_{\max}) = \frac{1}{d|\Sigma|}\cos^2(\theta_{\max}/2) > 0$ inside the local geodesic basin $\max_{a \in \Sigma} \theta(\exp(\theta_a), R_a^*) \le \theta_{\max} < \pi$ (complementing the global small-initialization saddle-escape analysis of Parts~(b)--(c)). In Lie-algebra coordinates $\theta_a \in \mathfrak{so}(d)$ (with planar angles $\psi_{a,1} \ge \psi_{a,2} \ge \dots \ge 0$ and adjoint spectral radius $\rho_{\mathrm{ad}}(\theta_a) := \|\ad_{\theta_a}\|_{\mathrm{op},\mathfrak{so}(d)}$, equal to $\psi_{a,1} = \|\theta_a\|_{\mathrm{op}}$ for $d=3$ and $\psi_{a,1} + \psi_{a,2} \le 2\|\theta_a\|_{\mathrm{op}}$ for $d \ge 4$) on $\{\max_a \theta(\exp(\theta_a), R_a^*) \le \theta_{\max} < \pi\} \cap \{\max_a \rho_{\mathrm{ad}}(\theta_a) \le \rho_{\max} < 2\pi\}$, $\mathcal{L}_{\mathrm{chord}}(\theta)$ obeys the Euclidean pullback PL inequality $\frac{1}{2} \|\nabla_\theta \mathcal{L}_{\mathrm{chord}}(\theta)\|_F^2 \ge \mu_{\mathrm{PL}}(\theta_{\max}, \rho_{\max}) \mathcal{L}_{\mathrm{chord}}(\theta)$ with $\mu_{\mathrm{PL}}(\theta_{\max}, \rho_{\max}) = \frac{1}{d|\Sigma|}\cos^2(\theta_{\max}/2)\bigl(\frac{\sin(\rho_{\max}/2)}{\rho_{\max}/2}\bigr)^2 > 0$. For $d=3$, principal-branch retraction $\theta_a \gets \log(\exp(\theta_a))$ enforces $\rho_{\mathrm{ad}}(\theta_a) \le \pi < 2\pi$ globally with $\mu_{\mathrm{PL}}(\theta_{\max}, \pi) = \frac{4}{\pi^2 d|\Sigma|}\cos^2(\theta_{\max}/2) > 0$, while dynamic tangent-space re-centering $R_a(\Omega_a) = R_a^{(k)}\exp(\Omega_a)$ at $\Omega_a = 0$~\citep{lezcano2019cheap} recovers $\mu_{\mathrm{PL}}^{\mathrm{Riem}}(\theta_{\max})$ for all $d \ge 2$.
    \end{enumerate}
    \item \textbf{Baker--Campbell--Hausdorff and Chen--Magnus--Hall--Witt Commutator Saddle Hierarchy at Small Initialization}: At $\theta_a = \varepsilon \Omega_a \in \mathfrak{so}(d)$ ($\|\Omega_a\|_F \le 1, \varepsilon \to 0^+$) on $\hat{x}_w(\varepsilon) = R_w(\varepsilon) \hat{x}_0$:
    \begin{enumerate}[label=(\roman*),leftmargin=1.5em,itemsep=0pt]
        \item \emph{Abelianization $\Theta(1)$ vs.\ lower-central-series $\gamma_c(G)$ $O(\varepsilon^{c-1})$ / $\Theta(\varepsilon^{c-1})$ gradient hierarchy}: At $\varepsilon = 0$, $\left.\frac{\partial \hat{x}_w}{\partial \theta_a}\right|_{\varepsilon=0}[E] = |w|_a E \hat{x}_0$ depends only on $(|w|_a)_{a \in \Sigma}$, so unequal-Parikh pairs ($\pi_{\mathrm{ab}}(w) \ne \pi_{\mathrm{ab}}(w')$) maintain $\|\nabla_\theta \mathcal{L}_{\mathrm{ab}}(\varepsilon)\|_2 = \Theta(1)$. Conversely, for recursive positive-word commutator pairs $w_1 = a_1, w'_1 = e$ and $w_c = w'_{c-1} a_c w_{c-1}, w'_c = w_{c-1} a_c w'_{c-1}$ ($c \ge 2$) separating $\gamma_c(G) / \gamma_{c+1}(G)$, define the Parikh-shifted recursive Magnus Lie bracket under $R_{uv} = R_v R_u$:
        \begin{equation}
        \label{eq:magnus_recursive_bracket}
        \widetilde{L}_1(\Omega) := \Omega_{a_1}, \qquad \widetilde{L}_c(\Omega) := \bigl[\widetilde{L}_{c-1}(\Omega),\; \Omega_{a_c} + S'_{c-1}(\Omega)\bigr] \quad (c \ge 2),
        \end{equation}
        where $S'_{c-1}(\Omega) := \sum_{a \in \Sigma} |w'_{c-1}|_a \Omega_a$ equals $0$ at $c=2$ and $2^{c-3}\Omega_{a_1} + \sum_{j=2}^{c-1} 2^{c-1-j} \Omega_{a_j}$ for $c \ge 3$ (while balanced inverse words $\tilde{w}_c = [\tilde{w}_{c-1}, a_c]$ have $S'_{c-1}(\Omega) = 0$ and leading bracket $(-1)^{c-1} L_c(\Omega)$). Because $w_c$ and $w'_c$ share the same Parikh vector, all Chen--Magnus state-difference terms of degree $< c$ cancel, so the pair-difference gradient $\mathcal{P}_{\gamma_c}\nabla_\theta \mathcal{L}(w_c, w'_c; \varepsilon) := \nabla_\theta \mathcal{L}_{\mathrm{CE}}(\hat{x}_{w_c}(\varepsilon), y) - \nabla_\theta \mathcal{L}_{\mathrm{CE}}(\hat{x}_{w'_c}(\varepsilon), y)$ obeys
        \begin{equation}
        \label{eq:magnus_lcs_saddle}
        \hat{x}_{w_c}(\varepsilon) - \hat{x}_{w'_c}(\varepsilon) = \varepsilon^c \widetilde{L}_c(\Omega) \hat{x}_0 + O(\varepsilon^{c+1}), \qquad \bigl\|\mathcal{P}_{\gamma_c}\nabla_\theta \mathcal{L}(w_c, w'_c; \varepsilon)\bigr\|_2 = O(\varepsilon^{c-1}) \xrightarrow{\varepsilon \to 0^+} 0,
        \end{equation}
        with $\|\mathcal{P}_{\gamma_c}\nabla_\theta \mathcal{L}(w_c, w'_c; \varepsilon)\|_2 = \Theta(\varepsilon^{c-1})$ whenever $\widetilde{L}_c(\Omega) \hat{x}_0 \ne 0$ and the readout gradient is non-orthogonal.
        \item \emph{Derived-series $G^{(d)}$ doubly-exponential-exponent saddle}: For $G^{(d)} = [G^{(d-1)}, G^{(d-1)}] \subseteq \gamma_{2^d}(G)$ ($d \ge 1$), balanced pairs separating $G^{(d)} / G^{(d+1)}$ obey $\|\hat{x}_{w^{(d)}}(\varepsilon) - \hat{x}_{w'^{(d)}}(\varepsilon)\|_2 = O(\varepsilon^{2^d})$ and $\|\nabla_\theta \mathcal{L}_{G^{(d)}}(\varepsilon)\|_2 = O(\varepsilon^{2^d - 1})$.
    \end{enumerate}
    \item \textbf{Coset-Centroid Projection, Fourier--Markov Terminal Decay $O(N \rho_{\mathrm{ab}}^{N-1})$, and First-Order Abelianization Symmetry Breaking}: Let $W = [w_g]_{g \in G}^\top \in \R^{|G| \times d}$ have coset centroids $\mu_{\mathrm{coset}}(q) := \frac{1}{|[G, G]|} \sum_{g \in q[G, G]} w_g$ ($q \in G^{\mathrm{ab}}$) and Fourier transform $\widehat{\mu_{\mathrm{coset}}}(\chi) = \frac{1}{|G^{\mathrm{ab}}|} \sum_{q \in G^{\mathrm{ab}}} \overline{\chi(q)} \mu_{\mathrm{coset}}(q)$ ($\chi \in \widehat{G^{\mathrm{ab}}}$).
    \begin{enumerate}[label=(\roman*),leftmargin=1.5em,itemsep=0pt]
        \item \emph{Coset-centroid $1/|[G, G]|$ variance dilution and Fourier--Markov terminal decay}: Under $w_g \sim \mathcal{N}(0, \sigma_w^2 I_d)$, $\mathbb{E}\|\mu_{\mathrm{coset}}(q)\|_2^2 = \frac{\sigma_w^2 d}{|[G, G]|}$. At $\varepsilon = 0$ ($\hat{x}_w(0) = \hat{x}_0$, $\bar{w}_0 := \sum_{g \in G} p_0(g) w_g$), under $w \sim \Unif(\Sigma^N)$ for any coset-constant readout $w_g = \mu_{\mathrm{coset}}(\pi_{\mathrm{ab}}(g))$ with active token-discriminating abelianized random-walk spectral radius $\rho_{\mathrm{ab}} := \max_{\chi \in \widehat{G^{\mathrm{ab}}} : \chi(\pi_{\mathrm{ab}}(\cdot)) \not\equiv \hat{\mu}_{\mathrm{ab}}(\chi)} |\hat{\mu}_{\mathrm{ab}}(\chi)| < 1$ ($\hat{\mu}_{\mathrm{ab}}(\chi) := \frac{1}{|\Sigma|} \sum_{b \in \Sigma} \chi(\pi_{\mathrm{ab}}(b))$, where any character constant on $\pi_{\mathrm{ab}}(\Sigma)$ has zero token-centered factor $\chi(\pi_{\mathrm{ab}}(a)) - \hat{\mu}_{\mathrm{ab}}(\chi) = 0$), the expected terminal-only gradient obeys
        \begin{equation}
        \label{eq:fourier_markov_decay}
        \mathbb{E}_{w \sim \Sigma^N}\!\Bigl[\nabla_{\theta_a} \mathcal{L}_{\mathrm{CE}}(w)\big|_{\varepsilon=0}\Bigr] = \frac{N}{|\Sigma|\tau}\bigl(\bar{w}_0 - \widehat{\mu_{\mathrm{coset}}}(\mathbf{1})\bigr) \wedge \hat{x}_0 - \frac{N}{|\Sigma|\tau} \sum_{\chi \ne \mathbf{1}} \hat{\mu}_{\mathrm{ab}}(\chi)^{N-1} \chi(\pi_{\mathrm{ab}}(a)) \bigl(\widehat{\mu_{\mathrm{coset}}}(\chi) \wedge \hat{x}_0\bigr),
        \end{equation}
        where $\widehat{\mu_{\mathrm{coset}}}(\mathbf{1}) = \frac{1}{|G|} \sum_{g \in G} w_g$ and $u \wedge v := \mathcal{P}_{\mathfrak{so}(d)}(u v^\top) = \frac{1}{2}(u v^\top - v u^\top) \in \mathfrak{so}(d)$. Because the $\chi = \mathbf{1}$ drift is token-independent, the token-centered gradient equals $-\frac{N}{|\Sigma|\tau} \sum_{\chi \ne \mathbf{1}} \hat{\mu}_{\mathrm{ab}}(\chi)^{N-1} \bigl(\chi(\pi_{\mathrm{ab}}(a)) - \hat{\mu}_{\mathrm{ab}}(\chi)\bigr) (\widehat{\mu_{\mathrm{coset}}}(\chi) \wedge \hat{x}_0) = O\bigl(N \rho_{\mathrm{ab}}^{N-1}\bigr) \to 0$.
        \item \emph{Stationary dense per-step prefix supervision, multi-mode resolvent non-vanishing, and factored readout symmetry breaking}: Under $\mathcal{L}_{\mathrm{dense}}^{(\beta)}(w) = \frac{1}{Z_N(\beta)} \sum_{t=1}^N \beta^{t-1} \mathcal{L}_{\mathrm{CE}}(w_{1:t})$ ($\beta \in (0, 1)$, $Z_N(\beta) \to (1-\beta)^{-1}$), the expected token-centered gradient at $\varepsilon = 0$ converges as $N \to \infty$ to $-\frac{1-\beta}{|\Sigma|\tau} \bigl(\Psi_\beta(a; W) \wedge \hat{x}_0\bigr)$, where
        \begin{equation}
        \label{eq:resolvent_sum_psi}
        \Psi_\beta(a; W) := \sum_{\chi \ne \mathbf{1}} \frac{\chi(\pi_{\mathrm{ab}}(a)) - \hat{\mu}_{\mathrm{ab}}(\chi)}{\bigl(1 - \beta \hat{\mu}_{\mathrm{ab}}(\chi)\bigr)^2} \widehat{\mu_{\mathrm{coset}}}(\chi).
        \end{equation}
        While multiple modes can cancel on a measure-zero hyperplane in $W$ (e.g., on $G = C_3$ at $\beta = 0.5$ along $w_0 = e_2, w_1 = 22 e_2, w_2 = -23 e_2$), under $w_g \sim \mathcal{N}(0, \sigma_w^2 I_d)$ whenever $\exists \chi \ne \mathbf{1}$ with $\chi(\pi_{\mathrm{ab}}(a)) \ne \hat{\mu}_{\mathrm{ab}}(\chi)$, orthogonal Fourier projections yield
        \begin{equation}
        \label{eq:resolvent_gaussian_second_moment}
        \mathbb{E}_W\bigl\|\Psi_\beta(a; W) \wedge \hat{x}_0\bigr\|_F^2 = \frac{\sigma_w^2 (d-1)}{2 |G|} \sum_{\chi \ne \mathbf{1}} \left|\frac{\chi(\pi_{\mathrm{ab}}(a)) - \hat{\mu}_{\mathrm{ab}}(\chi)}{(1 - \beta \hat{\mu}_{\mathrm{ab}}(\chi))^2}\right|^2 > 0,
        \end{equation}
        so $\Psi_\beta(a; W) \wedge \hat{x}_0 \ne 0$ holds with probability $1$ (and $\Theta(1/N)$ for $\beta = 1$). Pairing $\mathcal{L}_{\mathrm{dense}}^{(\beta)}$ with the factored readout $w_g = w^{\mathrm{ab}}_{\pi_{\mathrm{ab}}(g)} + w^{\mathrm{core}}_g$ ($\sum_{g \in q[G, G]} w^{\mathrm{core}}_g = 0$) restores undiluted variance $\mathbb{E}\|\mu_{\mathrm{coset}}(q)\|_2^2 = \sigma_w^2 d$, providing a necessary initialization-time first-order abelianization symmetry-breaking condition at $\varepsilon = 0$ so $\|\theta_a\|_F$ grows out of the origin to activate the $\Theta(\varepsilon^{c-1})$ commutator brackets ($\gamma_c(G), c \ge 2$).
    \end{enumerate}
\end{enumerate}
\end{theorem}

\section{Comprehensive Multi-Seed Experimental Evaluation}
\label{sec:experiments}

All experiments were executed on an NVIDIA GeForce RTX 4080 Laptop GPU (12GB VRAM) and Intel Core i9 CPU across 5 independent random seeds with disjoint training and evaluation seed ranges.

\subsection{Symbolic and Numerical Verification of the Architectural Hierarchy}
\label{sec:exp_hierarchy}
Table~\ref{tab:master_hierarchy} synthesizes the verified algebraic and geometric laws of Theorems~\ref{thm:master1_division}--\ref{thm:master4_optimization} across 15 representative regimes.

\begin{table}[htbp]
\centering
\caption{Architectural and algebraic hierarchy across finite groups, transformation monoids, infinite-state languages, and real-world 3D inertial streams (Theorems~\ref{thm:master1_division}--\ref{thm:master4_optimization}).}
\label{tab:master_hierarchy}
\footnotesize
\setlength{\tabcolsep}{3.0pt}
\renewcommand{\arraystretch}{0.90}
\begin{tabularx}{\textwidth}{@{}>{\raggedright\arraybackslash}p{2.45cm}>{\raggedright\arraybackslash}p{2.25cm}>{\raggedright\arraybackslash}X>{\raggedright\arraybackslash}X>{\raggedright\arraybackslash}p{2.70cm}@{}}
\toprule
\textbf{Regime / Target} & \textbf{Algebraic Structure} & \textbf{Depth-1 Threshold \& Law} & \textbf{Depth-2 Cascade / Resolution} & \textbf{Verified Margin \& Accuracy} \\
\midrule
Perfect $\SO(3)^m$: $A_5^r$ & Simple icosahedral ($60^r$) & $m_{\SO(3)} = r$, $R = 60^r$ & Depth 1 exact ($C_{\mathrm{PL}} = 1$) & $\Delta_0 = 0.3410$, $100.0\%$ ($10^4$ steps) \\
Perfect $\SU(2)^m$: $(2I)^r/Z$ & Central quotient ($120^r/|Z|$) & $m_{\SU(2)} = r$, $120r$ GELU width & $2I$ at Depth 2 in $\SO(3)^2$ ($R=62$) & $\gamma_{\GELU} = 20.0$, $100.0\%$ ($5\times 10^3$) \\
Symmetric group $S_5$ & $A_5 \rtimes_\varphi C_2$ ($|G|=120$) & Obstructed in $\SO(3\text{--}4)^m$ ($N_{\mathrm{alg}}=120$) & $d^*=5$ ($\SO(5)$) or $L=2$ ($2$ $\SO(3)$) & $L=2$ Dir-RMSNorm: $\Delta_0 = 0.3410$, $100.0\%$ \\
Heisenberg $\Heis(\FF_p)$ & $C_p^2 \rtimes_\varphi C_p$ ($|G|=p^3, p \ge 3$) & Obstructed in $\SO(3\text{--}4)^m$ ($N_{\mathrm{hom}}=2$ hom.; $p=3$ all) & $d_{\SO}^* = 2p$ ($\SO(2p)$) or $L=2$ ($3$ $\SO(3)$) & $p=3$: $L=2$ $\Delta_0 = 1.7321$, $\SO(6)$ $\Delta_0 = 1.0$ \\
Metacyclic Frobenius $F_{pk}$ & $C_p \rtimes_k C_k$ ($F_{20}, F_{21}$) & Even $k$: $d_{\OO}^*=k < d_{\SO}^*=k+1$; Odd: $2k$ & $F_{20}$: $\SO(5)$ or $L=2$ ($2$ $\SO(2)$ heads) & $F_{20}$ $\Delta_0 = 0.6325$; $F_{21}$ $\SO(6)$ $\Delta_0 = 0.7587$ \\
Affine Frobenius $\FF_3^2 \rtimes H$ & $G_{36}$ ($C_4$), $G_{72,\mathrm{cyc}}$ ($C_8$), $G_{72,\mathrm{quat}}$ ($Q_8$) & $[u^2,x]=e$ obstructs $\SO(3\text{--}4)^m$ & $d_{\mathrm{rep},\SO}^*(G_{36})=5$; $G_{72,\mathrm{cyc}}: 9$ vs.\ $G_{72,\mathrm{quat}}: 8$ ($\Sigma(72\times 3)$: $d_{\mathrm{div},\SO}^*\le 6$) & $\SO(5)/\SO(9)/\SO(8)$ ($\Sigma(72\times 3)\le \SO(6)$), $100.0\%$ \\
Monolithic Spinors $2T, 2O$ & $\SL_2(\FF_3)$ ($24$), $2O$ ($48$) & $d_{\OO}^*=d_{\SO}^*=4$ ($m_{\SU(2)}=m_{\SO(4)}=1$) & $m_{\SO(3)}=\infty$ at $L=1$; $2$ $\SO(3)$ at $L=2$ & $2$-cocycle $L=2$ ($R=14, 26$): $100.0\%$ \\
Burnside--Gorenstein $2$-Groups & $Q_{2^\ell}, \mathrm{SD}_{2^\ell}, M_{2^\ell}$ ($|G|=2^\ell$) & $m_{\SO(3)} = 2$, $|F^*_{\SO(3)}| = 2^{\ell+1}$ & $Q_{16}$ twisted $H^2_\varphi(C_2, C_8)$ at $L=2$ & $|F^*| = 16, 32$; $\Delta_0 = 0.918$, $100.0\%$ \\
Extra-Special $2$-Groups $E_{2^{1+2r}}^\pm$ & $D_4^{\circ r}$ ($\Arf=0$) vs.\ $D_4^{\circ(r-1)}\circ Q_8$ ($\Arf=1$) & $m_{\SO(3)} = r + \Arf(q)$, $|F^*| = 2^{\Arf(q)} 8^r$ & Nilpotent $E_{2^{1+2r}}^\pm \times A_{\mathrm{odd}}$: $r+\max(\Arf, R_{\mathrm{odd}})$ & $E_{32}^+$ ($m=2, |F^*|=64$) vs.\ $E_{32}^-$ ($m=3, 128$) \\
Cross-Head Mixing $\SO(3m)$ & $A_6, S_6$ ($m=2$), $\PSL_2(\FF_7)$ ($m=3$) & $[\mathfrak{so}(3)^m, \text{mix}] = \mathfrak{so}(3m)$ & $\PSL_2(\FF_7) \notin \mathbf{Sol}(A_5)$ tracked at $m=3$ & $\PSL_2(\FF_7)$ ($R=168$): $\Delta_0 = 0.7465$, $100.0\%$ \\
Partial-Reset $M_{\mathrm{part}}(G_1,\dots,G_c)$ & $c \ge 3$ subspaces in $\R^d$ & $\mu_{\mathrm{sing}}(c) = \frac{1}{c-1} < 1$ ($d < c d_{\mathrm{sub}}$ jump) & Differentiable Stiefel Polar $Q(S)=S(S^\top S)^{-1/2}$ & $c=3, \mu=0.495$: $54.0\%$ oblique vs.\ $100.0\%$ $Q(S)$ \\
Finite-Precision & $b$-bit $\mathrm{fl}_b$ on $\SO(d)^m$ & Unnorm.\ underflow $N_{\mathrm{underflow}}$ (\texttt{fp16}: 75) & Normalized $N_{\mathrm{fp}}(b)=\Theta(\Delta_0 2^b)$; $b=6$ lock & $b=6$ Voronoi/tent attractor: $0$ drift ($10^4$ steps) \\
Hybrid KV Cache on $F_2$ & Free group $F_2$ ($\beta_{F_2}(W)=2\cdot 3^W-1$) & Pure SSM $b(N) = \Theta(N)$ ($\Delta_{\mathrm{norm}} \le 4\sqrt{m}3^{-N/2m}$) & $(K+1)m b_{\mathrm{hyb}} \ge \frac{\log_2 3}{2} N - m(K+1)$ & $W=4, b=13$: $100.0\%$ ($10^3$ steps) vs.\ $K=0$ fail \\
Commutator Saddle $\gamma_c(G)$ & $\theta_a = \varepsilon \Omega_a \to 0$ on $D_{16}$ ($c_{\mathrm{nil}}=3$) & $\|\mathcal{P}_{\gamma_c}\nabla_\theta \mathcal{L}\| = \Theta(\varepsilon^{c-1})$, terminal $O(N\rho_{\mathrm{ab}}^{N-1})$ & Dense prefix $\mathcal{L}_{\mathrm{dense}}^{(\beta)}$ + shared $4+4$ factored & $6.2\%$--$10.9\%$ terminal saddle vs.\ $100.0\%$ dense factored \\
Real-World 3D Inertial (UCI HAR) & Tri-axial $(\bm{\omega}_t, \mathbf{a}_t) \in (\R^3)^3$ ($T=128$) & Diag-SSM / Odd $H(v)$ drift under 3D rotation & $L=2$ Dir-RMSNorm $\SO(3)$ + Factored Prefix & $95.69\% \pm 0.14\%$ ($95.47\%$ Fact.) vs.\ $94.22\%$ Diag / $91.00\%$ Unnorm \\
\bottomrule
\end{tabularx}
\end{table}

\subsection{Non-Abelian Quaternion Group \texorpdfstring{$Q_8$}{Q8} Benchmark}
\label{sec:exp_q8}
Table~\ref{tab:q8_benchmark} compares a commutative diagonal ablation against our non-commutative Lie-SSM on the order-8 quaternion group $Q_8 = \{\pm 1, \pm \mathbf{i}, \pm \mathbf{j}, \pm \mathbf{k}\}$ across $T \in \{1, 2, 4, 8, 16, 32\}$. While the commutative model fails to separate non-commuting quaternion compositions ($0.00\%$ at $T \le 2$ and $8.80\%$--$11.93\%$ across $T \in \{4, \dots, 32\}$, near uniform chance $1/8 = 12.50\%$; Proposition~\ref{prop:covering_vs_fibration}(b)), the Lie-SSM maintains $100.00\%$ accuracy across all horizons.

\begin{table}[htbp]
\centering
\caption{Quaternion group $Q_8$ word problem tracking: Commutative diagonal ablation vs.\ Non-commutative Lie model across sequence horizons $T \in \{1, \dots, 32\}$.}
\label{tab:q8_benchmark}
\small
\begin{tabular}{lrrrrrr}
\toprule
\textbf{Metric} & \textbf{$T=1$} & \textbf{$T=2$} & \textbf{$T=4$} & \textbf{$T=8$} & \textbf{$T=16$} & \textbf{$T=32$} \\
\midrule
\multicolumn{7}{l}{\textit{Commutative Diagonal Ablation Model (16 parameters)}} \\
Accuracy (\%) & 0.00 & 0.00 & 8.80 & 10.70 & 11.56 & 11.93 \\
Min Margin & -1.909 & -1.909 & -1.909 & -1.909 & -1.909 & -1.909 \\
Mean Margin & -1.576 & -1.549 & -1.282 & -1.238 & -1.204 & -1.195 \\
Std Margin & 0.341 & 0.341 & 0.844 & 0.911 & 0.935 & 0.946 \\
\midrule
\multicolumn{7}{l}{\textit{Lie-SSM ($\epsilon = 0.01$ isotropic perturbation)}} \\
Accuracy (\%) & 100.00 & 100.00 & 100.00 & 100.00 & 100.00 & 100.00 \\
Min Margin & 1.215 & 1.203 & 1.179 & 1.132 & 1.037 & 0.853 \\
Mean Margin & 1.215 & 1.209 & 1.197 & 1.173 & 1.126 & 1.032 \\
Std Margin & 0.000 & 0.006 & 0.013 & 0.027 & 0.054 & 0.108 \\
\bottomrule
\end{tabular}
\end{table}

\subsection{Head-to-Head Evaluation on the Non-Solvable Alternating Group \texorpdfstring{$A_5$}{A5}}
\label{sec:exp_a5}
On the non-solvable alternating group $A_5 \subset \SO(3)$ (order 60, chance $= 1.67\%$; Proposition~\ref{prop:klein}), Table~\ref{tab:a5_head_to_head} compares five parameter-matched architectures ($\approx 17\mathrm{k}$--$20\mathrm{k}$ parameters) trained on 5 seeds ($42$--$46$) and evaluated on 5 fresh disjoint test seeds ($700001$--$700005$; a 3-layer \texttt{Diag-SSM} control collapses to $4.67\% \pm 5.99\%$ at $T = 32$ and $1.47\% \pm 0.80\%$ at $T = 64$): (1) \textbf{Lie-$\SO(3)$-SSM} ($18{,}492$ params, $\exp(\Delta t_t A(u_t)) \in \SO(3)$); (2) \textbf{DeltaProduct ($n_h=2$)} ($20{,}524$ params, $H_1 H_2 \in \SO(3)$, Theorem~\ref{thm:master2_dimension_headcount}(c)(ii)); (3) \textbf{DeltaProduct ($n_h=1$)} ($17{,}452$ params, $\det(H) = -1$); (4) \textbf{RotRNN-$\SO(3)$} ($17{,}452$ params, 2D planar Givens); and (5) \textbf{Diag-SSM} ($19{,}852$ params).

\begin{table}[htbp]
\centering
\caption{Head-to-head empirical comparison on the non-solvable $A_5$ word problem (order 60, chance $= 1.67\%$). All evaluations use 5 fresh disjoint test seeds (700001--700005).}
\label{tab:a5_head_to_head}
\footnotesize
\setlength{\tabcolsep}{4.0pt}
\begin{tabular}{lrrrrr}
\toprule
\textbf{Configuration} & \textbf{Lie-$\SO(3)$-SSM} & \textbf{DeltaProduct ($n_h=2$)} & \textbf{DeltaProduct ($n_h=1$)} & \textbf{RotRNN-$\SO(3)$} & \textbf{Diag-SSM} \\
\midrule
Trainable Params & 18,492 & 20,524 & 17,452 & 17,452 & 19,852 \\
\midrule
\multicolumn{6}{l}{\textit{Part A1: Length Extrapolation (Last-Step Accuracy \% $\pm$ Std)}} \\
$T=16$ & $100.00 \pm 0.00$ & $100.00 \pm 0.00$ & $2.56 \pm 0.96$ & $3.20 \pm 1.41$ & $22.72 \pm 4.56$ \\
$T=32$ & $100.00 \pm 0.00$ & $100.00 \pm 0.00$ & $2.80 \pm 0.75$ & $1.84 \pm 0.46$ & $2.64 \pm 1.61$ \\
$T=64$ & $100.00 \pm 0.00$ & $100.00 \pm 0.00$ & $1.76 \pm 0.92$ & $2.00 \pm 0.69$ & $1.44 \pm 0.61$ \\
$T=128$ & $100.00 \pm 0.00$ & $100.00 \pm 0.00$ & $1.76 \pm 0.73$ & $2.16 \pm 0.78$ & $1.60 \pm 0.49$ \\
$T=256$ & $100.00 \pm 0.00$ & $99.60 \pm 0.57$ & $1.60 \pm 0.75$ & $2.32 \pm 1.21$ & $1.68 \pm 0.72$ \\
\midrule
\multicolumn{6}{l}{\textit{Part A2: Sample Efficiency at $T=32$ (Last-Step Accuracy \% $\pm$ Std)}} \\
$N=500$ & $9.84 \pm 6.12$ & $37.28 \pm 27.26$ & $1.92 \pm 0.91$ & $2.00 \pm 1.17$ & $1.28 \pm 0.77$ \\
$N=1000$ & $99.92 \pm 0.18$ & $100.00 \pm 0.00$ & $1.84 \pm 0.73$ & $1.52 \pm 0.52$ & $2.00 \pm 0.75$ \\
$N=2500$ & $100.00 \pm 0.00$ & $100.00 \pm 0.00$ & $1.76 \pm 0.96$ & $1.68 \pm 0.59$ & $5.52 \pm 2.09$ \\
$N=5000$ & $100.00 \pm 0.00$ & $100.00 \pm 0.00$ & $2.00 \pm 0.63$ & $1.52 \pm 0.72$ & $11.76 \pm 1.25$ \\
\midrule
\multicolumn{6}{l}{\textit{Part A3: Continuous Lie-Algebra Jitter Robustness at $T=32$ (Accuracy \% $\pm$ Std)}} \\
$\sigma=0.00$ & $100.00 \pm 0.00$ & $100.00 \pm 0.00$ & $2.80 \pm 0.75$ & $1.84 \pm 0.46$ & $2.64 \pm 1.61$ \\
$\sigma=0.05$ & $94.96 \pm 7.28$ & $99.76 \pm 0.54$ & $2.24 \pm 1.08$ & $1.76 \pm 0.96$ & $2.64 \pm 1.61$ \\
$\sigma=0.10$ & $48.72 \pm 13.14$ & $76.16 \pm 13.32$ & $1.84 \pm 0.73$ & $2.24 \pm 0.46$ & $2.64 \pm 1.61$ \\
\bottomrule
\end{tabular}
\end{table}

Table~\ref{tab:a5_head_to_head} confirms three results: (i) \texttt{Lie-SO(3)-SSM} sustains $100.00\% \pm 0.00\%$ across $T \in \{16, \dots, 256\}$ with $10\%$ fewer parameters than \texttt{DeltaProduct} ($n_h = 2$, $99.60\% \pm 0.57\%$ at $T=256$), whereas \texttt{DeltaProduct} ($n_h = 1$) collapses to $1.60\%$ due to per-step orientation reversal ($\det H(v) = -1$, alternating $\det = (-1)^t$) and rank-1 reflection restriction ($H(v) = I - 2vv^\top$ fixes a $(d-1)$-dimensional hyperplane per step, whereas generating $\SO(d)$ requires $n_h = 2\lfloor d/2\rfloor$ reflections per step by Theorem~\ref{thm:master2_dimension_headcount}(c)(ii)); (ii) \texttt{RotRNN-SO(3)} ($2.32\%$) and \texttt{Diag-SSM} ($1.68\%$ at $T=256$) fail to generalize ($p \le 1.85 \times 10^{-9}$, $d \ge 106.68$ at $T=64$); and (iii) while both \texttt{Lie-SO(3)-SSM} ($99.92\%$) and \texttt{DeltaProduct} ($n_h=2$, $100.00\%$) converge at $N \ge 1000$, \texttt{DeltaProduct} ($n_h=2$) achieves higher accuracy at $N=500$ ($37.28\%$ vs.\ $9.84\%$) and $\sigma = 0.10$ ($76.16\%$ vs.\ $48.72\%$).

\subsection{Diagnosing and Resolving the \texorpdfstring{$S_5$}{S5} Random-Initialization Bottleneck}
\label{sec:exp_s5}
By Theorem~\ref{thm:master1_division}(b)(iv), $S_5$ (order 120, chance $= 0.83\%$) violates $[u^{30}, v^{30}] = e$ and is obstructed at Depth 1 in $\SO(3\text{--}4)^m$. In an initial 4-way ablation (Table~\ref{tab:s5_legacy_ablation}), a 4-factor reflection-product parameterization (\texttt{Multilinear-SO(6)}, $R_\sigma = R_{s_0} R_{s_1} R_{s_2} R_{s_3}, R_{\sigma^{-1}} = R_{s_3} R_{s_2} R_{s_1} R_{s_0} \in \SO(6)$) stalled at $1.30\%$--$1.70\%$ from random initialization due to three bottlenecks (Theorems~\ref{thm:master3_cascades_gating}--\ref{thm:master4_optimization}): (i) at $R_{s_\ell} = \exp(\varepsilon \Delta_\ell)$, $R_{\sigma^{-1}} = R_\sigma + O(\varepsilon^2) \ne R_\sigma^\top$; (ii) 4-factor coupling $\prod_{\ell=0}^3 \exp(\Delta_\ell)$; and (iii) $1/60$ coset-centroid dilution in the unfactored 120-way readout.

\begin{table}[htbp]
\centering
\caption{Diagnostic 4-way initialization-prior ablation on $S_5$ (order 120, chance $= 0.83\%$) under the 4-factor reflection-product $\SO(6)$ parameterization (\texttt{Multilinear-SO(6)}) with dense per-step supervision and an unfactored 120-way readout.}
\label{tab:s5_legacy_ablation}
\footnotesize
\setlength{\tabcolsep}{3.5pt}
\begin{tabular}{@{}lrrrrr@{}}
\toprule
\textbf{Parameterization \& Initialization} & \textbf{Params} & \textbf{$T=16$ (\%)} & \textbf{$T=32$ (\%)} & \textbf{$T=64$ (\%)} & \textbf{$T=128$ (\%)} \\
\midrule
Multilin.-$\SO(6)$ (Rand.\ Init, Unconstr.) & 9,664 & $1.30 \pm 0.76$ & $1.30 \pm 0.67$ & $0.70 \pm 0.57$ & $0.60 \pm 0.42$ \\
Multilin.-$\SO(6)$ (Rand.\ Init, Inv.\ Reg.) & 9,664 & $1.70 \pm 0.45$ & $0.90 \pm 0.55$ & $0.70 \pm 0.76$ & $0.60 \pm 0.55$ \\
Multilin.-$\SO(6)$ (Coxeter Init, Frozen) & 9,424 & $54.00 \pm 17.36$ & $8.30 \pm 3.88$ & $1.70 \pm 0.57$ & $0.90 \pm 0.74$ \\
Multilin.-$\SO(6)$ (Coxeter Init, Tuned) & 9,664 & $100.00 \pm 0.00$ & $100.00 \pm 0.00$ & $99.80 \pm 0.45$ & $98.30 \pm 1.92$ \\
\midrule
DeltaProd.-$\SO(6)$ ($n_h=1$, Den.\ Unf.) & 9,496 & $7.60 \pm 1.47$ & $1.80 \pm 1.44$ & $2.50 \pm 0.79$ & $1.30 \pm 0.57$ \\
DeltaProd.-$\SO(6)$ ($n_h=2$, Den.\ Unf.) & 9,568 & $63.40 \pm 8.33$ & $28.20 \pm 10.03$ & $7.10 \pm 2.84$ & $2.30 \pm 1.15$ \\
DeltaProd.-$\SO(6)$ ($n_h=3$, Den.\ Unf.) & 9,640 & $67.50 \pm 11.45$ & $29.40 \pm 10.64$ & $7.00 \pm 3.89$ & $2.10 \pm 0.55$ \\
\bottomrule
\end{tabular}
\end{table}

To test learning $S_5$ from \textbf{100\% pure random Gaussian initialization} ($N = 2000, T_{\mathrm{train}} = 32$, batch size $64$, $40$ epochs, 5 training seeds $42$--$46$ and 5 disjoint test seeds $800001$--$800005$), Table~\ref{tab:s5_random_init_closure} evaluates $19$ configurations across four groups: (I) \textbf{Depth-1 $\SO(3)^{16}$ Obstruction} (\texttt{Lie-SO(3)-1L}, $14{,}302$ params) and \textbf{4-Factor Product} (\texttt{Multilinear-SO(6)}, $9{,}664$ params); (II) \textbf{$2 \times 2$ ($\text{Supervision} \times \text{Readout}$) Factorial on \texttt{Lie-SO(6)-1L}} ($15$ upper-triangular $\mathfrak{so}(6)$ parameters per head across Terminal vs.\ Dense Prefix and Unfactored $28{,}856$ vs.\ Factored $21{,}310$ params); (III) \textbf{Householder and Block-Diagonal Static-Orthogonal Controls on $\SO(6)^{16}$}: while generic rotations in $\SO(6)$ require $2\lfloor 6/2 \rfloor = 6$ reflections and $S_5 \hookrightarrow \OO(4) \hookrightarrow \SO(5) \hookrightarrow \SO(6)$ (Theorem~\ref{thm:master2_dimension_headcount}(a)(iii)) acts on two 2D planes (rank $2\lfloor 5/2 \rfloor = 4$), we evaluate isometric \textbf{\texttt{Orthogonal DeltaProduct-SO(6)}} ($\beta = 2$) and gated \textbf{\texttt{DeltaProduct-SO(6)}} ($\beta_{u,m,i} = 2\sigma(b_{u,m,i}) \in (0, 2)$, $b_0 = 2.0$)~\citep{grazzi2025unlocking,deltaproduct} at both $n_h = 4$ and $n_h = 6 = 2\lfloor 6/2 \rfloor$, plus a matched \textbf{\texttt{Block-Diag Radial expRNN-Inspired Control}} (\textbf{\texttt{expRNN-RadBlk}}, evolving $z_{t,m} = W_{\mathrm{orth},m} h_{t-1,m} + U_m x_t$ and $h_{t,m} = \mathrm{ReLU}(\|z_{t,m}\|_2 + b_m) z_{t,m}/\|z_{t,m}\|_2$ with static $W_{\mathrm{orth},m} = \exp(A_m - A_m^\top) \in \SO(6)$ and static scalar bias $b_m \in \R$ across $M=16$ heads [$d_{\mathrm{state}}=96$], which contrasts token-independent $W_{\mathrm{orth},m}$ plus additive input $U_m x_t$ against multiplicative token-selected $\exp(\Delta t_{t,m} A_m(u_t))$ while differing in three respects from the canonical dense expRNN $h_t = \sigma(R h_{t-1} + U x_t + b)$ of Lezcano-Casado \& Mart\'inez-Rubio~\citep{lezcano2019cheap}: (i) $M=16$ block-diagonal $\SO(6)$ recurrent blocks instead of a single dense $d_{\mathrm{hid}} \times d_{\mathrm{hid}}$ orthogonal matrix $R$, (ii) a shared per-head scalar bias $b_m \in \R$ instead of a per-coordinate bias vector $b_i \in \R^{d_{\mathrm{hid}}}$, and (iii) 6D blockwise radial norm gating $z_{t,m} \mapsto \mathrm{ReLU}(\|z_{t,m}\|_2 + b_m) z_{t,m}/\|z_{t,m}\|_2$ on the additive pre-activation $z_{t,m} = W_{\mathrm{orth},m} h_{t-1,m} + U_m x_t$ instead of coordinate-wise real $\mathrm{modReLU}$ $\operatorname{sign}(z_i)\mathrm{ReLU}(|z_i|+b_i)$); and (IV) \textbf{$2 \times 2$ Factorial and Normalization Dichotomy on Depth-2 \texttt{Lie-SO(3)-Cascade}} ($M_1=8, M_2=16$, differentiable router $p_{t-1} = \mathrm{Softmax}(W_{\mathrm{route}} \hat{h}^{(1)}_{t-1} / \tau_{\mathrm{route}})$ at $\tau_{\mathrm{route}} = 0.30$, unnormalized $\rho_1 = 0.8$ vs.\ Directional RMSNorm).

\begin{table}[htbp]
\centering
\caption{Controlled 19-model random-initialization evaluation on the non-solvable symmetric group $S_5$ (order 120, chance $= 0.83\%$) trained from \textbf{100\% pure random Gaussian initialization} (without analytic Coxeter priors, standardized $N=2000, T_{\mathrm{train}}=32, 40$ epochs) across 5 training seeds (42--46) and 5 disjoint test seeds (800001--800005), including complete $2 \times 2$ ($\text{Supervision} \times \text{Readout}$) factorials and \texttt{DeltaProduct} ($n_h \in \{4, 6\}$) / block-diagonal radial \texttt{expRNN-RadBlk} control comparisons. Both Last-Step Accuracy ($\%$) and All-Prefix-Step Accuracy ($\%$) are reported.}
\label{tab:s5_random_init_closure}
\footnotesize
\medmuskip=0mu\thinmuskip=0mu\thickmuskip=0mu
\setlength{\tabcolsep}{2.3pt}
\renewcommand{\arraystretch}{0.90}
\begin{tabular}{@{}lrrrrrr@{}}
\toprule
\textbf{Model} & \textbf{Par.} & \textbf{$\|\Delta\theta\|_2$} & \textbf{$T=16$} & \textbf{$T=32$} & \textbf{$T=64$} & \textbf{$T=128$} \\
\midrule
\multicolumn{7}{@{}l@{}}{\textit{Group I: Depth-1 $\SO(3)^{16}$ Obstruction \& 4-Factor Multilinear Baseline}} \\
$\SO(3)$ (Den.+Fac.) & 14,302 & 5.849 & $85.40 \pm 8.93$ ($96.14$) & $52.80 \pm 15.87$ ($81.73$) & $9.70 \pm 7.10$ ($53.22$) & $2.70 \pm 2.56$ ($28.40$) \\
Multilin.\ (4-Fac.) & 9,664 & 2.185 & $9.60 \pm 3.07$ ($42.57$) & $2.20 \pm 1.15$ ($23.67$) & $1.40 \pm 1.52$ ($13.03$) & $0.90 \pm 1.02$ ($6.83$) \\
\midrule
\multicolumn{7}{@{}l@{}}{\textit{Group II: $2 \times 2$ ($\text{Supervision} \times \text{Readout}$) Factorial on Depth-1 \texttt{Lie-$\SO(6)$-1L}}} \\
$\SO(6)$ (Term., Unf.) & 28,856 & 3.302 & $1.20 \pm 0.84$ ($0.78$) & $0.60 \pm 0.42$ ($0.84$) & $0.60 \pm 0.65$ ($0.83$) & $1.40 \pm 0.65$ ($0.81$) \\
$\SO(6)$ (Term., Fac.) & 21,310 & 3.659 & $0.70 \pm 0.57$ ($0.79$) & $0.80 \pm 0.27$ ($0.79$) & $0.70 \pm 0.76$ ($0.77$) & $0.50 \pm 0.87$ ($0.85$) \\
$\SO(6)$ (Den., Unf.) & 28,856 & 10.289 & $99.90 \pm 0.22$ ($99.98$) & $97.80 \pm 3.19$ ($99.57$) & $92.70 \pm 11.08$ ($97.49$) & $88.40 \pm 19.92$ ($94.07$) \\
\textbf{$\SO(6)$ (Den.+Fac.)} & \textbf{21,310} & \textbf{11.198} & $\mathbf{100.00 \pm 0.00}$ ($\mathbf{100.00}$) & $\mathbf{100.00 \pm 0.00}$ ($\mathbf{100.00}$) & $\mathbf{99.90 \pm 0.22}$ ($\mathbf{99.99}$) & $\mathbf{99.90 \pm 0.22}$ ($\mathbf{99.94}$) \\
\midrule
\multicolumn{7}{@{}l@{}}{\textit{Group III: $\SO(6)^{16}$ Householder (\texttt{DeltaProduct}, DP $n_h \in \{4, 6\}$) \& Block-Diag Radial (\texttt{expRNN-RadBlk})}} \\
DP-$4$ ($\beta{=}2$, Unf.) & 29,432 & 9.454 & $99.90 \pm 0.22$ ($99.95$) & $96.60 \pm 6.80$ ($99.08$) & $85.40 \pm 29.93$ ($94.75$) & $80.30 \pm 41.03$ ($88.81$) \\
DP-$4$ ($\beta{=}2$, Fac.) & 21,886 & 9.907 & $100.00 \pm 0.00$ ($100.00$) & $98.80 \pm 2.41$ ($99.82$) & $95.00 \pm 10.62$ ($98.45$) & $89.50 \pm 21.27$ ($95.21$) \\
\textbf{DP-$6$ ($\beta{=}2$, Fac.)} & \textbf{22,654} & \textbf{9.412} & $\mathbf{100.00 \pm 0.00}$ ($\mathbf{100.00}$) & $\mathbf{100.00 \pm 0.00}$ ($\mathbf{100.00}$) & $\mathbf{100.00 \pm 0.00}$ ($\mathbf{100.00}$) & $\mathbf{99.90 \pm 0.22}$ ($\mathbf{99.96}$) \\
DP-$4$ (Gat., Unf.) & 29,688 & 22.100 & $100.00 \pm 0.00$ ($100.00$) & $99.90 \pm 0.22$ ($99.99$) & $100.00 \pm 0.00$ ($99.99$) & $74.80 \pm 11.40$ ($96.07$) \\
DP-$4$ (Gat., Fac.) & 22,142 & 25.220 & $100.00 \pm 0.00$ ($100.00$) & $100.00 \pm 0.00$ ($100.00$) & $99.80 \pm 0.45$ ($99.98$) & $67.00 \pm 28.02$ ($93.74$) \\
DP-$6$ (Gat., Fac.) & 23,038 & 30.821 & $100.00 \pm 0.00$ ($100.00$) & $100.00 \pm 0.00$ ($100.00$) & $100.00 \pm 0.00$ ($100.00$) & $63.10 \pm 12.64$ ($94.64$) \\
expRNN-RadBlk (Unf.) & 28,536 & 1.589 & $3.10 \pm 1.43$ ($56.79$) & $0.90 \pm 0.22$ ($29.10$) & $1.20 \pm 0.91$ ($14.97$) & $1.00 \pm 0.50$ ($7.86$) \\
expRNN-RadBlk (Fac.) & 20,990 & 1.338 & $2.90 \pm 1.82$ ($32.41$) & $0.80 \pm 0.45$ ($16.87$) & $0.70 \pm 0.76$ ($9.01$) & $0.60 \pm 0.42$ ($4.87$) \\
\midrule
\multicolumn{7}{@{}l@{}}{\textit{Group IV: $2 \times 2$ ($\text{Supervision} \times \text{Readout}$) Factorial \& Norm.\ on Depth-2 \texttt{Lie-$\SO(3)$-Cascade}}} \\
Casc.\ (Unnorm.) & 14,542 & 12.712 & $97.00 \pm 3.10$ ($99.39$) & $51.00 \pm 14.60$ ($90.74$) & $2.20 \pm 0.57$ ($51.01$) & $1.30 \pm 1.04$ ($26.01$) \\
Casc.\ (Term., Unf.) & 25,354 & 3.022 & $1.00 \pm 0.71$ ($1.12$) & $1.30 \pm 1.60$ ($1.08$) & $1.20 \pm 0.57$ ($1.05$) & $1.20 \pm 0.57$ ($0.98$) \\
Casc.\ (Term., Fac.) & 14,542 & 3.712 & $0.70 \pm 0.45$ ($0.90$) & $0.60 \pm 0.55$ ($0.84$) & $1.00 \pm 0.35$ ($0.84$) & $0.80 \pm 0.57$ ($0.84$) \\
Casc.\ (Den., Unf.) & 25,354 & 8.904 & $94.70 \pm 6.18$ ($98.82$) & $73.50 \pm 17.98$ ($91.79$) & $38.60 \pm 35.90$ ($73.13$) & $24.40 \pm 41.55$ ($50.99$) \\
\textbf{Casc.\ (Den.+Fac.)} & \textbf{14,542} & \textbf{10.117} & $\mathbf{100.00 \pm 0.00}$ ($\mathbf{100.00}$) & $\mathbf{100.00 \pm 0.00}$ ($\mathbf{99.98}$) & $\mathbf{99.60 \pm 0.55}$ ($\mathbf{99.87}$) & $\mathbf{99.40 \pm 0.65}$ ($\mathbf{99.61}$) \\
\bottomrule
\end{tabular}
\end{table}

Table~\ref{tab:s5_random_init_closure} and Figure~\ref{fig:horizon_s5_d16}(a,c) establish four findings: (1) \texttt{Lie-SO(3)-1L} decays from $85.40\%$ ($T=16$) to $2.70\% \pm 2.56\%$ at $T=128$ ($p = 9.31 \times 10^{-8}$, Eq.~\eqref{eq:exp30_law}); (2) in the $2 \times 2$ factorials, terminal-only models remain at chance ($0.50\%$--$1.40\%$), dense unfactored models reach $88.40\% \pm 19.92\%$ (\texttt{Lie-SO(6)-1L}) and $24.40\% \pm 41.55\%$ (\texttt{Cascade}), and combining dense prefix supervision with factored $A_5 \rtimes C_2$ readout cuts parameters by $26.1\%$--$42.6\%$ while raising $T=128$ accuracy to $99.90\% \pm 0.22\%$ (\texttt{Lie-SO(6)-1L}) and $99.40\% \pm 0.65\%$ (\texttt{Cascade}, $+75.00\text{ pp}$); (3) in Group~III, \texttt{Orthogonal DeltaProduct-SO(6)} ($\beta = 2$) reaches $89.50\% \pm 21.27\%$ at $n_h = 4$ and $99.90\% \pm 0.22\%$ ($99.96\%$ all-step) at the full Cartan--Dieudonn\'e threshold $n_h = 6 = 2\lfloor 6/2 \rfloor$, whereas gated $\beta \in (0, 2)$~\citep{grazzi2025unlocking,deltaproduct} degrades at $T=128$ to $74.80\%/67.00\%$ ($n_h=4$) and $63.10\% \pm 12.64\%$ ($n_h=6$, $p = 0.0030$) due to $|1 - \beta_{u,m,i}| < 1$ contraction over $n_h T \in \{512, 768\}$ stages, and the block-diagonal radial control \texttt{expRNN-RadBlk} ($z_{t,m} = W_{\mathrm{orth},m} h_{t-1,m} + U_m x_t$, $h_{t,m} = \mathrm{ReLU}(\|z_{t,m}\|_2 + b_m) z_{t,m}/\|z_{t,m}\|_2$ with token-independent $W_{\mathrm{orth},m} = \exp(A_m - A_m^\top) \in \SO(6)$ and static scalar bias $b_m \in \R$) stalls at $0.60\%$--$1.00\%$ (strictly attributable to \texttt{expRNN-RadBlk} rather than to dense coordinate-wise $\mathrm{modReLU}$ expRNNs~\citep{lezcano2019cheap}); and (4) omitting Directional RMSNorm collapses the cascade from $97.00\%$ ($T=16$) to $1.30\%$ at $T=128$ ($p = 1.37 \times 10^{-8}$, Theorem~\ref{thm:master3_cascades_gating}(b)).

\subsection{Nilpotency-Class-3 Dihedral Group \texorpdfstring{$D_{16}$}{D16} Commutator Saddle-Escape Closure}
\label{sec:exp_d16}
On the nilpotency-class-3 dihedral group $D_{16} = \langle r, s \mid r^8 = s^2 = e, s r s = r^{-1} \rangle$ (order 16, $c_{\mathrm{nil}} = 3$: $\gamma_2 \cong C_4, \gamma_3 \cong C_2$; $N=1200, T_{\mathrm{train}}=24$, $20$ epochs, 5 training seeds $42$--$46$ and 5 disjoint test seeds $900001$--$900005$), Table~\ref{tab:d16_saddle_closure} evaluates two random walks: \textbf{Panel A ($\Sigma = \{r, s, r^{-1}\}$)}, where parity inversion restricts even-$T$ support to $8$ elements ($1/8 = 12.50\%$ baseline; on $D_{16}^{\mathrm{ab}} \cong C_2 \times C_2$, the two single-generator reflection characters $\chi_{(1,0)}, \chi_{(0,1)}$ have $|\hat{\mu}_{\mathrm{ab}}(\chi)| = 1/3$, whereas the bipartite parity character $\chi_{(1,1)}(\pi_{\mathrm{ab}}(a)) = -1 = \hat{\mu}_{\mathrm{ab}}(\chi_{(1,1)})$ is constant across $a \in \Sigma$ and has identically zero token-centered coefficient $\chi_{(1,1)}(\pi_{\mathrm{ab}}(a)) - \hat{\mu}_{\mathrm{ab}}(\chi_{(1,1)}) = 0$, leaving active token-discriminating spectral radius $\rho_{\mathrm{ab}} = 1/3$); and \textbf{Panel B (Lazy Walk $\Sigma_{\mathrm{lazy}} = \{e, r, s, r^{-1}\}$)}, where all $16$ elements are reachable ($1/16 = 6.25\%$ baseline, $|\hat{\mu}_{\mathrm{ab}}(\chi)| \le 1/2 = \rho_{\mathrm{ab}}$ for all $\chi \ne \mathbf{1}$). We compare \textbf{\texttt{Diag-SSM (Affine Commutative)}} ($17{,}688$ params), \textbf{\texttt{DeltaProduct-SO(3) ($n_h=1$)}} ($9{,}396$ params, $M = 12$), and the $2 \times 2$ ($\text{Terminal/Dense} \times \text{Unfactored } 10{,}176 / \text{Shared-Core } 4{+}4\text{ Factored } 8{,}632$) factorial on \textbf{\texttt{Lie-SO(3)-1L}} ($M = 12$).

\begin{table}[htbp]
\centering
\caption{Nilpotency-class-3 dihedral group $D_{16}$ (order 16, $c_{\mathrm{nil}} = 3$) commutator saddle-escape and $2 \times 2$ ($\text{Supervision} \times \text{Readout}$) factorial benchmark across 5 training seeds (42--46) and 5 disjoint test seeds (900001--900005), evaluating both the standard bipartite generator set $\Sigma = \{r, s, r^{-1}\}$ (even-horizon reachable support $= 8$, uniform reachable baseline $= 12.50\%$) and the non-bipartite lazy walk $\Sigma_{\mathrm{lazy}} = \{e, r, s, r^{-1}\}$ (all 16 states reachable, uniform baseline $= 6.25\%$).}
\label{tab:d16_saddle_closure}
\footnotesize
\medmuskip=0mu\thinmuskip=0mu\thickmuskip=0mu
\setlength{\tabcolsep}{3.2pt}
\renewcommand{\arraystretch}{0.90}
\begin{tabular}{@{}lrrrrrr@{}}
\toprule
\textbf{$D_{16}$ Model} & \textbf{Par.} & \textbf{$\|\Delta\theta\|_2$} & \textbf{$T=16$} & \textbf{$T=32$} & \textbf{$T=64$} & \textbf{$T=128$} \\
\midrule
\multicolumn{7}{@{}l@{}}{\textit{Panel A: Standard Walk $\Sigma = \{r, s, r^{-1}\}$ (Even-$T$ Support $= 8$, Prior $= 12.50\%$, Active $\rho_{\mathrm{ab}} = 1/3$)}} \\
Diag-SSM (Den.) & 17,688 & 3.266 & $99.60 \pm 0.42$ ($99.92$) & $89.60 \pm 1.78$ ($98.30$) & $29.50 \pm 4.77$ ($76.39$) & $4.30 \pm 5.25$ ($44.34$) \\
DP ($n_h{=}1$) & 9,396 & 3.197 & $26.30 \pm 4.18$ ($42.04$) & $25.10 \pm 2.88$ ($34.26$) & $25.40 \pm 2.10$ ($30.18$) & $23.90 \pm 1.75$ ($27.08$) \\
Lie (Term., Unf.) & 10,176 & 4.825 & $15.70 \pm 8.88$ ($8.49$) & $16.60 \pm 6.46$ ($7.90$) & $14.80 \pm 3.99$ ($7.54$) & $10.90 \pm 6.33$ ($6.58$) \\
Lie (Term., $4{+}4$) & 8,632 & 3.769 & $15.30 \pm 6.02$ ($8.20$) & $16.90 \pm 6.17$ ($8.55$) & $15.50 \pm 4.72$ ($7.83$) & $15.60 \pm 7.30$ ($7.66$) \\
\textbf{Lie (Den., Unf.)} & \textbf{10,176} & \textbf{3.996} & $\mathbf{100.00 \pm 0.00}$ ($\mathbf{100.00}$) & $\mathbf{100.00 \pm 0.00}$ ($\mathbf{100.00}$) & $\mathbf{100.00 \pm 0.00}$ ($\mathbf{100.00}$) & $\mathbf{100.00 \pm 0.00}$ ($\mathbf{100.00}$) \\
\textbf{Lie (Den.+$4{+}4$)} & \textbf{8,632} & \textbf{4.263} & $\mathbf{100.00 \pm 0.00}$ ($\mathbf{100.00}$) & $\mathbf{100.00 \pm 0.00}$ ($\mathbf{100.00}$) & $\mathbf{100.00 \pm 0.00}$ ($\mathbf{100.00}$) & $\mathbf{100.00 \pm 0.00}$ ($\mathbf{100.00}$) \\
\midrule
\multicolumn{7}{@{}l@{}}{\textit{Panel B: Lazy Walk $\Sigma_{\mathrm{lazy}} = \{e, r, s, r^{-1}\}$ ($16$ Reachable Classes, Prior $= 6.25\%$, $\rho_{\mathrm{ab}} = 1/2$)}} \\
Diag-SSM (Den.) & 17,688 & 3.040 & $99.90 \pm 0.22$ ($99.97$) & $87.70 \pm 6.29$ ($97.91$) & $22.60 \pm 13.54$ ($72.44$) & $9.80 \pm 5.16$ ($43.27$) \\
Lie (Term., Unf.) & 10,176 & 4.936 & $6.90 \pm 1.39$ ($6.27$) & $5.60 \pm 1.75$ ($6.61$) & $6.40 \pm 1.43$ ($6.41$) & $6.20 \pm 1.72$ ($6.34$) \\
Lie (Term., $4{+}4$) & 8,632 & 5.007 & $13.60 \pm 12.71$ ($10.20$) & $10.40 \pm 5.91$ ($9.53$) & $10.90 \pm 9.05$ ($8.88$) & $8.10 \pm 7.17$ ($7.64$) \\
\textbf{Lie (Den., Unf.)} & \textbf{10,176} & \textbf{3.530} & $\mathbf{100.00 \pm 0.00}$ ($\mathbf{100.00}$) & $\mathbf{100.00 \pm 0.00}$ ($\mathbf{100.00}$) & $\mathbf{100.00 \pm 0.00}$ ($\mathbf{100.00}$) & $\mathbf{100.00 \pm 0.00}$ ($\mathbf{100.00}$) \\
\textbf{Lie (Den.+$4{+}4$)} & \textbf{8,632} & \textbf{3.860} & $\mathbf{100.00 \pm 0.00}$ ($\mathbf{100.00}$) & $\mathbf{100.00 \pm 0.00}$ ($\mathbf{100.00}$) & $\mathbf{100.00 \pm 0.00}$ ($\mathbf{100.00}$) & $\mathbf{99.90 \pm 0.22}$ ($\mathbf{100.00}$) \\
\bottomrule
\end{tabular}
\end{table}

Table~\ref{tab:d16_saddle_closure} and Figure~\ref{fig:horizon_s5_d16}(b) show that: (1) terminal-only training stalls near the $12.50\%$ bipartite reachable prior ($10.90\%/15.60\%$ at $T=128$, Panel~A) and $6.25\%$ uniform prior ($6.20\%/8.10\%$, Panel~B) due to $(1/3)^{24}$ and $(1/2)^{24}$ Fourier--Markov decay; (2) because $|[D_{16}, D_{16}]| = 4$ induces only $1/4$ dilution, dense prefix supervision reaches $100.00\% \pm 0.00\%$ with both the unfactored ($10{,}176$ params) and shared-core $4+4$ factored ($8{,}632$ params, $99.90\% \pm 0.22\%$ at $T=128$ lazy) readouts; and (3) \texttt{Diag-SSM} collapses to $4.30\%/9.80\%$ at $T=128$, while \texttt{DeltaProduct-SO(3)} ($n_h=1$) saturates near $23.90\% \approx 1/|C_4| = 25\%$.

\begin{figure*}[t!]
\centering
\includegraphics[width=0.96\textwidth]{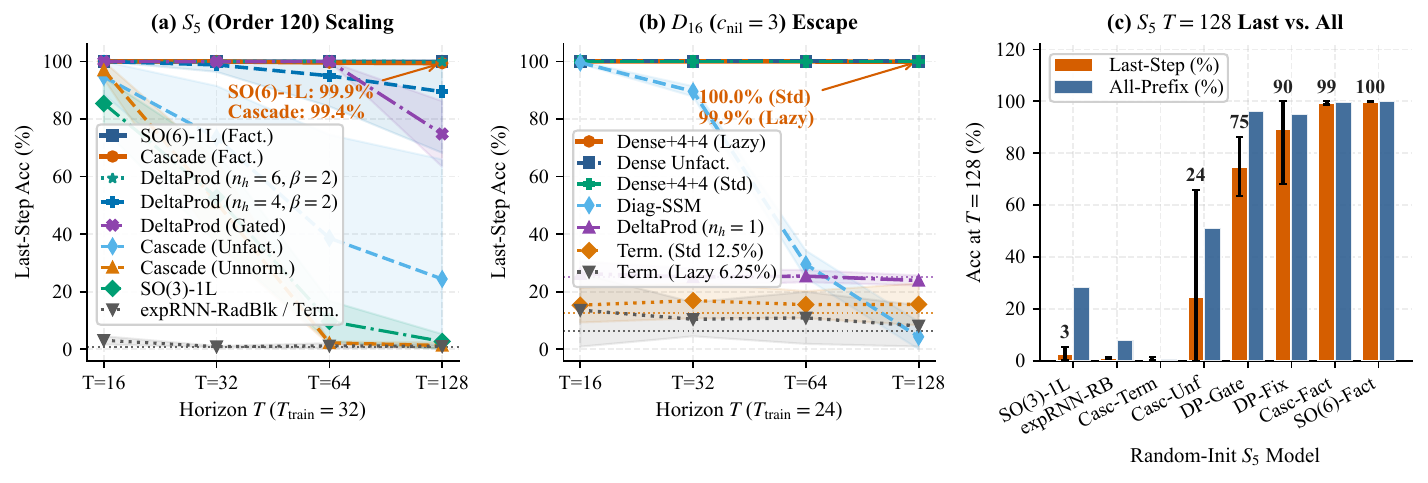}
\caption{Random-initialization horizon scaling, $2 \times 2$ ($\text{Supervision} \times \text{Readout}$) factorial attribution, and commutator saddle escape across $S_5$ (order 120) and $D_{16}$ (order 16, $c_{\mathrm{nil}} = 3$) over 5 seeds. \textbf{(a)} $S_5$ across $T \in \{16, 32, 64, 128\}$. \textbf{(b)} $D_{16}$ across $T \in \{16, 32, 64, 128\}$ on standard ($\Sigma$) and lazy ($\Sigma_{\mathrm{lazy}}$) walks. \textbf{(c)} $T=128$ accuracy and generator displacement $\|\Delta\theta_{\mathrm{gen}}\|_2$ on $S_5$ (see Table~\ref{tab:s5_random_init_closure} for the full-rank $n_h = 6 = 2\lfloor 6/2 \rfloor$ \texttt{DeltaProduct-SO(6)} controls: $99.90\% \pm 0.22\%$ fixed $\beta=2$ and $63.10\% \pm 12.64\%$ gated $\beta \in (0,2)$).}
\label{fig:horizon_s5_d16}
\end{figure*}

\subsection{Formal Language Counter Depth (Dyck-2) and Controlled Synthetic 6-DoF Inertial Pose Integration (\texorpdfstring{$\SE_2(3)$}{SE2(3)})}
\label{sec:exp_dyck2_se3}
Theorem~\ref{thm:master3_cascades_gating}(b)(iii) separates polynomial-growth abelian counters $\ZZ^r$ from exponential-growth free-group reduced-word balls $F_2 = \langle a, b \rangle \subset \SO(3)$ over $\{a, a^{-1}, b, b^{-1}\}$ ($\beta_{F_2}(W) = 2 \cdot 3^W - 1$, where compact $\SO(3)$ representations experience cap-packing decay $\Delta_{\mathrm{norm}} \le 4\sqrt{m} 3^{-W/(2m)}$ without non-compact counters or hybrid KV-cache checkpoints) and bounded-depth binary Dyck-2 LIFO stacks ($\sum_{d=0}^D 2^d = 2^{D+1}-1$ distinct stack configurations up to nesting depth $D$, i.e., $31$ states at $D=4$ up to $131{,}071$ states at $D=16$). Table~\ref{tab:dyck2_se3_combined} and Figure~\ref{fig:a5_se3_scan}(b) validate this duality across 5 seeds on bounded-depth Dyck-2 stack tracking (Part~A: $T=64$, $D \in \{4, 8, 12, 16\}$, where each input sequence presents typed opening brackets $\{0, 1\}$ and a single generic closing token [token ID $2$] and models predict the matching bracket type $\{0, 1\}$ at every closing position with opening positions masked out) and a controlled synthetic 6-DoF strapdown inertial odometry simulation (Part~B).

\begin{table}[htbp]
\centering
\caption{Geometric representation trade-off between compact $\SO(3)$ rotations, scalar counters, and non-compact extended-pose manifolds $\SE_2(3) / \SE(3)$. Part A: Dyck-2 context-free language accuracy across nesting depths $D \in \{4, 8, 12, 16\}$. Part B: Controlled synthetic 6-DoF rigid-body strapdown inertial odometry simulation benchmark under continuous $20\times$-substepped high-rate aerobatic motion ($50\text{ Hz}, \Delta t = 0.02\text{ s}, T=256$, 5 seeds) across matched parameter tiers ($0, 679, \sim 50\mathrm{K}, 81.9\mathrm{K}$).}
\label{tab:dyck2_se3_combined}
\footnotesize
\medmuskip=0mu\thinmuskip=0mu\thickmuskip=1mu
\setlength{\tabcolsep}{3.5pt}
\renewcommand{\arraystretch}{0.90}
\begin{tabular}{@{}lrrrrr@{}}
\toprule
\multicolumn{6}{@{}l@{}}{\textbf{Part A: Dyck-2 Context-Free Language Tracking Across Nesting Depths $D \in \{4, 8, 12, 16\}$}} \\
\midrule
\textbf{Model} & \textbf{Params} & \textbf{$D=4$ (\%)} & \textbf{$D=8$ (\%)} & \textbf{$D=12$ (\%)} & \textbf{$D=16$ (\%)} \\
\midrule
Diagonal-1L (Non-Compact Counter) & 5,570 & $98.31 \pm 0.11$ & $95.39 \pm 0.33$ & $93.92 \pm 0.46$ & $93.66 \pm 0.34$ \\
Diagonal-3L (Multi-Layer Counter) & 24,738 & $98.96 \pm 0.55$ & $96.04 \pm 1.03$ & $94.57 \pm 0.21$ & $94.64 \pm 0.89$ \\
Lie-$\SO(3)$-SSM (Compact Sphere $S^2$) & 5,538 & $81.97 \pm 2.07$ & $80.41 \pm 2.42$ & $79.11 \pm 1.11$ & $76.88 \pm 2.16$ \\
NCSSM-Bloch (Contracting Bloch Ball) & 2,954 & $76.21 \pm 2.05$ & $74.00 \pm 1.37$ & $74.38 \pm 2.71$ & $72.55 \pm 0.83$ \\
\midrule
\multicolumn{6}{l}{\textbf{Part B: Synthetic 6-DoF Rigid-Body Inertial Odometry ($50\text{ Hz}$ $20\times$-SubStepped Aerobatic, $T=256$)}} \\
\midrule
\textbf{Estimator \& Operator} & \textbf{Params} & \textbf{Pos ATE (m)} & \textbf{Geodesic ($^\circ$)} & \textbf{$\|R^\top R - I\|_F$} & \textbf{$|\det R - 1|$} \\
\midrule
Open-Loop DR (Euler-Rodrigues) & 0 & $7.324 \pm 0.503$ & $16.58 \pm 0.79$ & $(1.25 \pm 0.04) \times 10^{-6}$ & $(5.50 \pm 0.50) \times 10^{-7}$ \\
Open-Loop DR ($\SE_2(3)$ Left-Jac.) & 0 & $7.074 \pm 0.518$ & $16.58 \pm 0.79$ & $(1.25 \pm 0.04) \times 10^{-6}$ & $(5.50 \pm 0.50) \times 10^{-7}$ \\
Lie-$\SO(3) + \R^3$ (MLP-679, Euler) & 679 & $3.165 \pm 0.111$ & $8.52 \pm 0.36$ & $(1.23 \pm 0.01) \times 10^{-6}$ & $(5.35 \pm 0.33) \times 10^{-7}$ \\
\textbf{Lie-$\SE(3)$-LeftJac (MLP-679)} & \textbf{679} & $\mathbf{2.636 \pm 0.167}$ & $\mathbf{8.36 \pm 0.24}$ & $(1.24 \pm 0.03) \times 10^{-6}$ & $(5.59 \pm 0.33) \times 10^{-7}$ \\
Diagonal-3L (Unconstrained Pose) & 50,236 & $37.385 \pm 0.497$ & $102.63 \pm 2.70$ & $1.42 \pm 0.04$ & $0.945 \pm 0.007$ \\
Unconstrained-GRU-SVD (SVD Proj.) & 49,468 & $33.815 \pm 0.553$ & $98.92 \pm 1.24$ & $(1.53 \pm 0.14) \times 10^{-7}$ & $(9.00 \pm 0.27) \times 10^{-8}$ \\
Lie-$\SO(3) + \R^3$ (SeqCalib-50.2K, Euler) & 49,575 & $3.137 \pm 0.176$ & $9.31 \pm 0.55$ & $(1.22 \pm 0.03) \times 10^{-6}$ & $(5.26 \pm 0.23) \times 10^{-7}$ \\
Lie-$\SE(3)$-LeftJac (SeqCalib-50.2K) & 49,575 & $3.124 \pm 0.295$ & $8.96 \pm 0.45$ & $(1.24 \pm 0.01) \times 10^{-6}$ & $(5.36 \pm 0.25) \times 10^{-7}$ \\
Lie-$\SO(3) + \R^3$ (SeqCalib-81.9K, Euler) & 81,927 & $3.150 \pm 0.251$ & $9.40 \pm 0.55$ & $(1.22 \pm 0.01) \times 10^{-6}$ & $(5.35 \pm 0.13) \times 10^{-7}$ \\
Lie-$\SE(3)$-LeftJac (SeqCalib-81.9K) & 81,927 & $3.111 \pm 0.269$ & $9.09 \pm 0.62$ & $(1.25 \pm 0.01) \times 10^{-6}$ & $(5.60 \pm 0.17) \times 10^{-7}$ \\
\bottomrule
\end{tabular}
\end{table}

On Dyck-2 (Part A), non-compact counter models (\texttt{Diagonal-1L}/\texttt{3L}) augment $h_t = \lambda_t \odot h_{t-1} + B u_t$ with pop-gate modulation $h_t \gets h_t \odot \gamma_{\mathrm{pop}}(u_t) + b_{\mathrm{pop}}(u_t)$ and maintain $94.64\% \pm 0.89\%$ at depth $D=16$, whereas compact $\SO(3)$ rotations saturate at $76.88\% \pm 2.16\%$ due to Archimedes cap crowding on $S^2$. Conversely, on synthetic 6-DoF strapdown odometry~\citep{barrau2017invariant} ($50\text{ Hz}, \Delta t = 0.02\text{ s}, T=256$, Part B and Figure~\ref{fig:a5_se3_scan}(b)), we de-confound parameter capacity ($0, 679, \sim 50\mathrm{K}, 81.9\mathrm{K}$) from integration geometry by comparing decoupled Euler integration ($R_t = R_{t-1}\exp(\hat{\bm{\theta}}_t)$, $\mathbf{v}_t = \mathbf{v}_{t-1} + (\mathbf{g} + R_{t-1}\hat{\mathbf{f}}_t)\Delta t$, $\mathbf{p}_t = \mathbf{p}_{t-1} + \mathbf{v}_{t-1}\Delta t + \frac{1}{2}(\mathbf{g} + R_{t-1}\hat{\mathbf{f}}_t)\Delta t^2$) against coupled $\SE_2(3)$ Left-Jacobian integration ($\Delta\mathbf{v}_t = J_l(\bm{\theta}_t)\hat{\mathbf{f}}_t\Delta t$, $\Delta\mathbf{p}_t = J_{l,2}(\bm{\theta}_t)\hat{\mathbf{f}}_t\Delta t^2$ with $J_l(\bm{\theta}) = I_3 + \frac{1-\cos\theta}{\theta^2}\hat{\bm{\theta}} + \frac{\theta-\sin\theta}{\theta^3}\hat{\bm{\theta}}^2$ and $J_{l,2}(\bm{\theta}) = \frac{1}{2}I_3 + \frac{\theta-\sin\theta}{\theta^3}\hat{\bm{\theta}} + \frac{\theta^2+2\cos\theta-2}{2\theta^4}\hat{\bm{\theta}}^2$). The $49{,}575$-parameter \texttt{Lie-SE(3)-LeftJac (SeqCalib-50.2K)} achieves $3.124 \pm 0.295\text{ m}$ ATE ($8.96 \pm 0.45^\circ$), cutting position error by $>10\times$ vs.\ \texttt{Diagonal-3L} ($37.385 \pm 0.497\text{ m}, 102.63 \pm 2.70^\circ$ via Kabsch--Umeyama SVD, $p_{\mathrm{Holm}} = 1.42 \times 10^{-7}$) and \texttt{Unconstrained-GRU-SVD} ($33.815 \pm 0.553\text{ m}, p_{\mathrm{Holm}} = 2.81 \times 10^{-7}$) with orthogonality drift $1.24 \times 10^{-6}$. Within the Lie family, $\SE_2(3)$ Left-Jacobian coupling outperforms Euler at $0$ parameters ($-0.250\text{ m}$, $p_{\mathrm{Holm}} = 0.0117$) and $679$ parameters (\texttt{MLP-679}: $2.636 \pm 0.167\text{ m}, 8.36 \pm 0.24^\circ$ vs.\ Euler $3.165 \pm 0.111\text{ m}, 8.52 \pm 0.36^\circ$, $-0.529\text{ m} / -16.7\%$, $t = -14.14, p_{\mathrm{Holm}} = 5.81 \times 10^{-4}$), whereas $49.6\mathrm{K}$--$81.9\mathrm{K}$ calibrators ($3.124\text{ m}, 3.111\text{ m}$) partially absorb $(J_l(\bm{\theta}_t)-I_3)\mathbf{f}_t$ into learned force corrections (and match at $0.574\text{ m}$ vs.\ $0.578\text{ m}$ on $100\text{ Hz}$ 1-step trajectories).

\subsection{Associative Parallel Prefix Scan Hardware Scaling}
\label{sec:exp_hardware}
Table~\ref{tab:scan_benchmark} and Figure~\ref{fig:a5_se3_scan}(c) benchmark the work-efficient \textbf{Blelloch} scan (Algorithm~\ref{alg:lie_prefix_scan}, $2\log_2 T$ launches, $\mathcal{O}(T)$ batched-GEMM multiplications) and the step-optimal \textbf{Hillis--Steele} doubling scan ($\log_2 T$ launches, $\mathcal{O}(T\log_2 T)$ batched-GEMM multiplications) against $\mathcal{O}(T)$ sequential recurrence on $\SE(3)$ ($B=16$, IEEE-754 \texttt{float64}, 5 seeds) across $T \in \{64, 256, 1024, 2048\}$ on matched PyTorch CUDA GPU (RTX 4080 Laptop GPU) and 4-thread CPU backends. Each transition matrix uses the Taylor-stabilized $\mathrm{sinc}$ $\mathfrak{se}(3)$ exponential $p_{t,m} = V(\bm{\omega}_{t,m}\Delta t)\mathbf{v}_{t,m}\Delta t$, and every full-layer step timer re-initializes the calibrator, linear readout, and \texttt{AdamW} optimizer from the identical seed. We evaluate \textbf{Workload~A ($B=16, M=1 \implies 16$ streams, single-head)} and \textbf{Workload~B ($B=16, M=16 \implies 256$ streams, multi-head)}.

\begin{table}[htbp]
\centering
\caption{De-confounded associative parallel prefix scan benchmark on $\SE(3)$ (batch size $B=16$, IEEE-754 \texttt{float64}, 5 seeds $\{42, 123, 456, 789, 2026\}$) across sequence horizons $T \in \{64, \dots, 2048\}$, comparing $\mathcal{O}(T)$ Sequential recurrence against the $\mathcal{O}(T\log_2 T)$-work / $\log_2 T$-depth Hillis--Steele doubling scan and the $\mathcal{O}(T)$-work / $2\log_2 T$-depth Blelloch up-sweep/down-sweep scan (Algorithm~\ref{alg:lie_prefix_scan}) on identical PyTorch CUDA GPU and 4-thread CPU backends with exact Left-Jacobian twist $p = V(\bm{\omega}\Delta t)\mathbf{v}\Delta t$ and fresh identical-seed timer initialization across Workload~A ($B=16, M=1 \implies 16$ streams) and Workload~B ($B=16, M=16 \implies 256$ streams). The \texttt{Full-Layer Peak VRAM} column reports end-to-end full-layer training-step peak autograd VRAM (isolated scan-primitive Fwd+Bwd peak VRAM is $121.88 / 196.87 / 133.88\text{ MB}$ [$-32.0\%$] on Workload~A and $1705.71 / 2860.64 / 1897.67\text{ MB}$ [$-33.7\%$] on Workload~B at $T=2048$).}
\label{tab:scan_benchmark}
\footnotesize
\medmuskip=0mu\thinmuskip=0mu\thickmuskip=1mu
\setlength{\tabcolsep}{2.5pt}
\renewcommand{\arraystretch}{0.90}
\begin{tabular}{l rrr rrr}
\toprule
\multicolumn{7}{l}{\textbf{Part I: GPU Primitive Scan Forward and Scan Forward+Backward Latency (ms) [Speedup]}} \\
\midrule
& \multicolumn{3}{c}{\textbf{GPU Primitive Scan (ms) [Speedup]}} & \multicolumn{3}{c}{\textbf{GPU Scan Fwd+Bwd (ms) [Speedup]}} \\
\cmidrule(lr){2-4} \cmidrule(l){5-7}
\textbf{Workload / $T$} & \textbf{Seq.} & \textbf{Hillis--Steele} & \textbf{Blelloch} & \textbf{Seq.} & \textbf{Hillis--Steele} & \textbf{Blelloch} \\
\midrule
Workload A ($M=1$): $64$ & 2.22 & $\mathbf{0.36}$ [$\mathbf{6.10\times}$] & 0.84 [$2.63\times$] & 16.39 & $\mathbf{5.55}$ [$\mathbf{2.95\times}$] & 7.22 [$2.27\times$] \\
Workload A ($M=1$): $256$ & 8.82 & $\mathbf{0.48}$ [$\mathbf{18.44\times}$] & 1.07 [$8.22\times$] & 54.89 & $\mathbf{8.75}$ [$\mathbf{6.27\times}$] & 9.38 [$5.85\times$] \\
Workload A ($M=1$): $1024$ & 36.46 & 1.37 [$26.71\times$] & $\mathbf{1.33}$ [$\mathbf{27.37\times}$] & 218.89 & 24.77 [$8.84\times$] & $\mathbf{15.24}$ [$\mathbf{14.37\times}$] \\
Workload A ($M=1$): $2048$ & 76.44 & 2.90 [$26.32\times$] & $\mathbf{1.58}$ [$\mathbf{48.24\times}$] & 433.44 & 48.81 [$8.88\times$] & $\mathbf{21.60}$ [$\mathbf{20.06\times}$] \\
\midrule
Workload B ($M=16$): $64$ & 2.36 & $\mathbf{0.77}$ [$\mathbf{3.08\times}$] & 0.86 [$2.74\times$] & 19.06 & 16.13 [$1.18\times$] & $\mathbf{12.05}$ [$\mathbf{1.58\times}$] \\
Workload B ($M=16$): $256$ & 9.28 & 3.97 [$2.34\times$] & $\mathbf{1.48}$ [$\mathbf{6.27\times}$] & 76.20 & 71.73 [$1.06\times$] & $\mathbf{37.34}$ [$\mathbf{2.04\times}$] \\
Workload B ($M=16$): $1024$ & 39.64 & 27.59 [$1.44\times$] & $\mathbf{5.61}$ [$\mathbf{7.07\times}$] & 851.66 & 357.82 [$2.38\times$] & $\mathbf{160.63}$ [$\mathbf{5.30\times}$] \\
Workload B ($M=16$): $2048$ & 72.87 & 63.85 [$1.14\times$] & $\mathbf{12.37}$ [$\mathbf{5.89\times}$] & 3500.22 & 790.43 [$4.43\times$] & $\mathbf{342.03}$ [$\mathbf{10.23\times}$] \\
\bottomrule
\end{tabular}
\vspace{1.0pt}
\begin{tabular}{l c c c}
\toprule
\multicolumn{4}{l}{\textbf{Part II: Full-Layer Step Latency (ms) [Speedup], Full-Layer Peak VRAM (MB) \& Diagnostics}} \\
\midrule
\textbf{Wkld / $T$} & \textbf{Full-Layer Step (Seq.\ / H.--S.\ / Blelloch)} & \textbf{Full-Layer Peak VRAM (MB)} & \textbf{CPU Fwd / Grad Diff} \\
\midrule
Wkld A: $64$ & 18.95 / $\mathbf{5.68}$ [$\mathbf{3.33\times}$] / 7.55 [$2.51\times$] & 20.04 / 21.05 / \textbf{20.41} & $2.07\times$ / $\mathbf{2.11\times}$ \\
Wkld A: $256$ & 56.75 / $\mathbf{7.94}$ [$\mathbf{7.15\times}$] / 9.67 [$5.87\times$] & 30.83 / 36.84 / \textbf{32.33} & $2.18\times$ / $\mathbf{3.71\times}$ \\
Wkld A: $1024$ & 223.51 / 24.74 [$9.04\times$] / $\mathbf{13.62}$ [$\mathbf{16.41\times}$] & 75.39 / 109.25 / \textbf{81.39} & $2.90\times$ / $\mathbf{6.63\times}$ \\
Wkld A: $2048$ & 444.44 / 48.87 [$9.09\times$] / $\mathbf{21.92}$ [$\mathbf{20.27\times}$] & 132.69 / 207.02 / \textbf{144.83} [$-30.0\%$] & $3.51\times$ / $\mathbf{8.41\times}$ \\
\midrule
Wkld B: $64$ & 21.00 / 16.04 [$1.31\times$] / $\mathbf{11.39}$ [$\mathbf{1.84\times}$] & 76.74 / 94.37 / \textbf{83.01} & $8.3 \times 10^{-17} / 1.1 \times 10^{-16}$ \\
Wkld B: $256$ & 71.14 / 71.75 [$0.99\times$] / $\mathbf{37.39}$ [$\mathbf{1.90\times}$] & 255.46 / 352.31 / \textbf{280.93} & $1.4 \times 10^{-16} / 1.4 \times 10^{-16}$ \\
Wkld B: $1024$ & 850.84 / 357.64 [$2.38\times$] / $\mathbf{160.78}$ [$\mathbf{5.29\times}$] & 961.15 / 1477.24 / \textbf{1057.11} & $2.7 \times 10^{-15} / 2.7 \times 10^{-15}$ \\
Wkld B: $2048$ & 3509.33 / 791.87 [$4.43\times$] / $\mathbf{341.41}$ [$\mathbf{10.28\times}$] & 1905.90 / 3061.99 / \textbf{2097.86} [$-31.5\%$] & $8.0 \times 10^{-15} / 8.4 \times 10^{-15}$ \\
\bottomrule
\end{tabular}
\end{table}

Table~\ref{tab:scan_benchmark} maintains forward error $\le 9.37 \times 10^{-14}\text{ m}$ (Hillis--Steele) and $\le 9.00 \times 10^{-14}\text{ m}$ (Blelloch) at $T=2048$, and backward autograd gradient difference $\max\|\nabla_{\mathrm{par}} - \nabla_{\mathrm{seq}}\|_\infty \le 5.33 \times 10^{-15}$ (Workload~A) and $\le 8.44 \times 10^{-15}$ (Workload~B): (1) \emph{Workload~A ($M=1$)}: at $T \in \{64, 256\}$, $\log_2 T$ kernel launch depth favors Hillis--Steele ($6.10\times$ and $18.44\times$ primitive forward speedup; $3.33\times$ and $7.15\times$ full-layer step speedup). By $T \ge 1024$, $\mathcal{O}(BT\log_2 T)$ work overtakes launch overhead: at $T=2048$, Blelloch reaches $48.24\times$ primitive forward speedup ($1.58\text{ ms}$ vs.\ $2.90\text{ ms}$ [$26.32\times$] and $76.44\text{ ms}$), $20.06\times$ scan Fwd+Bwd speedup ($21.60\text{ ms}$ vs.\ $48.81\text{ ms}$ [$8.88\times$] and $433.44\text{ ms}$), and $20.27\times$ full-layer step speedup ($21.92\text{ ms}$ vs.\ $48.87\text{ ms}$ [$9.09\times$] and $444.44\text{ ms}$), while cutting full-layer peak VRAM by $30.0\%$ ($144.83\text{ MB}$ vs.\ $207.02\text{ MB}$; isolated scan-primitive VRAM $133.88\text{ MB}$ vs.\ $196.87\text{ MB}$ [$-32.0\%$]); (2) \emph{Workload~B ($M=16$)}: with $256$ concurrent $\SE(3)$ streams, Blelloch outperforms sequential and Hillis--Steele across all horizons, reaching $5.89\times$ primitive forward ($12.37\text{ ms}$ vs.\ $63.85\text{ ms}$ [$1.14\times$] and $72.87\text{ ms}$), $10.23\times$ scan Fwd+Bwd ($342.03\text{ ms}$ vs.\ $790.43\text{ ms}$ [$4.43\times$] and $3500.22\text{ ms}$), and $10.28\times$ full-layer training step speedup at $T=2048$ ($341.41\text{ ms}$ vs.\ $791.87\text{ ms}$ [$4.43\times$] and $3509.33\text{ ms}$), cutting full-layer peak VRAM by $31.5\%$ ($2097.86\text{ MB}$ vs.\ $3061.99\text{ MB}$; scan-primitive VRAM $1897.67\text{ MB}$ vs.\ $2860.64\text{ MB}$ [$-33.7\%$]); and (3) \emph{4-thread CPU}: Blelloch rises from $2.11\times$ at $T=64$ to $8.41\times$ vs.\ $3.51\times$ at $T=2048$ ($24.54\text{ ms}$ vs.\ $58.50\text{ ms}$, sequential $205.39\text{ ms}$).

\begin{figure*}[t!]
\centering
\includegraphics[width=0.96\textwidth]{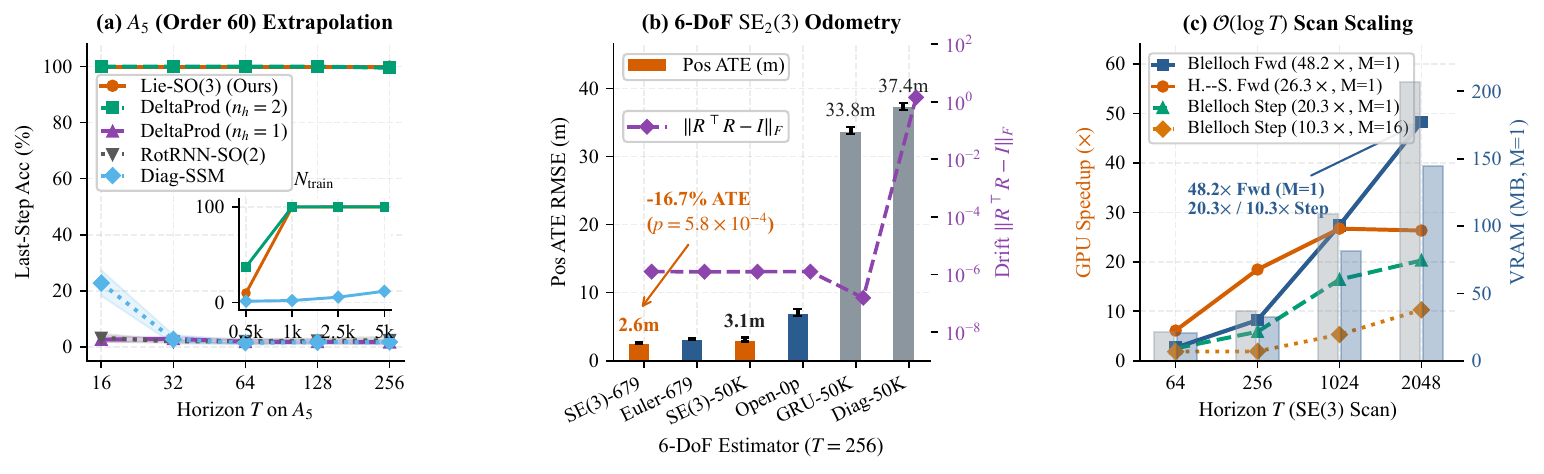}
\caption{Non-solvable $A_5$ extrapolation, de-confounded 6-DoF $\SE_2(3)$ rigid-body strapdown pose integration, and $\mathcal{O}(\log T)$ associative parallel prefix scan scaling (Hillis--Steele vs.\ Blelloch). \textbf{(a)} 5-model head-to-head horizon extrapolation ($T \in \{16, \dots, 256\}$) and sample efficiency inset ($N_{\mathrm{train}} \in \{500, \dots, 5000\}$) on $A_5$ (order 60). \textbf{(b)} De-confounded 6-DoF rigid-body inertial pose integration ($50\text{ Hz}$ $20\times$-substepped aerobatic trajectories, $T=256$), comparing position ATE RMSE (left axis) and log-scale orthogonality drift $\|R^\top R - I\|_F$ (right axis) across $0, 679,$ and $50\mathrm{K}$ parameter tiers. \textbf{(c)} Associative parallel prefix scan GPU/CPU forward speedup, end-to-end training step speedup ($48.24\times$ Blelloch vs.\ $26.32\times$ Hillis--Steele forward, $20.06\times$ vs.\ $8.88\times$ scan Fwd+Bwd, and $20.27\times$ vs.\ $9.09\times$ full-layer step on Workload~A, and $10.28\times$ Blelloch vs.\ $4.43\times$ Hillis--Steele full-layer step on Workload~B at $T=2048$), and peak autograd VRAM across horizons $T \in \{64, 256, 1024, 2048\}$.}
\label{fig:a5_se3_scan}
\end{figure*}

\subsection{Real-World 3D Inertial Human Activity and Postural Transition Pattern Recognition (UCI HAR and UCI HAPT-PT)}
\label{sec:exp_realworld}
We evaluate whether our theoretical hierarchy and optimization mechanisms transfer to noisy physical sensor streams under single-variable parity across two public subject-disjoint 3D inertial datasets:
\begin{enumerate}[label=(\arabic*),leftmargin=2em,itemsep=0pt]
    \item \textbf{Dataset 1 --- UCI HAR}~\citep{anguita2013public}: $10{,}299$ sequences ($50\text{ Hz}, T = 128$) from waist-mounted smartphone IMUs across $30$ volunteers performing $K = 6$ activities (\texttt{WALKING}, \texttt{WALKING\_UPSTAIRS}, \texttt{WALKING\_DOWNSTAIRS}, \texttt{SITTING}, \texttt{STANDING}, \texttt{LAYING}) under the official subject-disjoint split ($N_{\mathrm{train}} = 7{,}352, 21$ subjects; $N_{\mathrm{test}} = 2{,}947, 9$ unseen subjects).
    \item \textbf{Dataset 2 --- UCI HAPT-PT}~\citep{reyes2016transition}: $1{,}574$ sequences ($50\text{ Hz}, T = 128, C = 9$ channels) around the $358$ postural transition episodes across the $30$-subject cohort using offsets $\{-16, 0, +16\}$ ($N_{\mathrm{train}} = 1{,}250$ windows from $21$ subjects; $N_{\mathrm{test}} = 324$ windows from $N_{\mathrm{event,test}} = 108$ episodes across $N_{\mathrm{subj,test}} = 9$ unseen test subjects) spanning $K = 6$ order-sensitive directional 3D transition classes (\texttt{STAND\_TO\_SIT}, \texttt{SIT\_TO\_STAND}, \texttt{SIT\_TO\_LIE}, \texttt{LIE\_TO\_SIT}, \texttt{STAND\_TO\_LIE}, \texttt{LIE\_TO\_STAND}). Because forward/reverse transitions traverse opposite temporal pitch/roll trajectories in $\SO(3)$ (while remaining single-maneuver transitions rather than multi-generator non-solvable word problems), Dataset~2 directly tests directional 3D order sensitivity and continuous frame transport. Because the official split has $N_{\mathrm{subj,test}} = 9$ test subjects, we report full-cohort $30$-subject $5$-fold subject-disjoint cross-validation ($N_{\mathrm{subj}} = 30$, $N_{\mathrm{event}} = 358$ episodes, $N_{\mathrm{win}} = 1{,}574$ total and $1{,}074$ central-$3$ windows; Section~\ref{sec:exp_realdisp_deconfounding} and Table~\ref{tab:stem_capacity_realdisp}) as a \textbf{conditional frozen-specification subject-disjoint evaluation} on the full 30-subject HAPT cohort~\citep{cawley2010overfitting,demvsar2006statistical} (complementing the prospective $17+4$ split and the external 17-subject 3-IMU REALDISP cohort), alongside the 9-subject split (Table~\ref{tab:realworld_second_benchmark}; architectural design and the 51-configuration diagnostic sensitivity grid were developed on the 21 training subjects of the 9-subject split, which overlap with $21/30$ of the subjects rotated across the frozen-protocol 30-subject 5-fold CV). Protocol~B evaluates causal prefix recognition at $T \in \{32, 64, 96\}$.
\end{enumerate}

Our capacity-matched ($\sim 21.5\text{K--}22.9\text{K}$ parameters) regime operates directly on 9-channel raw triads $\mathbf{u}_t = (\bm{\omega}_t, \mathbf{a}_t^{\mathrm{body}}, \mathbf{a}_t^{\mathrm{tot}}) \in (\R^3)^3$. Across all 10 Lie-SSM/ablation models and 4 modern sequence backbones (\texttt{Selective-Mamba}~\citep{gu2023mamba}, \texttt{Causal-Transformer}, \texttt{Dilated-TCN}, \texttt{DeepConvGRU}), every model shares the identical $\SO(3)$-equivariant \textbf{Continuous Tri-Axial Kinematic Lifting Stem} ($M_1, 3{,}248$ params, stride $4$, mapping $T = 128 \to T_{\mathrm{sub}} = 32$ steps aligned to $\mathcal{T} = \{16, 32, \dots, 128\}$, plus a stride-$1$ $T=128$ control) and \textbf{Factored Norm-Direction Readout} ($M_3$), isolating $M_2$. At each timestep $t$, $M_1$ forms $8$ $\SO(3)$-equivariant triads $\mathcal{V}_t \in \R^{8 \times 3}$ ($3$ raw triads, $3$ finite-difference jerk triads, and $2$ cross-product triads $\bm{\omega}_t \times \mathbf{a}_t^{\mathrm{body}}, \mathbf{a}_t^{\mathrm{tot}} \times \mathbf{a}_t^{\mathrm{body}}$) and $12$ scalar invariants $\mathbf{s}_t \in \R^{12}$ ($8$ norms, $4$ inner products). Strided 1D convolution ($\text{kernel } 4$, stride $4$) and causal residual 1D convolution ($\text{kernel } 3$) lift $\mathcal{V}_t, \mathbf{s}_t$ to $\mathbf{V}_\tau \in \R^{12 \times 3}$ and $\mathbf{g}_\tau \in \R^{32}$:
\begin{equation}
\label{eq:realworld_tangent_lifting}
\bm{\Omega}_{\tau,m} = \Delta t_{\tau,m}(\mathbf{g}_\tau) \sum_{c=1}^{C_{\mathrm{vec}}} W_{m,c}^{\mathrm{lift}} \mathbf{V}_{\tau,c} \in \R^3 \simeq \mathfrak{so}(3), \qquad \mathbf{b}_{\tau,m} = \beta_{\tau,m}(\mathbf{g}_\tau) \sum_{c=1}^{C_{\mathrm{vec}}} W_{m,c}^{\mathrm{inj}} \mathbf{V}_{\tau,c} \in \R^3.
\end{equation}
Both \texttt{DeltaProduct-SO(3)} variants ($n_h = 1, 2$) execute $H_\tau = \prod_{k=1}^{n_h} (I_3 - 2\mathbf{v}_{\tau,k}\mathbf{v}_{\tau,k}^\top)$ with identical equivariant lifting and readout. In \texttt{Lie-SO(3)-Cascade}, Layer 1 ($M_1 = 7$) evolves from the input-anchored frame $h^{(1)}_{0,m}(u) = \mathrm{RMSNorm}_{\mathrm{dir}}\bigl(\sum_{c} W^{(1)}_{0,m,c} \mathbf{V}_{0,c}\bigr) \in S^2$ via $h^{(1)}_{\tau,m} = \rho_1 \exp(\hat{\bm{\Omega}}^{(1)}_{\tau,m}) h^{(1)}_{\tau-1,m}$ ($\rho_1 = 0.65$), routing Layer-2 generators $\bm{\omega}^{(2)}_{\tau,m}$ and injections $\mathbf{v}^{(2)}_{\tau,m}$ through the $\SO(3)$-equivariant bilinear quarter-turn frame operator $\mathcal{J}_\alpha(f_1, \mathbf{v})$ ($\alpha_{\tau,m} = 0.35 p^{\mathrm{route}}_{\tau,m}$) and pulling $h^{(2)}_{\tau,m}$ back via $\mathcal{J}_{-2-\alpha}(f_1, h^{(2)}_{\tau,m})$:
\begin{equation}
\label{eq:quarter_turn_operator}
\mathcal{J}_{\alpha}(f_1, \mathbf{v}) = \langle f_1, \mathbf{v}\rangle f_1 + (1 + \alpha)(f_1 \times \mathbf{v}), \qquad \mathcal{J}_{-2-\alpha}(f_1, h^{(2)}_{\tau,m}) = \langle f_1, h^{(2)}_{\tau,m}\rangle f_1 - (1 + \alpha)(f_1 \times h^{(2)}_{\tau,m}),
\end{equation}
where $\mathcal{J}_{-2-\alpha}(f_1, \mathcal{J}_{\alpha}(f_1, \mathbf{v})) = \langle f_1, \mathbf{v}\rangle f_1 + (1 + \alpha)^2(\mathbf{v} - \langle f_1, \mathbf{v}\rangle f_1)$ preserves $f_1$ and scales $f_1^\perp$ by $(1 + \alpha)^2$. Anchoring $h^{(1)}_{0,m}(u)$ to $\mathbf{V}_0$ makes Layer 1 and the forced Layer-2 series $\SO(3)$-equivariant by construction, while Layer 2 initializes from $h^{(2)}_0 \in (S^2)^{M_2}$. The unnormalized ablation (\texttt{Unnorm., $\rho_1=0.65$}) differs from \texttt{Dir-RMSNorm + Fact.} \emph{solely} in omitting $f_1 = h^{(1)}_{\tau,m} / \|h^{(1)}_{\tau,m}\|_2$. To decouple supervision density from readout factorization ($22{,}252$ vs.\ $22{,}280$ params), we compare \texttt{Terminal-Only, Unfact.}, \texttt{Dense, Unfact.}, and \texttt{Dir-RMSNorm + Fact.}, where a $14 \times 2$ coarse head $\bm{\ell}_{\tau}^{\mathrm{coarse}} = W_{\mathrm{posture}} \bigl((\mathbf{z}_{\tau} / \|\mathbf{z}_{\tau}\|_2) \odot \mathbf{z}_{\tau}\bigr)$ ($w_q = 0.20$) acts on $\mathbf{z}_{\tau} = \mathbf{f}_{\tau,-14:}^{\mathrm{inv}} \in \R^{14}$ ($\pi_{\mathrm{regime}} = [0, 0, 0, 1, 1, 1]$ on HAR; $[0, 1, 0, 1, 0, 1]$ on HAPT-PT):
\begin{equation}
\label{eq:realworld_factored_loss}
\ell_{\tau,y} = \bigl(\bar{\ell}_{\tau,\pi_{\mathrm{regime}}(y)}^{\mathrm{fine}} + w_q \ell_{\tau,\pi_{\mathrm{regime}}(y)}^{\mathrm{coarse}}\bigr) + \bigl(\ell_{\tau,y}^{\mathrm{fine}} - \bar{\ell}_{\tau,\pi_{\mathrm{regime}}(y)}^{\mathrm{fine}}\bigr) = \ell_{\tau,y}^{\mathrm{fine}} + w_q \ell_{\tau,\pi_{\mathrm{regime}}(y)}^{\mathrm{coarse}}, \quad \mathcal{L}_{\mathrm{MH}} = \sum_{k=1}^8 \tilde{w}_k \mathcal{L}_{\mathrm{CE}}\bigl(\bm{\ell}_{\tau_k}, y\bigr),
\end{equation}
where $\bar{\ell}_{\tau,\pi_{\mathrm{regime}}(y)}^{\mathrm{fine}} = \frac{1}{3} \sum_{y' \in \pi_{\mathrm{regime}}^{-1}(\pi_{\mathrm{regime}}(y))} \ell_{\tau,y'}^{\mathrm{fine}}$ cancels algebraically in the forward sum, acting as an auxiliary $\SO(3)$-invariant coarse-regime logit bias.

\begin{table}[htbp]
\centering
\caption{Real-world 3D inertial human activity pattern recognition across 14 capacity-matched architectures (4 modern sequence backbones and 10 Lie-SSM / ablation models) on the subject-disjoint UCI HAR benchmark ($N_{\mathrm{train}} = 7{,}352$, $N_{\mathrm{test}} = 2{,}947$, $T=128$, $9$ raw inertial channels, $6$ classes, $5$ seeds $\{42, 123, 456, 789, 1024\}$). Protocol A reports full-horizon ($T=128$) test accuracy and macro-F1; Protocol B reports causal streaming accuracy across prefixes $T \in \{32, 64, 96\}$; Protocol C reports test accuracy under 3D sensor-frame $\SO(3)$ orientation perturbations $\sigma_{\mathrm{rot}} \in \{15^\circ, 30^\circ, 45^\circ\}$.}
\label{tab:realworld_har_benchmark}
\footnotesize
\medmuskip=0mu\thinmuskip=0mu\thickmuskip=0mu
\setlength{\tabcolsep}{3.0pt}
\renewcommand{\arraystretch}{0.90}
\begin{tabular}{@{}lrrrrrrrrr@{}}
\toprule
\multirow{2}{*}{\textbf{Model}} & \multirow{2}{*}{\textbf{Par.}} & \multicolumn{2}{c}{\textbf{Prot.\ A ($T=128$)}} & \multicolumn{3}{c}{\textbf{Prot.\ B: Prefix Acc (\%)}} & \multicolumn{3}{c}{\textbf{Prot.\ C: $\sigma_{\mathrm{rot}}$ Acc (\%)}} \\
\cmidrule(lr){3-4} \cmidrule(lr){5-7} \cmidrule(lr){8-10}
 &  & \textbf{Acc} & \textbf{F1} & \textbf{$T=32$} & \textbf{$T=64$} & \textbf{$T=96$} & \textbf{$15^\circ$} & \textbf{$30^\circ$} & \textbf{$45^\circ$} \\
\midrule
Mamba & 22,115 & $92.95 \pm 0.80$ & $93.05 \pm 0.82$ & $90.78 \pm 0.73$ & $92.30 \pm 0.60$ & $92.79 \pm 0.71$ & $87.68 \pm 1.61$ & $77.03 \pm 2.35$ & $65.66 \pm 4.18$ \\
Transf. & 22,163 & $93.89 \pm 0.35$ & $93.81 \pm 0.36$ & $91.27 \pm 0.30$ & $93.33 \pm 0.48$ & $93.89 \pm 0.65$ & $91.59 \pm 0.55$ & $82.90 \pm 0.75$ & $71.43 \pm 1.94$ \\
TCN & 22,204 & $91.57 \pm 0.82$ & $91.73 \pm 0.75$ & $89.89 \pm 1.11$ & $91.56 \pm 0.83$ & $91.37 \pm 0.92$ & $85.63 \pm 1.77$ & $75.94 \pm 1.83$ & $64.66 \pm 2.74$ \\
GRU & 22,125 & $94.79 \pm 1.03$ & $94.80 \pm 1.03$ & $\mathbf{92.04 \pm 0.66}$ & $93.97 \pm 0.85$ & $94.58 \pm 1.02$ & $91.33 \pm 0.99$ & $83.28 \pm 1.09$ & $72.68 \pm 1.84$ \\
\midrule
Diag-SSM & 22,950 & $94.22 \pm 1.02$ & $94.33 \pm 1.00$ & $91.90 \pm 0.18$ & $93.72 \pm 0.58$ & $94.12 \pm 0.79$ & $90.30 \pm 1.69$ & $81.17 \pm 1.34$ & $70.61 \pm 2.44$ \\
RotRNN & 22,878 & $91.72 \pm 1.33$ & $91.82 \pm 1.29$ & $90.40 \pm 0.46$ & $92.25 \pm 1.24$ & $92.03 \pm 1.75$ & $84.27 \pm 1.22$ & $73.44 \pm 0.87$ & $62.78 \pm 1.40$ \\
DP ($n_h{=}1$) & 21,474 & $92.11 \pm 0.48$ & $92.16 \pm 0.48$ & $89.04 \pm 0.63$ & $91.37 \pm 0.36$ & $92.04 \pm 0.41$ & $89.51 \pm 0.42$ & $84.47 \pm 0.58$ & $78.93 \pm 0.84$ \\
DP ($n_h{=}2$) & 22,374 & $91.94 \pm 0.95$ & $92.04 \pm 0.98$ & $88.80 \pm 0.66$ & $91.37 \pm 0.55$ & $92.40 \pm 0.71$ & $89.37 \pm 0.97$ & $84.02 \pm 1.51$ & $78.65 \pm 1.66$ \\
$\SO(3)$-1L & 22,134 & $95.24 \pm 0.85$ & $95.29 \pm 0.86$ & $91.27 \pm 0.48$ & $93.91 \pm 0.92$ & $95.05 \pm 0.86$ & $92.64 \pm 1.07$ & $87.20 \pm 0.89$ & $81.31 \pm 1.35$ \\
$\SO(6)$-1L & 20,874 & $93.99 \pm 0.57$ & $94.02 \pm 0.58$ & $90.55 \pm 0.65$ & $93.02 \pm 0.76$ & $94.10 \pm 0.60$ & $90.05 \pm 1.30$ & $81.06 \pm 0.94$ & $70.31 \pm 2.66$ \\
Casc.\ (Unn.) & 22,280 & $91.00 \pm 1.32$ & $90.97 \pm 1.37$ & $87.72 \pm 1.08$ & $90.05 \pm 0.90$ & $90.75 \pm 1.10$ & $88.02 \pm 1.30$ & $82.04 \pm 1.40$ & $74.93 \pm 1.23$ \\
Casc.\ (Term.) & 22,252 & $94.49 \pm 0.48$ & $94.50 \pm 0.49$ & $64.83 \pm 6.86$ & $85.35 \pm 2.98$ & $92.59 \pm 0.85$ & $92.31 \pm 0.83$ & $87.65 \pm 0.91$ & $82.08 \pm 1.43$ \\
\textbf{Casc.\ (Den.)} & \textbf{22,252} & $\mathbf{95.69 \pm 0.14}$ & $\mathbf{95.75 \pm 0.15}$ & $91.18 \pm 0.32$ & $94.22 \pm 0.39$ & $95.28 \pm 0.20$ & $\mathbf{93.36 \pm 0.59}$ & $\mathbf{88.35 \pm 0.73}$ & $\mathbf{82.95 \pm 0.77}$ \\
\textbf{Casc.\ (Fact.)} & \textbf{22,280} & $\mathbf{95.47 \pm 0.44}$ & $\mathbf{95.53 \pm 0.46}$ & $91.44 \pm 0.30$ & $\mathbf{94.39 \pm 0.51}$ & $\mathbf{95.31 \pm 0.29}$ & $93.32 \pm 0.60$ & $88.24 \pm 1.32$ & $82.69 \pm 0.75$ \\
\bottomrule
\end{tabular}
\end{table}

\begin{table}[htbp]
\centering
\caption{Real-world 3D postural transition pattern recognition across all 14 capacity-matched architectures (4 modern sequence backbones and 10 Lie-SSM / ablation models) on the subject-disjoint UCI HAPT-PT benchmark ($N_{\mathrm{train}} = 1{,}250$, $N_{\mathrm{test}} = 324$, $T=128$, $9$ inertial channels, $6$ non-commutative postural transition classes: \texttt{STAND\_TO\_SIT}, \texttt{SIT\_TO\_STAND}, \texttt{SIT\_TO\_LIE}, \texttt{LIE\_TO\_SIT}, \texttt{STAND\_TO\_LIE}, \texttt{LIE\_TO\_STAND}; $5$ seeds $\{42, 123, 456, 789, 1024\}$).}
\label{tab:realworld_second_benchmark}
\footnotesize
\setlength{\tabcolsep}{3.5pt}
\renewcommand{\arraystretch}{0.90}
\begin{tabular}{@{}lrrrrrr@{}}
\toprule
\textbf{Model (UCI HAPT-PT)} & \textbf{Params} & \textbf{Test Acc (\%)} & \textbf{Macro-F1 (\%)} & \textbf{Pref.\ $T=64$ (\%)} & \textbf{$\sigma_{\mathrm{rot}}=30^\circ$ (\%)} & \textbf{$\sigma_{\mathrm{rot}}=45^\circ$ (\%)} \\
\midrule
Selective-Mamba & 22,115 & $81.23 \pm 1.74$ & $81.13 \pm 1.73$ & $72.59 \pm 1.93$ & $72.35 \pm 1.79$ & $64.88 \pm 3.71$ \\
Causal-Transformer & 22,163 & $81.85 \pm 2.06$ & $81.77 \pm 2.08$ & $73.15 \pm 1.23$ & $74.63 \pm 2.69$ & $65.86 \pm 4.65$ \\
Dilated-TCN & 22,204 & $73.09 \pm 1.14$ & $72.61 \pm 1.08$ & $70.86 \pm 2.63$ & $65.93 \pm 2.47$ & $57.84 \pm 2.26$ \\
DeepConvGRU & 22,125 & $84.07 \pm 1.77$ & $84.08 \pm 1.77$ & $74.51 \pm 2.03$ & $76.17 \pm 1.40$ & $67.78 \pm 2.18$ \\
\midrule
Diag-SSM & 22,950 & $82.96 \pm 0.74$ & $83.22 \pm 0.72$ & $73.52 \pm 3.29$ & $74.63 \pm 2.68$ & $66.36 \pm 2.80$ \\
RotRNN-$\SO(2)$ & 22,878 & $79.51 \pm 1.42$ & $79.68 \pm 1.48$ & $72.28 \pm 1.94$ & $69.01 \pm 0.94$ & $59.38 \pm 2.70$ \\
DeltaProd.-$\SO(3)$ ($n_h{=}1$) & 21,474 & $75.80 \pm 1.42$ & $76.03 \pm 1.45$ & $65.12 \pm 1.95$ & $73.89 \pm 1.75$ & $70.68 \pm 3.20$ \\
DeltaProd.-$\SO(3)$ ($n_h{=}2$) & 22,374 & $78.40 \pm 1.79$ & $78.10 \pm 1.71$ & $63.58 \pm 2.65$ & $76.60 \pm 2.30$ & $74.44 \pm 3.52$ \\
Lie-$\SO(3)$-1L & 22,134 & $85.25 \pm 1.87$ & $85.22 \pm 1.86$ & $73.40 \pm 1.53$ & $83.33 \pm 1.38$ & $80.68 \pm 1.28$ \\
Lie-$\SO(6)$-1L & 20,874 & $84.01 \pm 1.49$ & $84.14 \pm 1.57$ & $74.51 \pm 1.70$ & $77.53 \pm 2.91$ & $71.48 \pm 3.43$ \\
Lie-$\SO(3)$-Casc.\ (Unnorm.) & 22,280 & $81.48 \pm 2.52$ & $81.45 \pm 2.48$ & $70.99 \pm 3.35$ & $75.43 \pm 1.17$ & $69.26 \pm 2.87$ \\
Lie-$\SO(3)$-Casc.\ (Term.) & 22,252 & $86.91 \pm 2.18$ & $86.79 \pm 2.19$ & $61.30 \pm 3.53$ & $\mathbf{86.23 \pm 1.36}$ & $\mathbf{84.20 \pm 2.46}$ \\
Lie-$\SO(3)$-Casc.\ (Den., Unf.) & 22,252 & $86.85 \pm 1.96$ & $86.79 \pm 1.90$ & $74.20 \pm 1.92$ & $84.57 \pm 1.76$ & $82.28 \pm 2.97$ \\
\textbf{Lie-$\SO(3)$-Casc.\ (RMS+Fac.)} & \textbf{22,280} & $\mathbf{87.16 \pm 0.64}$ & $\mathbf{87.09 \pm 0.66}$ & $\mathbf{75.00 \pm 1.15}$ & $85.37 \pm 1.77$ & $82.72 \pm 2.31$ \\
\bottomrule
\end{tabular}
\end{table}

Tables~\ref{tab:realworld_har_benchmark}--\ref{tab:realworld_second_benchmark} and Figure~\ref{fig:realworld_har}(a) establish four findings: (1) \emph{Full-horizon recognition}: \texttt{Lie-SO(3)-Cascade (Dense, Unfact.)} reaches $95.69\% \pm 0.14\%$ ($95\%\text{ CI }[95.58\%, 95.80\%]$) and \texttt{Dir-RMSNorm + Fact.} reaches $95.47\% \pm 0.44\%$ on UCI HAR; on UCI HAPT-PT, \texttt{Dir-RMSNorm + Fact.} reaches $87.16\% \pm 0.64\%$ ($87.09\% \pm 0.66\%$ macro-F1; $88.89\% \pm 2.17\%$ episode accuracy with $95\%$ cluster-bootstrap CIs $[82.78\%, 90.99\%]$ across $N_{\mathrm{event,test}} = 108$ episodes and $[82.84\%, 91.11\%]$ across $N_{\mathrm{subj,test}} = 9$ subjects). Under Holm--Bonferroni corrections across both the 10-baseline ($p_{\mathrm{Holm,10}}$) and 11-baseline ($p_{\mathrm{Holm,11}}$) families, \texttt{Dir-RMSNorm + Fact.} outperforms 8/10 non-$\SO(3)$ baselines on HAR ($p_{\mathrm{Holm,11}} \le 0.0102$) and 10/10 on HAPT-PT ($p_{\mathrm{Holm,11}} \le 0.0285$: \texttt{Diag-SSM} $+4.20\,\mathrm{pp}, p_{\mathrm{Holm,11}} = 1.17 \times 10^{-4}$; \texttt{Selective-Mamba} $+5.93\,\mathrm{pp}, p_{\mathrm{Holm,11}} = 0.0047$; \texttt{DeepConvGRU} $+3.09\,\mathrm{pp}, p_{\mathrm{Holm,11}} = 0.0285$), gaining $+1.91\,\mathrm{pp}$ over \texttt{Lie-SO(3)-1L} ($87.16\%$ vs.\ $85.25\%, p = 0.0835, d = 1.37$) with $3\times$ lower seed SD ($0.64\%$ vs.\ $1.87\%$); (2) \emph{Directional RMSNorm}: omitting spherical projection (\texttt{Unnorm., $\rho_1=0.65$}) drops accuracy by $-4.47\,\mathrm{pp}$ on HAR ($91.00\% \pm 1.32\%, p_{\mathrm{Holm,11}} = 0.0054$) and $-5.68\,\mathrm{pp}$ on HAPT-PT ($81.48\% \pm 2.52\%, p_{\mathrm{Holm,11}} = 0.0237$); (3) \emph{Supervision density and stride-1 control}: multi-horizon supervision raises prefix accuracy by $+26.35\,\mathrm{pp}$ at $T=32$ on HAR ($91.18\%$ vs.\ $64.83\%$) and $+12.90\,\mathrm{pp}$ at $T=64$ on HAPT-PT ($74.20\%$ vs.\ $61.30\%$), while factored readout reaches $94.39\%$ and $75.00\%$ at $T=64$ (and diagnostic sensitivity exploration across $51$ configurations $\times 5$ seeds plus stride-$1$ $T=128$ recurrence confirms $86.98\% \pm 1.99\%$ clean and $84.14\% \pm 2.19\%$ at $\sigma_{\mathrm{rot}}=45^\circ$ [$87.10\% \pm 2.39\%$ and $85.80\% \pm 2.15\%$ with $\SO(3)$-equivariant Layer-2 initialization] vs.\ stride-$1$ \texttt{Lie-SO(3)-1L} [$86.17\% \pm 2.12\%$ clean, $82.16\% \pm 2.33\%$ at $45^\circ$] and \texttt{Diag-SSM} [$81.85\% \pm 2.05\%$ clean, $63.64\% \pm 1.83\%$ at $45^\circ$]); and (4) \emph{Rotational robustness (Protocol~C)}: at $\sigma_{\mathrm{rot}}=45^\circ$, \texttt{Dir-RMSNorm + Fact.} retains $82.69\% \pm 0.75\%$ (HAR) and $82.72\% \pm 2.31\%$ (HAPT-PT), beating all 10/10 non-$\SO(3)$ baselines ($p_{\mathrm{Holm,11}} \le 0.0066$) and reaching $86.48\% \pm 0.74\%$ with $\SO(3)$-equivariant Layer-2 initialization ($+5.80\,\mathrm{pp}$ over \texttt{Lie-SO(3)-1L}, $p_{\mathrm{Holm,11}} = 6.08 \times 10^{-4}$).

\subsection{Input-Stem Attribution, Capacity Pareto, 30-Subject Cross-Validation, and Multi-IMU Sensor Displacement (UCI REALDISP)}
\label{sec:exp_realdisp_deconfounding}
To de-confound input-stem representation from recurrent operator geometry, test whether multi-lag scalar differences or chiral triple products suffice without $\SO(3)$ frame transport, test whether $4\times$ capacity ($\sim 90\mathrm{K}$ vs.\ $22\mathrm{K}$) substitutes for Lie structure, evaluate full-cohort $30$-subject $5$-fold subject-disjoint cross-validation ($N_{\mathrm{subj}} = 30$, $N_{\mathrm{event}} = 358$ episodes, $N_{\mathrm{win}} = 1{,}574$ total and $1{,}074$ central-$3$ windows $\{-16, 0, +16\}$ across $5$ contiguous $6$-subject folds $\{6k+1\dots 6k+6\}$ with fold-isolated $24$-subject RMS standardization and archived OOF predictions across all 13 arms) alongside the 9-subject split and prospective $17+4$ protocol on UCI HAPT-PT, and extend validation to multi-IMU sensor displacement, Table~\ref{tab:stem_capacity_realdisp} (complete 58-row tables in Supplementary Material, \ref{app:supplementary_tables}, Tables~\ref{tab:app_full_stem_hapt_cv}--\ref{tab:app_full_realdisp}) and Figure~\ref{fig:realworld_har}(b--c) evaluate six input stems across UCI HAR / HAPT-PT (Panel~A) and 3-IMU UCI REALDISP~\citep{banos2014dealing} (Panel~B): (i) \textbf{Tier 1 (\texttt{Native-Conv1D}, $22\mathrm{K}$ \& $90\mathrm{K}$)} on raw coordinates (where an un-ablated control reaches $86.73\%$ Clean HAPT-PT before collapsing to $24.69\%$ under $\mathrm{Unif}(\SO(3))$); (ii) \textbf{Tier 2 (\texttt{Equiv-Concat-Stem}, $22\mathrm{K}$)} passing $[\mathrm{vec}(\mathbf{V}_\tau), \mathbf{g}_\tau]$; (iii) \textbf{Tier 3a (\texttt{Inv-Scalar-Stem}, $16$ ch, $22\mathrm{K}$ \& $90.5\mathrm{K}$)}, comprising $14$ achiral Gram norms/inner products across the triad, jerk, and cross-product vectors plus $2$ instantaneous chiral triple products $\det[\bm{\omega}_t, \mathbf{a}_t^{\mathrm{body}}, \mathbf{a}_t^{\mathrm{tot}}]$ and $\det[\Delta\bm{\omega}_t, \Delta\mathbf{a}_t^{\mathrm{body}}, \Delta\mathbf{a}_t^{\mathrm{tot}}]$ ($\max_{R \in \SO(3)} \|f(R\mathbf{u}) - f(\mathbf{u})\|_\infty \le 3.5 \times 10^{-7}$ in \texttt{float32} and $5.00 \times 10^{-16}$ in \texttt{float64}); (iv) \textbf{Tier 3b (\texttt{Lagged-Inv-Scalar-Stem}, $28$ ch, $22\mathrm{K}$)}, adding $12$ multi-lag inner products ($\langle\mathbf{a}_{\mathrm{tot}}(t), \mathbf{a}_{\mathrm{tot}}(t-\Delta)\rangle$, $\langle\bm{\omega}(t), \bm{\omega}(t-\Delta)\rangle$, $\langle\mathbf{a}_{\mathrm{body}}(t), \bm{\omega}(t-\Delta)\rangle$ for $\Delta \in \{1, 2, 4, 8\}$, drift $\le 2.8 \times 10^{-7}$ in 30-subject CV and $\le 1.7 \times 10^{-5}$ on the 9-subject split); (v) \textbf{Tier 3c (\texttt{Chiral-Triple-Inv-Stem}, $36$ ch, $22\mathrm{K}$)}, adding $8$ multi-lag $\SO(3)$-invariant chiral triple products $\det[\mathbf{a}_{\mathrm{tot}}(t), \bm{\omega}(t), \mathbf{a}_{\mathrm{tot}}(t-\Delta)]$ and $\det[\mathbf{a}_{\mathrm{tot}}(t), \mathbf{a}_{\mathrm{tot}}(t-\Delta), \mathbf{a}_{\mathrm{tot}}(t-2\Delta)]$ for $\Delta \in \{1, 2, 4, 8\}$ (drift $\le 3.3 \times 10^{-7}$ in 30-subject CV and $\le 2.8 \times 10^{-5}$ on the 9-subject split); and (vi) \textbf{Tier 4 (\texttt{Full-Equiv-Frame Lie-$\SO(3)$}, $22\mathrm{K}$)}, where Layer~1 and Layer~2 initialize from input-anchored $\SO(3)$-equivariant Gram--Schmidt Darboux triads $B_{\mathrm{eq}}(\mathbf{u}) = [\mathbf{e}_1, \mathbf{e}_2, \mathbf{e}_3]$ ($h^{(1)}_0(\mathbf{u}), h^{(2)}_0(\mathbf{u}) \in (S^2)^M$, forming exact right-handed frames in $\SO(3)$ with $\det B_{\mathrm{eq}} = +1$ whenever $(u_1, u_2)$ are linearly independent (non-collinear) and acting as a bounded $\varepsilon$-normalized [$\varepsilon = 10^{-7}$] $\SO(3)$-equivariant feature map on the measure-zero collinear/zero degenerate locus where $\det B_{\mathrm{eq}} = 0$) and the readout uses $\hat{\mathbf{g}}_0(\mathbf{u})$, achieving global $\SO(3)$ and multi-sensor $\SO(3)^3$ invariance (synthetic drift $\le 3.5 \times 10^{-7}$ in \texttt{float32} and $\le 2.31 \times 10^{-13}$ in \texttt{float64}; trained HAPT Fold-0 test-window invariance $1.53 \times 10^{-5}$ in \texttt{float32} and $2.40 \times 10^{-14}$ in \texttt{float64} with $100.00\%$ prediction match and degenerate collinear-triad equivariance stability $1.13 \times 10^{-14}$, giving identical Clean and $\mathrm{Unif}(\SO(3))$ OOF accuracy under all-window episode voting [$88.44\% \pm 0.83\%$], uniform central-3 episode voting [$87.37\% \pm 1.21\%$], and hard-label voting [$87.21\% \pm 1.33\%$], and $88.04\% \pm 0.85\%$ under dynamic $\sigma_{\mathrm{rot}}=45^\circ$ drift). The 30-subject 5-fold CV evaluates a frozen specification (\texttt{training\_core\_sha256}, \texttt{extended\_operator\_ast\_sha256}, \texttt{hyperparameter\_and\_split\_sha256}, and full upstream \texttt{dependency\_module\_sha256\_manifest} verified in cache metadata) on the reused 30-subject HAPT cohort~\citep{cawley2010overfitting,demvsar2006statistical} (sharing 21 subjects with the exploratory 9-subject split), serving as a conditional frozen-specification evaluation paired with prospective $17+4$ validation (\texttt{Val-Clean}, \texttt{Val-Orient.-Agnostic}) and the external 17-subject 3-IMU REALDISP cohort (evaluated under reused architectural diagnostics). On \textbf{Dataset~3 --- UCI REALDISP}~\citep{banos2014dealing} ($K = 33$ fitness activities, $50\text{ Hz}, T=128$, $17$ subjects in a $11/3/3$ split with validation subjects $\#10, 14, 16$ and locked test subjects $\#2, 5, 15$, $S = 3$ 6-axis IMUs on \texttt{BACK}, \texttt{RLA}, \texttt{RC}, $18$ channels), models train on \textbf{Ideal-Placement}, select checkpoints on validation subjects $\#10, 14, 16$ (\texttt{Ideal + Self-Placement}, with \texttt{Mutual-Displacement} unseen), and are tested zero-shot on held-out subjects $\#2, 5, 15$ across (1) \textbf{Ideal-Placement} ($1{,}867$ windows, $103$ bouts, $T=128, 64$), (2) \textbf{Self-Placement} ($2{,}037$ windows, $101$ bouts), (3) \textbf{Mutual Sensor Displacement} ($7{,}780$ windows, $353$ bouts with $180^\circ$ inversions and $90^\circ/180^\circ$ limb rotations), and (4) synthetic $\SO(3)^3$ \textbf{Rot45}; on this $11/3/3$ split, 5-seed paired $t$-tests, exact paired sign-flip permutation $p$-values ($p_{\mathrm{perm}}$, including the $2^{15} = 32{,}768$ auxiliary 15-cell [$3\text{ subjects} \times 5\text{ seeds}$] sensitivity summary), and $B=5{,}000$ bout-cluster bootstrap CIs ($p_{\mathrm{boot},+1} = (r+1)/(B+1) \ge 1/5001$) quantify training-seed reproducibility conditioned on the $n_{\mathrm{test}}=3$ held-out test subjects ($\#2, \#5, \#15$), alongside per-subject breakdowns.

\begin{table}[!htbp]
\centering
\caption{Summary (12 representative rows) of de-confounded 6-stem $\times$ parameter-capacity Pareto attribution, full-cohort 30-subject 5-fold cross-validation on UCI HAPT-PT, and 3-IMU sensor displacement on UCI REALDISP across 5 seeds ($\{42, 123, 456, 789, 1024\}$). Complete 58-row benchmarks and per-subject ($\#2, \#5, \#15$) breakdowns are reported in \ref{app:supplementary_tables} (Tables~\ref{tab:app_full_stem_hapt_cv}--\ref{tab:app_full_realdisp}). In the 30-subject CV panel, \texttt{Ep-358 Vote} and \texttt{Ep-358 C3 Vote} report all-window and uniform central-3 episode logit-voting.}
\label{tab:stem_capacity_realdisp}
\footnotesize
\medmuskip=0mu\thinmuskip=0mu\thickmuskip=0mu
\setlength{\tabcolsep}{3.5pt}
\renewcommand{\arraystretch}{0.90}
\begin{tabular}{@{}lrrrrrrr@{}}
\toprule
\multicolumn{8}{@{}l@{}}{\textbf{Panel A (Summary): Single-IMU 6-Stem $\times$ Capacity Pareto on UCI HAR and UCI HAPT-PT (9-Subj Split)}} \\
\midrule
\textbf{Model \& Stem} & \textbf{Par.} & \textbf{HAR Cln} & \textbf{HAR Unif} & \textbf{HAPT Cln} & \textbf{HAPT R45} & \textbf{HAPT Unif} & \textbf{Drift} \\
\midrule
GRU (Nat., 90K) & 90,304 & $92.48 \pm 0.86$ & $24.24 \pm 1.45$ & $84.01 \pm 1.33$ & $61.11 \pm 3.58$ & $20.12 \pm 1.90$ & $3.8 \times 10^{-1}$ \\
GRU (Inv, 90K) & 90,546 & $92.59 \pm 1.19$ & $92.59 \pm 1.19$ & $82.65 \pm 1.93$ & $82.65 \pm 1.93$ & $82.65 \pm 1.93$ & $1.8 \times 10^{-7}$ \\
GRU (Chiral, 22K) & 22,354 & --- & --- & $83.27 \pm 1.44$ & $83.15 \pm 1.77$ & $83.27 \pm 1.44$ & $2.8 \times 10^{-5}$ \\
\textbf{Casc.\ (Equiv, 22K)} & \textbf{22,280} & $84.22 \pm 2.96$ & $84.22 \pm 2.96$ & $\mathbf{86.23 \pm 0.64}$ & $\mathbf{86.05 \pm 0.34}$ & $\mathbf{86.23 \pm 0.64}$ & $\mathbf{1.9 \times 10^{-7}}$ \\
\midrule
\multicolumn{8}{@{}l@{}}{\textbf{Frozen-Protocol 30-Subject 5-Fold CV on UCI HAPT-PT ($N_{\mathrm{subj}}=30, N_{\mathrm{event}}=358, N_{\mathrm{win}}=1{,}574 / 1{,}074$)}} \\
\midrule
\textbf{30-Subj CV Model} & \textbf{Par.} & \textbf{Win-1574} & \textbf{Win-C3} & \textbf{Ep-Vote} & \textbf{Ep-C3} & \textbf{Subj C3 CI} & \textbf{Holm $p$} \\
\midrule
Casc.\ (Unnorm., 22K) & 22,280 & $77.89 \pm 1.81$ & $78.83 \pm 1.69$ & $85.25 \pm 1.74$ & $83.63 \pm 1.62$ & $[+3.27, +7.08]$ & $7.6 \times 10^{-5} / 0.0019$ \\
GRU (Inv, 90K) & 90,546 & $81.86 \pm 1.10$ & $82.14 \pm 1.27$ & $83.91 \pm 1.21$ & $83.58 \pm 1.16$ & $[+0.07, +3.68]$ & $0.0561 / 3.0 \times 10^{-4}$ \\
GRU (Lag, 22K) & 22,400 & $80.39 \pm 0.65$ & $81.19 \pm 0.37$ & $83.13 \pm 0.73$ & $83.13 \pm 0.70$ & $[+0.78, +4.83]$ & $0.0366 / 3.3 \times 10^{-5}$ \\
\textbf{Casc.\ (Equiv, 22K)} & \textbf{22,280} & $\mathbf{82.69 \pm 1.11}$ & $\mathbf{83.97 \pm 1.20}$ & $\mathbf{88.44 \pm 0.83}$ & $\mathbf{87.37 \pm 1.21}$ & Reference & Reference \\
\midrule
\multicolumn{8}{@{}l@{}}{\textbf{Panel B (Summary): Multi-IMU Sensor Displacement on UCI REALDISP ($S=3$ IMUs, $K=33$, Subjects $\#2, 5, 15$)}} \\
\midrule
\textbf{3-IMU Model ($K=33$)} & \textbf{Par.} & \textbf{Ideal $128$} & \textbf{Ideal $64$} & \textbf{Self} & \textbf{Mutual} & \textbf{Rot45} & \textbf{Subj $\#2/5/15$} \\
\midrule
GRU (Native) & 25,301 & $95.58 \pm 0.82$ & $95.47 \pm 0.67$ & $76.06 \pm 1.77$ & $18.54 \pm 1.86$ & $54.50 \pm 2.07$ & $21.4 / 25.4 / 5.2$ \\
GRU (Lag-Inv) & 25,587 & $96.25 \pm 0.64$ & $95.64 \pm 0.50$ & $95.32 \pm 0.85$ & $87.51 \pm 0.79$ & $96.24 \pm 0.60$ & --- \\
\textbf{Casc.\ (Dir-RMS)} & 25,173 & $93.15 \pm 0.74$ & $92.59 \pm 0.64$ & $92.92 \pm 0.76$ & $\mathbf{83.37 \pm 1.00}$ & $93.28 \pm 0.84$ & $92.5 / 85.7 / 66.1$ \\
\textbf{Casc.\ (Q-Turn)} & 25,379 & $94.09 \pm 0.53$ & --- & $93.33 \pm 1.03$ & $\mathbf{84.23 \pm 1.00}$ & $94.01 \pm 0.51$ & --- \\
\bottomrule
\end{tabular}
\end{table}

\begin{figure*}[!htbp]
\centering
\includegraphics[width=0.96\textwidth]{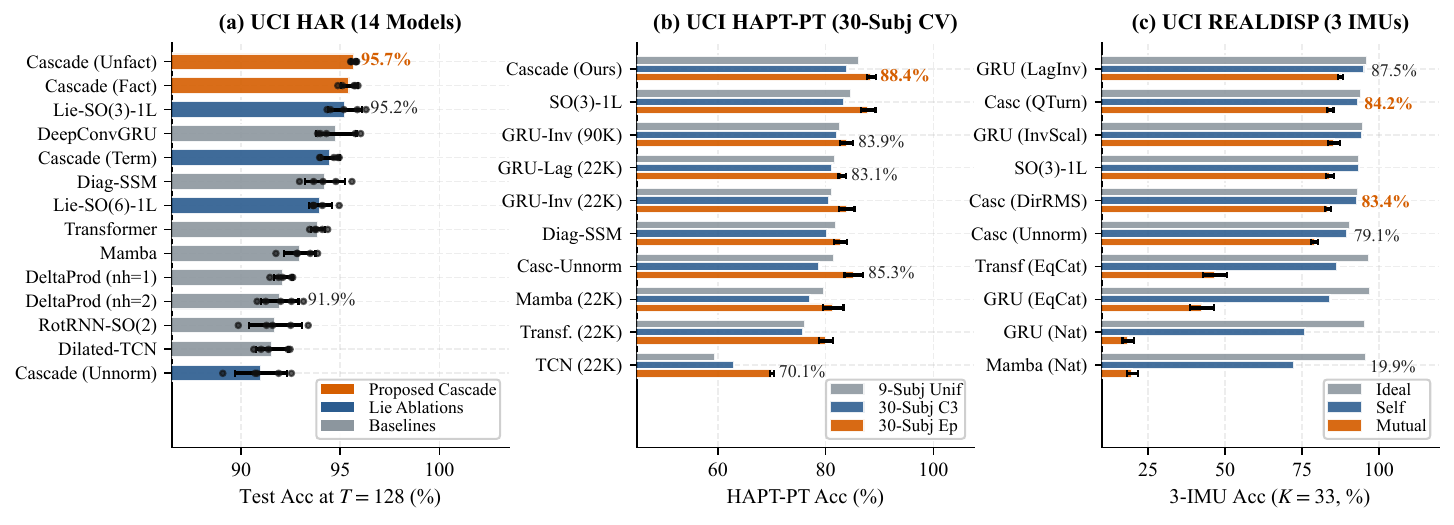}
\caption{Cross-cohort 3D inertial pattern recognition and 6-stem $\times$ capacity attribution across UCI HAR, UCI HAPT-PT, and multi-IMU UCI REALDISP (5 seeds $\{42, 123, 456, 789, 1024\}$). \textbf{(a)} UCI HAR full-horizon ($T=128$) subject-disjoint test accuracy across all 14 capacity-matched architectures ($\text{mean} \pm \text{SD}$ bars with 5-seed scatter). \textbf{(b)} UCI HAPT-PT order-sensitive 3D postural transition accuracy across representative architectures under 3-way grouped comparison: frozen-protocol 30-subject 5-fold cross-validation at uniform central-3 window (\texttt{Win-1074 C3}) and episode-vote (\texttt{Ep-358 Vote}) levels alongside the 9-subject official test split under $\mathrm{Unif}(\SO(3))$ mounting ($88.44\% \pm 0.83\%$ on \texttt{Ep-358 Logit-Vote} and $87.37\% \pm 1.21\%$ on \texttt{Ep-358 C3 Vote} vs.\ $83.91\% \pm 1.21\%$ and $83.58\% \pm 1.16\%$ for $4\times$ larger $90.5\mathrm{K}$ \texttt{DeepConvGRU (Inv-Scalar)}, $p_{\mathrm{Holm,6}} = 0.0017$ and $0.0057$). \textbf{(c)} UCI REALDISP 3-IMU ($K=33, N_{\mathrm{subj,test}}=3$ held-out subjects $\#2, 5, 15$) accuracy across Ideal placement, participant Self-Placement, and physical Mutual Sensor Displacement ($353$ bouts), highlighting the $+4.24\,\mathrm{pp}$ Directional-RMSNorm gain ($83.37\%$ vs.\ $79.13\%, p_{\mathrm{Holm}}=0.0353$), $84.23\% \pm 1.00\%$ for the Eq.~\eqref{eq:quarter_turn_operator} quarter-turn cascade, and $87.51\% \pm 0.79\%$ for \texttt{DeepConvGRU (Lagged-Inv-Scalar)}.}
\label{fig:realworld_har}
\end{figure*}

Table~\ref{tab:stem_capacity_realdisp} (with Tables~\ref{tab:app_full_stem_hapt_cv}--\ref{tab:app_full_realdisp}) and Figure~\ref{fig:realworld_har}(b--c) establish three domain-of-validity and attribution findings across all three wearable benchmarks:
\begin{enumerate}[leftmargin=2em,itemsep=0pt]
    \item \textbf{When Non-Commutative $\SO(3)$ Frame Transport Outperforms Instantaneous, Lagged, and Chiral-Triple Scalar Stems and $4\times$ Capacity Scaling (UCI HAPT-PT 30-Subject CV \& 9-Subject Split, Panel~A)}: On order-sensitive 3D postural transitions, scaling \texttt{Native-Conv1D} $4\times$ to $\sim 90\mathrm{K}$ parameters fails to prevent collapse under $\mathrm{Unif}(\SO(3))$ ($19.81\%$--$21.42\%$, trailing our $22\mathrm{K}$ cascade by $>64\,\mathrm{pp}, p_{\mathrm{Holm,23}} \le 1.45 \times 10^{-6}$). Even with an $\SO(3)$-invariant instantaneous Gram stem (\texttt{Inv-Scalar-Stem}), a 28-channel multi-lag inner-product stem (\texttt{Lagged-Inv-Scalar-Stem}, $\Delta \in \{1, 2, 4, 8\}$), or a 36-channel chiral triple-product stem (\texttt{Chiral-Triple-Inv-Stem}), fixed finite-lag scalars do not match continuous $\SO(3)$ frame transport. Under 30-subject 5-fold CV ($N_{\mathrm{subj}}=30, N_{\mathrm{event}}=358, 1{,}074$ C3 / $1{,}574$ total windows across $5\text{ folds} \times 5\text{ seeds} = 25$ fold-seed runs), our $22.3\mathrm{K}$ cascade reaches $88.44\% \pm 0.83\%$ on all-window \texttt{Ep-358 Logit-Vote} ($87.21\% \pm 1.33\%$ Hard-Vote, identical under Clean and $\mathrm{Unif}(\SO(3))$ and $88.04\% \pm 0.85\%$ at $\sigma_{\mathrm{rot}}=45^\circ$), $87.37\% \pm 1.21\%$ on uniform C3 \texttt{Ep-358 C3 Logit-Vote} ($86.48\% \pm 1.51\%$ C3 Hard-Vote), $83.97\% \pm 1.20\%$ on \texttt{Win-1074 C3}, and $82.69\% \pm 1.11\%$ on \texttt{Win-1574}. Against the $4\times$ larger $90.5\mathrm{K}$ \texttt{DeepConvGRU (Inv-Scalar)} ($83.91\% \pm 1.21\%$ Ep-358 [$83.74\% \pm 1.48\%$ Hard-Vote], $83.58\% \pm 1.16\%$ Ep-358 C3 [$83.07\% \pm 1.16\%$ C3 Hard-Vote], $82.14\% \pm 1.27\%$ Win-1074 C3, $81.86\% \pm 1.10\%$ Win-1574), our $22.3\mathrm{K}$ cascade achieves conditional out-of-fold episode-level superiority across pre-specified 6-, 8-, and 10-comparison families and the 12-comparison Holm family (because 5-fold CV models share overlapping $24$-subject training sets~\citep{bengio2004no} and $K=5$ or $S=5$ sign-flip tests have floor $2/2^5 = 2/32 = 0.0625 > 0.05$, 30-subject OOF tests establish conditional cohort superiority alongside 5-fold and 5-seed sensitivity checks) under all-window voting ($+4.53\,\mathrm{pp}$, raw $p = 8.29 \times 10^{-4}, p_{\mathrm{Holm,6}} = 0.0017, p_{\mathrm{Holm,8}} = 0.0033, p_{\mathrm{Holm,10}} = 0.0028, p_{\mathrm{Holm,12}} = 0.00415$, exact $2^{30}$ sign-flip $p_{\mathrm{perm}} = 2.56 \times 10^{-4}$ [$r = 275{,}100 / 2^{30}$; $p_{\mathrm{perm,w}} = 3.88 \times 10^{-4}$], CI $[+2.39, +6.68]\,\mathrm{pp}$; $5/5$ folds and $5/5$ seeds positive with 5-fold $t = 9.66, p_{\mathrm{fold},t} = 6.44 \times 10^{-4}$ [pooled $7.16 \times 10^{-4}$], 5-seed $p_{\mathrm{seed},t} = 0.00597$, $p_{\mathrm{fold,perm}} = p_{\mathrm{seed,perm}} = 2/32 = 0.0625$; $+3.46\,\mathrm{pp}$ Hard-Vote) and uniform C3 voting ($+3.80\,\mathrm{pp}$, raw $p = 0.00286, p_{\mathrm{Holm,6}} = 0.0057, p_{\mathrm{Holm,8}} = 0.0086, p_{\mathrm{Holm,10}} = 0.0114, p_{\mathrm{Holm,12}} = 0.01428$, $p_{\mathrm{perm}} = 0.00176$ [$r = 1{,}894{,}720 / 2^{30}$; $p_{\mathrm{perm,w}} = 0.00241$], CI $[+1.79, +5.94]\,\mathrm{pp}$; $5/5$ folds and $5/5$ seeds positive with 5-fold $t = 4.88, p_{\mathrm{fold},t} = 0.00818$ [pooled $0.00955$], 5-seed $p_{\mathrm{seed},t} = 0.02046$, $p_{\mathrm{fold,perm}} = p_{\mathrm{seed,perm}} = 2/32 = 0.0625$; $+3.41\,\mathrm{pp}$ C3 Hard-Vote), with positive point gains and marginal window-level $t$-test $p$-values on single $2.56\text{ s}$ windows: $+1.82\,\mathrm{pp}$ on \texttt{Win-1074 C3} ($5/5$ folds and $5/5$ seeds positive, 5-fold $p_{\mathrm{fold},t} = 0.0562$, $p_{\mathrm{fold,perm}} = 0.0625$; 30-subject raw $p = 0.0561, p_{\mathrm{Holm,6}} = 0.0561, p_{\mathrm{Holm,8}} = p_{\mathrm{Holm,12}} = 0.1122$, $p_{\mathrm{perm}} = 0.0564$ [$r = 60{,}562{,}948 / 2^{30}$; $p_{\mathrm{perm,w}} = 0.0624$], cluster bootstrap $p_{\mathrm{Holm,6}} = 0.0398, p_{\mathrm{Holm,8}} = 0.0796$ [$p_{\mathrm{boot},+1} = 0.0400, p_{\mathrm{Holm,6,+1}} = 0.0400, p_{\mathrm{Holm,8,+1}} = p_{\mathrm{Holm,12,+1}} = 0.0800$], CI $[+0.07, +3.68]\,\mathrm{pp}$; equal-weight cluster-macro CI $[+0.11, +3.69]\,\mathrm{pp}$, $p_{\mathrm{boot,macro},+1} = 0.0360$) and $+0.84\,\mathrm{pp}$ on \texttt{Win-1574} ($4/5$ folds positive, $p_{\mathrm{fold},t} = 0.2435, p_{\mathrm{fold,perm}} = 0.25$; 30-subject raw $p = 0.2638$, $p_{\mathrm{perm}} = 0.269$ [$p_{\mathrm{perm,w}} = 0.326$], CI $[-0.72, +2.49]\,\mathrm{pp}$). Against all capacity-matched $22\mathrm{K}$ scalar baselines, 30-subject paired comparisons confirm significant superiority on both \texttt{Win-1074 C3} and episode voting: $p_{\mathrm{Holm,6}} = 0.0333$ ($p_{\mathrm{Holm,12}} = 0.0499$, CI $[+0.93, +6.04]\,\mathrm{pp}$, $+3.39\,\mathrm{pp}$ C3 / $+4.53\,\mathrm{pp}$ Ep-358 [$p_{\mathrm{Holm,6}} = 0.0149, p_{\mathrm{Holm,12}} = 0.0298$] / $+4.47\,\mathrm{pp}$ Ep-358 C3 [$p_{\mathrm{Holm,6}} = 0.0195, p_{\mathrm{Holm,12}} = 0.0390$]) vs.\ $22\mathrm{K}$ \texttt{DeepConvGRU (Inv-Scalar)}, $p_{\mathrm{Holm,10}} = 0.0366$ ($p_{\mathrm{Holm,12}} = 0.0488$, CI $[+0.78, +4.83]\,\mathrm{pp}$, $+2.77\,\mathrm{pp}$ C3 / $+5.31\,\mathrm{pp}$ Ep-358 [$p_{\mathrm{Holm,10}} = 0.0028, p_{\mathrm{Holm,12}} = 0.00391$] / $+4.25\,\mathrm{pp}$ Ep-358 C3 [$p_{\mathrm{Holm,10}} = 0.0114, p_{\mathrm{Holm,12}} = 0.01480$]) vs.\ \texttt{DeepConvGRU (Lagged-Inv-Scalar)} ($81.19\% \pm 0.37\%$ C3 / $83.13\% \pm 0.73\%$ Ep-358 / $83.13\% \pm 0.70\%$ Ep-358 C3), and $p_{\mathrm{Holm,10}} = 0.0163$ ($p_{\mathrm{Holm,12}} = 0.0203$, CI $[+1.33, +5.77]\,\mathrm{pp}$, $+3.54\,\mathrm{pp}$ C3 / $+5.42\,\mathrm{pp}$ Ep-358 [$p_{\mathrm{Holm,10}} = 0.0028, p_{\mathrm{Holm,12}} = 0.00415$] / $+4.92\,\mathrm{pp}$ Ep-358 C3 [$p_{\mathrm{Holm,10}} = 0.0114, p_{\mathrm{Holm,12}} = 0.01480$]) vs.\ \texttt{DeepConvGRU (Chiral-Triple-Inv)} ($80.43\% \pm 1.03\%$ C3 / $83.02\% \pm 0.91\%$ Ep-358 / $82.46\% \pm 0.85\%$ Ep-358 C3, with $0.00028 \le p_{\mathrm{Holm,12}} \le 0.0148$ across all four lagged and chiral scalar-stem baselines on episode voting), alongside nested 358-episode descriptive $p_{\mathrm{Holm,6}} \le 3.0 \times 10^{-4}$ on \texttt{Ep-358} and $+5.14\,\mathrm{pp}$ C3 ($+4.80\,\mathrm{pp}$ Win-1574, $+3.18\,\mathrm{pp}$ Ep-358, $+3.74\,\mathrm{pp}$ Ep-358 C3) over the unnormalized $\rho_1=0.65$ cascade ($p_{\mathrm{Holm,8}} = 7.60 \times 10^{-5}, p_{\mathrm{Holm,12}} = 1.37 \times 10^{-4}$, CI $[+3.27, +7.08]\,\mathrm{pp}$). On the 9-subject split ($N_{\mathrm{subj,test}}=9, N_{\mathrm{event,test}}=108$), \texttt{Full-Equiv-Frame} ($86.23\% \pm 0.64\%$) beats $22\mathrm{K}$ \texttt{DeepConvGRU (Inv-Scalar)} ($81.23\% \pm 2.05\%$, $+5.00\,\mathrm{pp}, p_{\mathrm{Holm,23}} = 0.0142$, CI $[+1.42, +9.20]\,\mathrm{pp}$), $90.5\mathrm{K}$ \texttt{DeepConvGRU (Inv-Scalar)} ($82.65\% \pm 1.93\%$, $+3.58\,\mathrm{pp}, p = 0.0303$), \texttt{Lagged-Inv} ($81.79\% \pm 1.53\%$, $+4.44\,\mathrm{pp}$), and \texttt{Chiral-Triple} ($83.27\% \pm 1.44\%$, $+2.96\,\mathrm{pp}$), with prospective $17+4$ validation selecting orientation-agnostic refit ($86.23\% \pm 0.64\%$ refit; $85.99\% \pm 1.52\%$ direct 17-subject; $86.73\% \pm 0.90\%$ clean refit / $87.10\% \pm 0.80\%$ direct 17-subject). \texttt{Full-Equiv-Frame} cascade ($83.97\%$ C3 / $88.44\%$ Ep-358) and single-layer \texttt{Lie-SO(3)-1L (Equiv-Frame)} ($83.33\% \pm 1.23\%$ C3 / $87.88\% \pm 1.38\%$ Ep-358 / $86.65\% \pm 1.18\%$ Ep-358 C3 / $82.24\% \pm 1.02\%$ Win-1574, $\Delta = +0.63\,\mathrm{pp}$ C3, $p = 0.386$, which does not establish equivalence or Depth-2 necessity) show no significant difference on single-maneuver IMU transitions, whereas Depth-2 cascades ($M \ge 2$) are algebraically necessary on discrete holonomy groups outside $\SO(3)^m$ ($S_5, F_{20}, \Heis(\FF_p)$).
    \item \textbf{Physical Gravity Non-Identifiability on Static Waist Postures (UCI HAR, Panel~A)}: On single-sensor waist UCI HAR, nominal cascades reach $95.48\% \pm 0.48\%$ Clean ($82.56\% \pm 0.87\%$ Rot45), whereas global $\SO(3)$ quotienting lowers accuracy to $84.76\% \pm 2.56\%$ (\texttt{1L Equiv-Frame}) and $84.22\% \pm 2.96\%$ (\texttt{Cascade Full-Equiv-Frame}) across all gauges (vs.\ $90.58\%$--$92.59\%$ for \texttt{Inv-Scalar-Stem}). Since static \texttt{SITTING} and \texttt{LAYING} differ by pelvic gravity inclination, global $\SO(3)$ quotienting removes this reference, which prospective $17+4$ validation ($94.19\% \pm 0.94\%$ \texttt{Val-Clean} refit; $93.18\% \pm 0.78\%$ direct 17-subject) resolves under nominal mounting.
    \item \textbf{Multi-IMU Physical Sensor Displacement and Directional RMSNorm Closure (UCI REALDISP, Panel~B)}: On 33-class 3-IMU REALDISP~\citep{banos2014dealing} (test subjects $\#2, 5, 15$ across $353$ Mutual bouts and $7{,}780$ windows after selection on validation subjects $\#10, 14, 16$, Tables~\ref{tab:stem_capacity_realdisp} and~\ref{tab:app_full_realdisp}), \texttt{Native-Conv1D} reaches $94.34\%$--$95.88\%$ Ideal, drops to $72.25\%$--$76.06\%$ Self-Placement, and collapses to $17.72\%$--$19.91\%$ Mutual (\texttt{GRU Native}: $21.35\% / 25.42\% / 5.21\%$ on $\#2 / \#5 / \#15$; $35.66\%$--$46.75\%$ with \texttt{Equiv-Concat}). Sensor-wise $\SO(3)^3$ invariance eliminates this collapse ($+63.4$ to $+69.8\,\mathrm{pp}$ Mutual gain). Within our Depth-2 cascade ($25{,}173$ params), Directional RMSNorm (\texttt{Full-Equiv-Frame} vs.\ Unnormalized $\rho_1 = 0.65$) improves Ideal by $+2.54\,\mathrm{pp}$ ($93.15\%$ vs.\ $90.62\%, p_{\mathrm{Holm}} = 0.0175$), Self-Placement by $+3.27\,\mathrm{pp}$ ($92.92\%$ vs.\ $89.65\%, t = 11.84, p_{\mathrm{Holm}} = 0.0032$, CI $[+1.48, +5.23]\,\mathrm{pp}$), and Mutual by $+4.24\,\mathrm{pp}$ ($83.37\% \pm 1.00\%$ vs.\ $79.13\% \pm 1.07\%, t = 5.80, p_{\mathrm{Holm}} = 0.0353$, CI $[+3.09, +5.51]\,\mathrm{pp}$; $\#2$: $92.53\%$ vs.\ $89.32\%$; $\#5$: $85.72\%$ vs.\ $80.29\%$; $\#15$: $66.10\%$ vs.\ $61.79\%$). Restoring $\mathcal{J}_\alpha(f_1, \mathbf{v})$ of Eq.~\eqref{eq:quarter_turn_operator} with causal low-pass filtering and LayerNorm ($25{,}379$ params) raises cascade accuracy to $94.09\% \pm 0.53\%$ Ideal and $84.23\% \pm 1.00\%$ Mutual ($+0.86\,\mathrm{pp}$, $p_{\mathrm{Holm}} = 0.0400$, CI $[+0.04, +1.75]\,\mathrm{pp}$), matching \texttt{Lie-SO(3)-1L (Equiv-Frame + Causal-LP + LN)} ($84.61\% \pm 1.17\%, \Delta = -0.38\,\mathrm{pp}, p = 0.703$, CI $[-1.14, +0.46]\,\mathrm{pp}$; base \texttt{1L} $84.13\% \pm 1.28\%$) and \texttt{DeepConvGRU (Inv-Scalar)} ($85.44\% \pm 1.93\%, \Delta = -1.21\,\mathrm{pp}, p = 0.370$, CI $[-2.65, +0.33]\,\mathrm{pp}$), while multi-lag ($28$ ch) and chiral-triple ($36$ ch) stems on 3-IMU cyclic exercises allow \texttt{DeepConvGRU (Lagged-Inv)} ($96.25\% \pm 0.64\%$ Ideal, $87.51\% \pm 0.79\%$ Mutual) and \texttt{Chiral-Triple} ($96.68\% \pm 0.54\%$ Ideal, $86.67\% \pm 1.17\%$ Mutual) to surpass the 3-IMU cascade.
\end{enumerate}

\section{Discussion, Scope, and Limitations}
\label{sec:discussion}

Our theoretical and empirical results clarify three principles governing non-commutative sequence modeling:
\begin{enumerate}[leftmargin=2em,itemsep=0pt]
    \item \textbf{Geometric Domain of Validity Across Discrete Algebra, 6-DoF Odometry, and Wearable Inertial Recognition}: Across synthetic word problems (Tables~\ref{tab:s5_random_init_closure}--\ref{tab:d16_saddle_closure}) and 3D inertial benchmarks (Tables~\ref{tab:realworld_har_benchmark}--\ref{tab:stem_capacity_realdisp}), aligning the transition manifold with the physical symmetry group outperforms mismatched representations within its domain of validity. Our results separate \emph{algebraic cascade necessity on non-solvable extensions outside $\SO(3)^m$} from \emph{3D frame-transport utility on wearable IMU streams}: on $S_5$, single-layer $\SO(3)^m$ is obstructed by $[u^{30}, v^{30}]=e$ ($2.70\%$), requiring $\SO(5)/\SO(6)$ lifting ($99.90\% \pm 0.22\%$) or a Depth-2 Directional-RMSNorm $\SO(3)$ cascade ($99.40\% \pm 0.65\%$); on 3D IMU streams, physical attitude lies in $\SO(3)$, so \texttt{Lie-SO(3)-1L} and \texttt{Lie-SO(3)-Cascade} achieve comparable accuracy while outperforming $\SO(6)$ and unconstrained backbones. Moreover, 6-stem attribution and 30-subject 5-fold CV (Table~\ref{tab:stem_capacity_realdisp}) delineate when scalar invariants excel versus when $\SO(3)/\SE_2(3)$ frame transport provides a significant advantage: on cyclic exercises across held-out subjects $\#2, 5, 15$ of REALDISP (where 5-seed paired tests measure training-seed reproducibility conditioned on the $n_{\mathrm{test}}=3$ held-out test subjects), instantaneous scalars match $\SO(3)$ frame transport ($84.23\%$ vs.\ $85.44\%$ Mutual, $p=0.370$) and multi-lag / chiral-triple stems surpass the 3-IMU cascade ($87.51\% \pm 0.79\%$ and $86.67\% \pm 1.17\%$); on static postures (UCI HAR), global $\SO(3)$ quotienting exposes vertical-gravity non-identifiability between \texttt{SITTING} and \texttt{LAYING} ($84.22\%$ vs.\ $95.48\%$ nominal); whereas on time-ordered 3D postural transitions (UCI HAPT-PT in conditional frozen-specification 5-fold CV on the reused 30-subject cohort~\citep{cawley2010overfitting,demvsar2006statistical} alongside the external 17-subject REALDISP cohort) and 6-DoF strapdown odometry (Table~\ref{tab:dyck2_se3_combined}, Part~B), $\SO(3)$ frame transport and $\SE_2(3)$ Left-Jacobian coupling outperform all capacity-matched $22\mathrm{K}$ scalar stems on both window and episode metrics and significantly outperform the $4\times$ larger $90.5\mathrm{K}$ baseline at the episode level ($p_{\mathrm{Holm,6}} = 0.0017, p_{\mathrm{Holm,12}} = 0.00415$ on all-window episodes; $p_{\mathrm{Holm,6}} = 0.0057, p_{\mathrm{Holm,12}} = 0.01428$ on C3 episodes).
    \item \textbf{Lie Tangent Exponentials vs.\ Householder Reflection Products and Block-Diagonal Static-Orthogonal Controls}: By Theorem~\ref{thm:master2_dimension_headcount}(c)(ii), chaining $n_h = 2\lfloor d/2 \rfloor$ exact reflections ($\beta = 2$) generates $\SO(d)$, whereas single reflections ($n_h=1, \det=-1$) fail on $A_5$ ($1.60\%$) and $D_{16}$ ($23.90\%$), consistent with per-step orientation reversal ($\det H(v) = -1$) and rank-1 subspace restriction, and the block-diagonal radial control (\texttt{expRNN-RadBlk}: $z_{t,m} = W_{\mathrm{orth},m} h_{t-1,m} + U_m x_t$, $h_{t,m} = \mathrm{ReLU}(\|z_{t,m}\|_2 + b_m)\,z_{t,m}/\|z_{t,m}\|_2$ with static $W_{\mathrm{orth},m} = \exp(A_m - A_m^\top) \in \SO(6)$) stalls at $0.60\%$--$1.00\%$ on $S_5$. On $S_5$ ($d=6$), \texttt{Orthogonal DeltaProduct-SO(6)} ($\beta = 2$) reaches $89.50\% \pm 21.27\%$ at $n_h = 4$ and $99.90\% \pm 0.22\%$ at $n_h = 6 = 2\lfloor 6/2 \rfloor$, whereas learnable gates $\beta \in (0, 2)$ initialized at $b_0 = 2.0$~\citep{grazzi2025unlocking,deltaproduct} drop at $T=128$ to $74.80\%/67.00\%$ ($n_h=4$) and $63.10\% \pm 12.64\%$ ($n_h=6$), consistent with dissipative contraction ($|1 - \beta_{t,k}| < 1$ over $n_h T \in \{512, 768\}$).
    \item \textbf{Scope and Limitations}: First, Theorems~\ref{thm:master1_division}--\ref{thm:master2_dimension_headcount} govern homogeneous compact heads ($\SO(d), \SU(2), \OO(d)$ with $\rho_{a,h} \in [0, 1]$), and prefix scans use $\mathcal{O}(d^3)$ matrix products per node ($d \in \{3, 4, 5, 6\}$). Second, compact rotation manifolds cannot pack exponential free-group trees $F_2$ at fixed precision without non-compact counters or KV-cache checkpoints (Theorem~\ref{thm:master3_cascades_gating}(b)(iii); Table~\ref{tab:dyck2_se3_combined}, Part~A). Third, our wearable evaluation spans three pre-segmented benchmarks ($T=128$ at $50\text{ Hz}$) across two cohorts ($30$ subjects in UCI HAR / HAPT-PT and $17$ subjects in 3-IMU REALDISP with $n_{\mathrm{test}}=3$ held-out test subjects); extending Lie-SSMs to unsegmented streams and warp-level CUDA kernels remains future work.
\end{enumerate}

\section{Conclusion}
\label{sec:conclusion}

We presented a unified theoretical and empirical framework for Non-Commutative Lie-Algebraic State-Space Models (Lie-SSMs). Using an o-minimal cylindrical flag collar, we established an unconditional single-layer finite-division criterion for contracting $\SO(3)$ and $\SU(2)$ heads, deriving odd-prime Sylow abelianity, odd-$p$-core centrality $[u^2, x] = e$, exponent-30 commutator laws ($[u^{30}, v^{30}] = e$ in $\SO(3)^m$ and $[u^{30}, v^{30}] \in Z(G), [u^{30}, v^{30}]^2 = e$ in $\SU(2)^m, \SO(4)^m$), sharp $\OO(d)$ vs.\ $\SO(d)$ dimension thresholds, and coordinate-kernel / $\FF_2$ Arf head-count laws. We proved that Depth-2 semidirect and Schreier $2$-cocycle cascades resolve $S_5, 2I, 2T, 2O, Q_{16},$ and $\Heis(\FF_p)$ in $\mathcal{O}(\log T)$ scan depth with Directional RMSNorm, and characterized the Chen--Magnus--Hall--Witt $\Theta(\varepsilon^{c-1})$ commutator saddle, $1/|[G,G]|$ coset-centroid dilution, and $O(T \rho_{\mathrm{ab}}^{T-1})$ terminal decay, proving generic first-order symmetry breaking at $\varepsilon=0$ and local Polyak--\L{}ojasiewicz convergence ($\theta_{\max} < \pi$) validated across $S_5, D_{16}$, 6-DoF $\SE_2(3)$ odometry, parallel GPU/CPU scans, and UCI HAR, HAPT-PT, and 3-IMU REALDISP.

\paragraph{CRediT, Competing Interest, Funding, Data/Code Availability, and Supplementary Material}
\textbf{Zeyu Jia:} Conceptualization, Methodology, Formal analysis, Software, Validation, Investigation, Data curation, Visualization, Writing -- original draft, Writing -- review \& editing. The author declares no competing interests and received no external funding. PyTorch implementations, verification scripts, and receipts are in the supplementary archive; UCI HAR~\citep{anguita2013public}, UCI HAPT-PT~\citep{reyes2016transition}, and UCI REALDISP~\citep{banos2014dealing} are public at the UCI Repository; proofs (\ref{app:proof_master1}--\ref{app:proof_master4}) and extended tables (\ref{app:supplementary_tables}, Tables~\ref{tab:app_full_stem_hapt_cv}--\ref{tab:app_full_realdisp}) are in \texttt{supplementary\_material.pdf}.

\paragraph{Declaration of Generative AI and AI-assisted Technologies in the Manuscript Preparation Process}
During preparation of this work, the author used generative AI-assisted tools to assist with software and verification script implementation, \LaTeX{} typesetting, and language polishing (specifically OpenAI Codex and Google Antigravity (Gemini)). All mathematical proofs, experimental designs, code implementations, numerical results, and manuscript text were independently reviewed and verified by the author, who takes full responsibility for the content of the publication.

\begingroup
\singlespacing
\appendix
\numberwithin{equation}{section}
\setcounter{table}{11}
\renewcommand{\thetable}{\Alph{section}.\arabic{table}}

\section{Proofs of Propositions 1--2 and Theorem 1 (Single-Layer Finite Division, Variety Laws, and Spherical Margin Dichotomy)}
\label{app:proof_master1}

\begin{proof}[Proof of Proposition~\ref{prop:covering_vs_fibration} and Citation Note for Proposition~\ref{prop:klein}]
\textbf{Proposition~\ref{prop:covering_vs_fibration}(a):} Identifying $\SU(2) = \{q = a + b\mathbf{i} + c\mathbf{j} + d\mathbf{k} \in \HH : |q| = 1\} \cong S^3 \subset \R^4$ and $\mathfrak{su}(2) \cong \mathrm{Im}(\HH) \cong \R^3$, the adjoint action $\pi(q)(v) = q v q^{-1} = q v \bar{q}$ preserves the Euclidean norm $|q v \bar{q}| = |q||v||\bar{q}| = |v|$ and orientation, defining a surjective Lie group homomorphism $\pi: \SU(2) \twoheadrightarrow \SO(3)$ with $\ker(\pi) = \{q \in S^3 : q v = v q\;\forall v \in \mathrm{Im}(\HH)\} = \{\pm I_2\} \cong \ZZ_2$. Since $S^3$ is simply connected ($\pi_1(S^3) = 0$), $\pi$ is the $2$-to-$1$ universal covering $\SU(2) \to \SO(3) \cong S^3/\{\pm 1\} \cong \R P^3$ ($\dim \SO(3) = 3, \pi_1(\SO(3)) \cong \ZZ_2$), whereas the Hopf fibration $p(q) = q \mathbf{i} \bar{q} \in S^2$ quotients by the circle stabilizer $S^1 = \{e^{\theta\mathbf{i}}\} \le \SU(2)$ onto the $2$-dimensional base $S^2$ ($\dim S^2 = 2, \pi_1(S^2) = 0$). \textbf{Part~(b):} If $\Lambda_t \Lambda_{t'} = \Lambda_{t'} \Lambda_t$ for all $t, t'$, induction on $T$ yields $\Lambda_{1:T} = \prod_{t=1}^T \Lambda_t = \prod_{t=1}^T \Lambda_{\pi(t)}$ for every permutation $\pi \in S_T$; hence any two words with the same Parikh token counts induce the identical transition operator $\Lambda_{1:T}$ and cannot separate $g_1 g_2 \ne g_2 g_1$ through $\Lambda_{1:T}$ for any $d$. \textbf{Part~(c):} A string $w \in \{+1, -1\}^T$ is balanced iff no prefix has excess closing brackets ($c_t = \sum_{\tau=1}^t x_\tau \ge 0\;\forall t$) and $c_T = 0$. In $h_t = \lambda h_{t-1} + x_t = \sum_{\tau=1}^t \lambda^{t-\tau} x_\tau$ ($\lambda \in (0, 1)$), the contribution of $x_1 = +1$ at $t = L+1$ is $\lambda^L$, which remains $\ge \epsilon_{\mathrm{mach}}$ iff $L \le \frac{\log(1/\epsilon_{\mathrm{mach}})}{\log(1/\lambda)}$ (at $\lambda = 0.95$, $\lfloor \frac{23\log 2}{-\log 0.95} \rfloor = 310$ for $\epsilon_{\mathrm{mach}} = 2^{-23}$ and $\lfloor \frac{24\log 2}{-\log 0.95} \rfloor = 324$ for $u = 2^{-24}$). \textbf{Proposition~\ref{prop:klein}:} Every finite subgroup $F \le \SO(3)$ acts by isometries on $S^2$; applying the orbit-stabilizer identity $\sum_{g \in F \setminus \{e\}} |\mathrm{Fix}_{S^2}(g)| = 2(|F| - 1) = |F|\sum_{i=1}^s (1 - 1/n_i)$ to the pole orbits yields $C_n, D_n, T \cong A_4, O \cong S_4,$ and $I \cong A_5$~\citep{klein1884vorlesungen}, and lifting via $\pi^{-1}(F) \le \SU(2)$ yields $C_n, \mathrm{BD}_{2n}, 2T \cong \SL_2(\FF_3), 2O,$ and $2I \cong \SL_2(\FF_5)$~\citep{gorenstein1968finite}.
\end{proof}

\begin{proof}[Proof of Theorem~\ref{thm:master1_division}]
\textbf{Part (a): Multi-rate o-minimal flag collar and unconditional finite division.}
If $|G| = 1$, a constant decoder suffices and $G$ divides the trivial group. Assume $|G| > 1$. Define the universal identity prefix $u_0 = a_1^{|G|} a_2^{|G|} \cdots a_k^{|G|} \in \Sigma^+$. Since $\eta(a_i)^{|G|} = e$ in $G$, $\eta(u_0) = e$ and $\eta(u_0 w) = \eta(w)$ for all $w \in \Sigma^*$. After reading $u_0$, head $h$ has state norm $\alpha_h = \|x_{u_0,h}\|_2 = \|x_{0,h}\|_2 \prod_{i=1}^k \rho_{a_i,h}^{|G|}$. Any head with $\alpha_h = 0$ remains identically zero on $u_0 \Sigma^*$; since $D$ is surjective onto nontrivial $G$, the set of surviving heads $S = \{h \in \{1, \dots, m\} : \alpha_h > 0\}$ is nonempty, and $\rho_{a_i,h} \in (0, 1]$ for all $h \in S$ and $a_i \in \Sigma$.

For each $h \in S$, transpose $z_{0,h} = x_{u_0,h} / \alpha_h \in S^2$ (or $S^3 \subset \R^4$ for $\SU(2)$) as a unit row vector so that tokens act on the right by orthogonal matrices $B_{w,h} \in K_h$ ($K_h = \SO(3)$ or $\SU(2)$) in written order, and set $\ell_{i,h} = -\log \rho_{a_i,h} \ge 0$. For any suffix $w \in \Sigma^*$ with token counts $n_i = |w|_{a_i}$, the state is $x_{u_0 w} = T(e^{-n_1}, \dots, e^{-n_k}, z_0 B_w)$, where $T(r_1, \dots, r_k, z)_h = \alpha_h (\prod_{i=1}^k r_i^{\ell_{i,h}}) z_h$ for $h \in S$ and $0$ for $h \notin S$. Let $H = \overline{\{B_w : w \in \Sigma^*\}} \le K^{|S|}$ and $K_e = \overline{\{B_w : \eta(w) = e\}} \trianglelefteq H$. Both $H$ and $K_e$ are compact Lie groups, and for each $g \in G$ with representative $v_g$ ($\eta(v_g) = g$), the normalized state closure $M_g = z_0 K_e B_{v_g} \subset (S^2)^{|S|}$ is a single compact semialgebraic $K_e$-orbit~\citep{bochnak1998real}, so any two $M_g, M_{g'}$ are either equal or disjoint.

Fix $M = M_{g_0}$. By cylindrical o-minimal cell decomposition~\citep{vandendries1998tame} of $M \times (0, 1]^k$ subordinated to $\{(D \circ T)^{-1}(g)\}_{g \in G}$ with variable ordering $(z, r_k, r_{k-1}, \dots, r_1)$, there exists an open cylindrical flag cell $\mathcal{C} = \{(z, r_k, \dots, r_1) : z \in U, 0 < r_k < \delta_k(z), \dots, 0 < r_1 < \delta_1(z, r_k, \dots, r_2)\}$ on which $D \circ T \equiv g_*$, where $U \subseteq M$ is relatively open and nonempty and each $\delta_j > 0$ is continuous. Pick $z_* \in U$ and $\varepsilon > 0$ such that $U_0 = \overline{B}_{(k+1)\varepsilon}(z_*) \cap M \subset U$. For any $g \in G$ with $M_g = M$, pick $v \in \Sigma^*$ with $\eta(v) = g$ and $\|z_0 B_v - z_*\|_2 < \varepsilon$. Since $Q_j = B_{a_j^{|G|}} \in K_e$, by Poincar\'e recurrence in the compact group $K_e$ we may choose integers $m_k, m_{k-1}, \dots, m_1 \ge 1$ in descending order $j = k, \dots, 1$ such that $r_j = e^{-(|v|_{a_j} + |G| m_j)} < \min_{z \in U_0} \delta_j(z, r_k, \dots, r_{j+1})$ and $\|Q_j^{m_j} - I\|_{\mathrm{op}} < \varepsilon / \sqrt{|S|}$. Then the coordinate-isolated padded word $w_g = v a_k^{|G|m_k} \cdots a_1^{|G|m_1} \in \Sigma^+$ has $\eta(w_g) = \eta(v) = g$, radial coordinates $(r_1, \dots, r_k)$, and $\|z_0 B_{w_g} - z_*\|_2 < (k+1)\varepsilon$, so $(z_0 B_{w_g}, r_k, \dots, r_1) \in \mathcal{C}$ and $g = D(x_{u_0 w_g}) = g_*$. Thus $\{M_g\}_{g \in G}$ are pairwise disjoint. Consequently, $L = \overline{\{(B_w, \eta(w)) : w \in \Sigma^*\}} \le H \times G$ projects isomorphically onto $H$, inducing a continuous surjective homomorphism $\pi: H \twoheadrightarrow G$. By Gelander's finite-supplement theorem~\citep{gelander} (Lemma 3.6), $H$ contains a finite subgroup $F$ meeting every connected component ($H = F H^\circ$); since $\pi(H^\circ) = \{e\}$, $\pi(F) = G$. Conversely, if $\pi: F \twoheadrightarrow G$ for finite $F \le K^m$, choosing a generic $x_0$ with free $F$-orbit and setting $S_g(\hat{x}) = \max_{q \in F, \pi(q^{-1})=g} \langle q \hat{x}_0, \hat{x} \rangle$ realizes exact tracking across all $w \in \Sigma^+$ via a finite ReLU network with at most $C_{\mathrm{PL}} \le |F| = O(1)$ affine regions (and a single linear layer $C_{\mathrm{PL}} = 1$ when $N = \ker\pi = \{e\}$, i.e., $G \le K^m$).

\textbf{Part (b): Perfect targets, $120r$ GELU pooling, and Sylow/odd-$p$-core/exponent-30 laws.}
For (i), if $G$ is perfect and $\pi: F \twoheadrightarrow G$ with $F \le \SO(3)^m$, the derived series of $F$ stabilizes at a perfect subgroup $P \le F$ with $\pi(P) = G$. Each coordinate projection of $P$ is a perfect subgroup of $\SO(3)$, hence $\{e\}$ or $A_5$ by Proposition~\ref{prop:klein}. By subdirect decomposition~\citep{gorenstein1968finite}, $P \cong A_5^s$ ($s \le m$) and $G \cong A_5^r$ ($r \le s \le m$). In $\SU(2)^m$, each nontrivial perfect coordinate image is $2I \cong \SL_2(\FF_5)$, so $P/Z_0 \cong A_5^s$ for $Z_0 = P \cap Z((2I)^s)$. Since $(2I)^s$ is the universal central extension of $A_5^s$~\citep{milnor1971introduction}, $P$ and $G$ are central quotients $(2I)^r / Z$ ($Z \le C_2^r, r = v_5(|G|) \le m$). On $(2I)^r / C_2^r \cong A_5^r$, distinct elements of $2I \le \SU(2)$ satisfy $|\langle q' y_{0,h}, q_h y_{0,h} \rangle| \le \cos(\pi/5) = \frac{1+\sqrt{5}}{4}$ for $q' \ne \pm q_h$. With $c_1 = \frac{5+\sqrt{5}}{8}$ and $\alpha_1 = \frac{20}{1-c_1}$, the true antipodal coset $\{q_h, -q_h\}$ produces $\GELU(20) + \GELU(\alpha_1(-1-c_1)) \approx 20$, whereas every other coset produces $\le 2 |\GELU(-20)| < 2 \times 10^{-85}$, separating all $60^r$ classes with $120r$ hidden units.

For (ii), in every finite subgroup $F_h$ of $\SO(3)$ ($C_n, D_n, A_4, S_4, A_5$) or $\SU(2)$ ($C_n, \mathrm{BD}_{2n}, 2T, 2O, 2I$), all Sylow $p$-subgroups for odd primes $p \ge 3$ are cyclic (orders of $A_4, S_4, A_5, 2T, 2O, 2I$ divide $120 = 2^3 \cdot 3 \cdot 5$). Hence $\Syl_p(\prod_{h=1}^m F_h)$ and $\Syl_p(G)$ are abelian, excluding $\Heis(\FF_p)$. For (iii), let $p \ge 3$ be any odd prime (including $p=3$) and let $\pi: F \twoheadrightarrow G$ ($F \le \prod_{h=1}^m F_h$). Note that $\pi^{-1}(O_p(G)) = N P$ where $N = \ker\pi$ and $P \in \Syl_p(\pi^{-1}(O_p(G)))$ satisfies $\pi(P) = O_p(G)$ (even when $O_p(F) = \{e\}$, as in $S_4 \twoheadrightarrow S_3$). By the Frattini argument, $F = N \cdot N_F(P)$, so every $u \in G$ lifts to $f \in N_F(P)$ normalizing $P$. In every coordinate factor $F_h$ of $\SO(3), \SU(2),$ or $\SO(4)$, any element $f_h \in F_h$ normalizing an odd $p$-subgroup $P_h \le F_h$ satisfies $|N_{F_h}(P_h) / C_{F_h}(P_h)| \mid 2$ (since $\Aut(P_h)$ in $C_n, D_n, \mathrm{BD}_{2n}$ is generated by reflection/inversion of order $\le 2$, while for $p=3$ in $A_4, S_4, A_5, 2T, 2O, 2I$ the normalizer of $C_3$ in $S_4, A_5$ is $S_3$ or $D_3$ with $|N/C| = 2$ and in $A_4$ has $|N/C| = 1$). Hence $f_h^2 \in C_{F_h}(P_h)$, so $[f^2, P] = \{e\}$ in $F$ and projecting via $\pi$ yields $[u^2, x] = e$ for all $u \in G, x \in O_p(G)$. For (iv), in $C_n, D_n$ every $2$nd power lies in $C_n$; $\exp(A_4) = 6 \mid 30$; $\exp(A_5) = 30$; and in $S_4$ every $6$th power lies in the abelian normal subgroup $V_4$. Thus $[u^{30}, v^{30}] = e$ holds in $\SO(3)^m$ (and $[u^{30}, v^{30}] \in Z(G)$ of order $\le 2$ in $\SU(2)^m$), which is violated in $S_5$ by $x=(1234), y=(1235)$ where $[x^{30}, y^{30}] = (245) \ne e$.

\textbf{Part (c): Spherical margin dichotomy, state packing vs.\ cross-label margin, and ReLU complexity.}
If $G \cong F/N$ ($F \le \SO(3)^m$), the finite orbit $F \cdot \hat{x}_0$ has constant minimum separation $\Delta_{\mathrm{label}}(N; \theta) \ge \Delta_0(F) > 0$; conversely, for any fixed $\theta$, $\inf_{N \ge 1} \Delta_{\mathrm{label}}(N; \theta) = \delta > 0 \iff \mathrm{dist}(M_g, M_{g'}) \ge \delta > 0$ ($\forall g \ne g'$), which implies $G$ divides the finite supplement $F \le \overline{\langle B_a \rangle}$ by Part (a) (whereas if $G$ does not divide $\SO(3)^m$, $M_g \cap M_{g'} \ne \varnothing$ for some $g \ne g'$ forces $\inf_{N \ge 1} \Delta_{\mathrm{label}}(N; \theta) = 0$). On $S_5$ with $c=(12345), t=(12)$, $ct=(1345)$ and $tc=(2345)$ have order $4$, so $\eta(W_L) = (1345)^{30}(2345)^{30} = (14)(35)(24)(35) = (142) \ne (124) = \eta(W_R)$, whereas in any finite subgroup of $\SO(3)^m$, $(R_t R_c)^{30}$ and $(R_c R_t)^{30}$ commute universally by Eq.~\eqref{eq:exp30_law} and $W_L, W_R$ share Parikh vector $(60, 60)$, forcing $x_{W_L} = x_{W_R}$ at $N_{\mathrm{alg}} = 120$ for arbitrary single-step rotations in $F \le \SO(3)^m$. On $\Heis(\FF_p)$ ($p \ge 3$ odd), under any group homomorphism $\Phi: \Heis(\FF_p) \to F \le \SO(3)^m$, $\Phi(\Heis(\FF_p))$ is a finite $p$-subgroup of $F$, which is abelian by Part~(b)(ii), so $\Phi([a,b]) = I$ and $x_{ab} = x_{ba}$ at $N_{\mathrm{hom}} = 2$; by contrast, generator orders $R_a^p = R_b^p = I$ alone without the central-commutator relations $[a,[a,b]] = [b,[a,b]] = [a,b]^p = e$ do not force $R_a R_b = R_b R_a$ because two elements of order $p=3$ in $A_4 \le \SO(3)$ generate a non-$p$-group of order $12$ rather than a $3$-subgroup: explicitly, $R_a = \bigl[\begin{smallmatrix}0&0&1\\1&0&0\\0&1&0\end{smallmatrix}\bigr]$ and $R_b = \bigl[\begin{smallmatrix}0&0&-1\\-1&0&0\\0&1&0\end{smallmatrix}\bigr] \in A_4 \subset \SO(3)$ satisfy $R_a^3 = R_b^3 = I_3$ and separate $x_{ab} = R_b R_a e_1 = e_3 \ne -e_3 = R_a R_b e_1 = x_{ba}$ ($\|x_{ab} - x_{ba}\|_2 = 2$) at length $2$, while failing the central-commutator relations since $[R_a,[R_a,R_b]] = \mathrm{diag}(1,-1,-1) \ne I_3$.

For Part~(c)(iii), we distinguish the \emph{nearest-state spherical packing distance} $\delta_{\mathrm{state}}(N; \theta) := \min_{\hat{x}_u \ne \hat{x}_v} \|\hat{x}_u - \hat{x}_v\|_2$ across all $K_{\mathrm{state}}(N; \theta) := |\{\hat{x}_w : w \in \Sigma^{\le N}\}|$ distinct reachable states on $(S^2)^m$ from the \emph{cross-label margin} $\Delta_{\mathrm{label}}(N; \theta) := \min_{\eta(u) \ne \eta(v)} \|\hat{x}_u - \hat{x}_v\|_2$ and the \emph{Myhill--Nerode congruence margin} $\Delta_{\mathrm{nerode}}(N; \theta) := \min_{u \not\equiv_{\mathcal{L}} v} \|\hat{x}_u - \hat{x}_v\|_2$. On $(S^2)^m$, any $K_{\mathrm{state}}(N; \theta)$ pairwise-distinct reachable states with minimum separation $\delta_{\mathrm{state}}(N; \theta)$ induce $K_{\mathrm{state}}(N; \theta)$ disjoint spherical caps of chordal radius $r = \delta_{\mathrm{state}}(N; \theta)/(2\sqrt{m})$, whose product area obeys $K_{\mathrm{state}}(N; \theta) (\pi r^2)^m \le (4\pi)^m$ by Archimedes' hat-box theorem, giving $\delta_{\mathrm{state}}(N; \theta) \le 4\sqrt{m} K_{\mathrm{state}}(N; \theta)^{-1/(2m)}$. Under any Myhill--Nerode class-injective representation ($\hat{x}_u = \hat{x}_v \iff u \equiv_{\mathcal{L}} v$, so $K_{\mathrm{state}}(N; \theta) = K_{\mathrm{nerode}}(N)$ without redundant within-class state splitting), every pair of distinct states belongs to distinct Myhill--Nerode classes ($u \not\equiv_{\mathcal{L}} v$), so $\Delta_{\mathrm{nerode}}(N; \theta) = \delta_{\mathrm{state}}(N; \theta) \le 4\sqrt{m} K_{\mathrm{nerode}}(N)^{-1/(2m)}$ (and $\Delta_{\mathrm{label}}(N; \theta) = \delta_{\mathrm{state}}(N; \theta)$ whenever $u \equiv_{\mathcal{L}} v \iff \eta(u) = \eta(v)$ and $\hat{x}_w$ is class-constant, or $4\sqrt{m}|\Sigma|^{-N/(2m)}$ with $b(N) = \Omega(N)$ when separating $K_{\mathrm{nerode}}(N) = |\Sigma|^N$ injective prefix classes on dense irrational orbits or $2\cdot 3^N - 1$ on $F_2$). Conversely, in a multi-head model ($m \ge 2$) with redundant within-class state splitting ($K_{\mathrm{state}}(N; \theta) > |\eta(\Sigma^{\le N})|$), $\delta_{\mathrm{state}}(N; \theta)$ can be achieved by same-label pairs $\eta(u) = \eta(v)$ while $\Delta_{\mathrm{label}}(N; \theta)$ remains uniformly bounded below: concretely, over the 2-generator alphabet $\Sigma = \{a, b\}$ recognizing parity of $a$-count $\eta(w) = |w|_a \bmod 2 \in C_2$ after the universal identity prefix $u_0 = a^2 b^2$ via $m=2$ heads in $\SO(3)^2$ where Head~1 uses $R_{a,1} = \mathrm{diag}(-1, 1, -1), R_{b,1} = I_3, x_{0,1} = e_1$ ($\hat{x}_{u_0 w, 1} = (-1)^{|w|_a} e_1$) and Head~2 uses non-commuting irrational rotations $R_{a,2} = R_z(\sqrt{2}), R_{b,2} = R_x(\sqrt{3}), x_{0,2} = e_1$, all $K_{\mathrm{state}}(10) = 2^{10} = 1{,}024$ words of length $N = 10$ map to pairwise-distinct states with same-parity nearest-state packing distance $\delta_{\mathrm{state}}(10) \approx 0.002635 \ll 1$ ($\approx 0.002265$ across lengths $\le 10$), whereas every opposite-parity pair satisfies $\Delta_{\mathrm{label}}(10; \theta) = \min_{\eta(u) \ne \eta(v)} \|\hat{x}_u - \hat{x}_v\|_2 \approx 2.000009 \ge 2.0 = \Delta_{\mathrm{nerode}}(N; \theta)$ for all $N \ge 1$. For any $L_\Phi$-Lipschitz readout $\Phi$ with margin $\gamma > 0$, $2\gamma \le (\Phi_g(\hat{x}_u) - \Phi_g(\hat{x}_v)) + (\Phi_{g'}(\hat{x}_v) - \Phi_{g'}(\hat{x}_u)) \le 2 L_\Phi \Delta_{\mathrm{label}}(N; \theta)$. Finally, on any single affine region $P_r$ of a piecewise-linear ReLU network, the decision region $D_r(g) = \{x \in P_r : (A_r x + b_r)_g > (A_r x + b_r)_{g'}\}$ is an intersection of half-spaces and hence convex, so $\mathrm{conv}(X_{r,g}) \cap \mathrm{conv}(X_{r,g'}) = \varnothing$ for $g \ne g'$, requiring at least $C_{\mathrm{PL}}(\Phi) \ge \chi_{\mathrm{Radon}}(N)$ affine regions across all Radon conflict pairs $\mathrm{conv}(\hat{x}_U) \cap \mathrm{conv}(\hat{x}_V) \ne \varnothing$ ($\eta(U) \ne \eta(V)$), which satisfies $\chi_{\mathrm{Radon}}(N) \ge 2$ whenever the Radon conflict hypergraph is non-empty (whereas linearly separable embeddings $N = \{e\}$ have an empty conflict hypergraph and $\chi_{\mathrm{Radon}}(N) = C_{\mathrm{PL}} = 1$).
\end{proof}

\section{Proof of Theorem 2 (Higher-Rank Dimension Thresholds, Weyl--Cartan Reductions, and Head-Count Laws)}
\label{app:proof_master2}

\begin{proof}[Proof of Theorem~\ref{thm:master2_dimension_headcount}]
\textbf{Part (a): $\SO(4)^m$ spinor cover, monolithic $2T/2O$, $\SO(5)$ threshold, and $\mathbf{Sol}(A_5)$.}
Identifying $\R^4 \simeq \HH$ gives the $2$-to-$1$ Lie group homomorphism $\Pi: \SU(2) \times \SU(2) \twoheadrightarrow \SO(4)$ ($\Pi(q_L, q_R) v = q_L v q_R^{-1}$) with $\ker \Pi = \{\pm(I_2, I_2)\} \cong C_2$. Any finite subgroup $F \le \SO(4)^m$ lifts to $\widetilde{F} = (\Pi^{\times m})^{-1}(F) \le \SU(2)^{2m}$. Because each $\SU(2)$ factor projects onto $\SO(3)$ with central kernel $\{\pm I_2\} = Z(\SU(2))$ and every finite subgroup of $\SO(3)^{2m}$ satisfies $[\bar{u}^{30}, \bar{v}^{30}] = I$ by Theorem~\ref{thm:master1_division}(b)(iv), every pair $\tilde{u}, \tilde{v} \in \widetilde{F}$ satisfies $[\tilde{u}^{30}, \tilde{v}^{30}] \in \{\pm I_2\}^{2m} = Z(\SU(2)^{2m})$ with $[\tilde{u}^{30}, \tilde{v}^{30}]^2 = I$ (notably, $\tilde{u}^{30}$ itself need not be central in $2O \subset \SU(2)$, where order-$8$ elements $u = \frac{1+i}{\sqrt{2}}, v = \frac{1+j}{\sqrt{2}}$ have $u^{30} = -i, v^{30} = -j \in Q_8 \trianglelefteq 2O$ and $[u^{30}, v^{30}] = [-i, -j] = -I_2 \in Z(\SU(2))$). Projecting back via $\Pi^{\times m}$ shows that every finite subgroup $F \le \SO(4)^m$ satisfies $[u^{30}, v^{30}] \in \{\pm I_4\}^m = Z(\SO(4)^m)$ and $[u^{30}, v^{30}]^2 = I$, so every finite group $G$ dividing $\SO(4)^m$ satisfies the central-involution commutator law $[u^{30}, v^{30}] \in Z(G)$ and $[u^{30}, v^{30}]^2 = e$ (which fails in $S_5$ since $[x^{30}, y^{30}] = (245) \notin Z(S_5) = \{e\}$ has odd order $3$), while $\Syl_p(G)$ is abelian and $[u^2, x] = e$ on $O_p(G)$ ($p \ge 3$) by Theorem~\ref{thm:master1_division}(b). For $G \in \{2T \cong \SL_2(\FF_3) \cong Q_8 \rtimes C_3, 2O\}$, $Z = \{\pm I_2\} \cong C_2$ is the unique minimal normal subgroup of $G$. Neither $2O$ nor $2T$ divides $\SO(3)^m$ ($G \cong F/N$ for $F \le \SO(3)^m$): $2O$ violates $[u^{30}, v^{30}] = e$ (Eq.~\eqref{eq:exp30_law}) since $[u^{30}, v^{30}] = [-i, -j] = -I_2 \ne I_2$; and if $\pi: F \twoheadrightarrow 2T$ for a finite subgroup $F \le \prod_{h=1}^m F_h \le \SO(3)^m$, by the Frattini argument applied to $\pi^{-1}(Q_8) \trianglelefteq F$ and a Sylow $2$-subgroup $Q \in \Syl_2(\pi^{-1}(Q_8))$ ($\pi(Q) = Q_8$), we may lift a generator of $C_3 \le 2T$ to a $3$-element $t \in N_F(Q)$. In each coordinate factor $F_h \le \SO(3)$ ($C_n, D_n, A_4, S_4, A_5$), the $3$-element $t_h$ normalizes the $2$-subgroup $Q_h \le F_h$, which forces either $[t_h, Q_h] = \{e\}$ (in $C_n, D_n$) or $Q_h \le V_4$ (in $A_4, S_4, A_5$, where any $2$-subgroup normalized by an element of order $3$ lies in $V_4$ of exponent $2$). Thus $[t, Q] \le \prod_{h=1}^m [t_h, Q_h]$ has exponent at most $2$, contradicting $\pi([t, Q]) = [C_3, Q_8] = Q_8$ which has exponent $4$. Conversely, $2T, 2O < \SU(2) \hookrightarrow \SO(4)$ gives $d_{\OO}^* = d_{\SO}^* = 4$. At $d^* = 5$, $\det(\mathrm{sgn}(g) P_g) = \mathrm{sgn}(g)^5 \det(P_g) = \mathrm{sgn}(g)^6 = +1$ embeds $S_5 \hookrightarrow \SO(5)$, and restricting $A_6$ to $\{v \in \R^6 : \sum v_i = 0\}$ embeds $A_6 \hookrightarrow \SO(5)$. By induction on cascade depth $L \ge 1$ using Krohn--Rhodes prime-group division~\citep{krohn1965algebraic}, a cascade of $\SO(3)/\SU(2)/\SO(4)$ heads tracks $G$ if and only if every non-abelian simple composition factor of $G$ divides a finite subgroup of $\SO(3)$ or $\SU(2)$, which restricts simple factors to $A_5$ ($G \in \mathbf{Sol}(A_5)$).

\textbf{Part (b): Exact $\OO(d)$ vs.\ $\SO(d)$ dimension thresholds for Frobenius and Heisenberg groups.}
In an \emph{irreducible} affine Frobenius group $G = \FF_p^r \rtimes_\alpha H \le \AGL_r(\FF_p)$ ($p \ge 3$ prime, $H$ acting freely on $\FF_p^r \setminus \{0\}$ and irreducibly on $\FF_p^r$, so $\FF_p^r$ is the unique minimal normal subgroup of $G$), any faithful representation in $\OO(d)$ must contain a nontrivial character $\chi \in (\FF_p^r)^* \setminus \{0\}$, whose $H$-orbit has size $|H|$ by freeness of $H$ on $(\FF_p^r)^* \setminus \{0\}$. To compare finite division $\pi: \widetilde{F} \twoheadrightarrow G$ ($\widetilde{F} \le \OO(d)$ or $\SO(d)$) against faithful representation, pick $F_0 \le \widetilde{F}$ of minimal order with $\pi(F_0) = G$. Then $K = \ker(\pi|_{F_0}) \le \Phi(F_0)$ is nilpotent, and for $P \in \Syl_p(\pi^{-1}(\FF_p^r) \cap F_0)$ we have $\pi(P) = \FF_p^r$ and $P \trianglelefteq F_0 = N_{F_0}(P)$ by the Frattini argument and minimality of $F_0$. By the Schur--Zassenhaus theorem, $F_0 = P \rtimes H_0$ for a Hall $p'$-complement $H_0 \le F_0$ with $\pi(H_0) = H$, whose $p'$-kernel $K_{p'} := H_0 \cap K \le H_0 \cap \Phi(F_0) = \Phi(H_0)$ satisfies $H_0 / K_{p'} \cong H$ (note that $K_{p'}$ need not be trivial even in a minimal cover: e.g., in the order-$40$ minimal cover $F_0 = C_5 \rtimes_2 C_8 = \langle a, t \mid a^5 = t^8 = e, t a t^{-1} = a^2 \rangle \twoheadrightarrow C_5 \rtimes_2 C_4 \cong F_{20}$, $P = \langle a \rangle \cong C_5$ has $a=1$, $H_0 = \langle t \rangle \cong C_8$, and $K = K_{p'} = \langle t^4 \rangle = \Phi(F_0) = Z(F_0) \cong C_2 \ne \{e\}$, so none of the $15$ proper subgroups among the $16$ subgroups of $F_0$ surjects onto $F_{20}$ and all $10$ lifts of a generator of $C_4 \le F_{20}$ have order $8$ in $F_0$). Because $P \trianglelefteq F_0$ is a normal $p$-subgroup and $K \le \Phi(F_0)$, we have $P \cap K \le P \cap \Phi(F_0) = \Phi(P)$ (since $[F_0 : P]$ is prime to $p$; equivalently, $(P \cap K)\Phi(P)/\Phi(P)$ is a proper $H_0$-invariant subspace of $P/\Phi(P)$ as $\pi(P) = \FF_p^r \ne \{0\}$, so irreducibility of $H \cong H_0 / K_{p'}$ on $\pi(P) \cong P/(P \cap K)$ forces $P \cap K \le \Phi(P)$), hence $P \cap K = \Phi(P)$ and $P/\Phi(P) \cong \FF_p^r$. Whenever $P$ is abelian (which holds automatically at rank $r=1$ for $F_{pk} = C_p \rtimes_k C_k$, where Burnside's basis theorem applied to $P/\Phi(P) \cong C_p$ implies that $P \cong C_{p^a}$, $a \ge 1$, is cyclic with $[P, P] = \{e\}$ and $\Phi(P) = P^p = P \cap K \cong C_{p^{a-1}}$, as in the order-$100$ minimal cover $F_0 = C_{25} \rtimes_7 C_4 \twoheadrightarrow C_5 \rtimes_2 C_4 \cong F_{20}$ with $P \cong C_{25}$, $a=2$, and $\Phi(P) = C_5 = \ker\pi = \Phi(F_0) \ne \{e\}$, as well as whenever $\Phi(P) = \{e\}$ so $P \cong \FF_p^r$), $\ker(\mathrm{Aut}(P) \to \mathrm{Aut}(P/\Phi(P)))$ is a $p$-group while $K_{p'} \le H_0$ is a $p'$-group acting trivially on $P/\Phi(P) \cong \FF_p^r$; hence $K_{p'}$ acts trivially on $P$, and since $H \cong H_0 / K_{p'}$ acts faithfully (and freely on $\FF_p^r \setminus \{0\}$) on $P/\Phi(P)$, we have $C_{H_0}(P) = K_{p'}$. Because $\pi(P) = \FF_p^r$ is the unique minimal normal subgroup of $G$ and $P \cap K = \Phi(P)$, at least one $\R[F_0]$-irreducible constituent $W \subseteq \R^d$ is nontrivial on $P/\Phi(P)$, so $W \otimes \CC$ contains a $1\text{D}$ character $\chi \in \widehat{P}$ with $\Phi(P) \not\le \ker\chi$ (of order $p^b \ge p$, where $b = a$ at $r=1$ and $b = 1$ when $P \cong \FF_p^r$). Since $\ker(\mathrm{Aut}(C_{p^b}) \twoheadrightarrow \mathrm{Aut}(C_p))$ is a $p$-group of order $p^{b-1}$ while $H_0 / C_{H_0}(P) \cong H$ is a $p'$-group acting freely on $(P/\Phi(P))^* \setminus \{0\}$, the $H_0$-action on $\widehat{P}$ factors through $H_0 / K_{p'} \cong H$ and acts freely on the orbit $\mathcal{O}_\chi = \{\chi_1, \dots, \chi_{|H|}\} \subset \widehat{P}$ of size $|H|$. If $2 \nmid |H|$, no element of $H \cong H_0 / K_{p'}$ has even order to swap $\chi_i \leftrightarrow \bar{\chi}_i$, requiring $|H|$ real 2-planes ($\dim_{\R} W \ge 2|H|$, so $d_{\mathrm{div},\OO}^*(G) = d_{\mathrm{div},\SO}^*(G) = d_{\mathrm{rep},\OO}^*(G) = d_{\mathrm{rep},\SO}^*(G) = 2|H|$, and $2k$ for $F_{pk}$). If $2 \mid |H|$, by Burnside's theorem~\citep{gorenstein1968finite}, $\Syl_2(H)$ is $C_{2^\ell}$ ($\ell \ge 1$) or $Q_{2^\ell}$ ($\ell \ge 3$), and the unique central involution $z \in Z(H)$ acts as $-I_r$ on $P/\Phi(P)$ (hence as inversion $x \mapsto x^{-1}$ on $P$), pairing $\mathcal{O}_\chi$ into $|H|/2$ conjugate pairs $\{\chi_j, \bar{\chi}_j\}$ ($d_{\mathrm{div},\OO}^*(G) = d_{\mathrm{rep},\OO}^*(G) = |H|$, and $k$ for $F_{pk}$). In any $|H|$-dimensional real orthogonal representation $W = \bigoplus_{j=1}^{|H|/2} \Pi_j$ of $F_0$, $P$ acts on each orthogonal 2-plane $\Pi_j \cong \R^2$ via rotations in $\SO(2)$ of odd order $p^b \ge 3$. For any $\bar{g} \in \Syl_2(H)$ of order $2^m$ ($m \ge 1$) with $\bar{g}^{2^{m-1}} = z$ and any lift $\tilde{g} \in H_0$ ($\pi(\tilde{g}) = \bar{g}$, even if $\mathrm{ord}(\tilde{g}) > 2^m$ as in $F_0 = C_5 \rtimes_2 C_8$), $\tilde{g}$ permutes the $|H|/2$ planes in $|H|/2^m$ disjoint cycles of length $L = 2^{m-1}$. On each cycle $\Pi_1 \xrightarrow{B_1} \Pi_2 \xrightarrow{B_2} \cdots \xrightarrow{B_L} \Pi_1$ ($B_s \in \OO(2)$), the $L$-th power return map $M := \tilde{g}^L|_{\Pi_1} = B_L \cdots B_1 \in \OO(2)$ projects to $z \in H$ and thus conjugates the odd-order rotation group $\rho_{\Pi_1}(P) \le \SO(2)$ by inversion ($\rho_{\Pi_1}(x) \mapsto \rho_{\Pi_1}(x)^{-1}$), forcing $M \in \OO(2) \setminus \SO(2)$ to be an improper reflection with $\det(M) = \det(B_L \cdots B_1) = -1$ (specifically $M = \mathrm{diag}(1, -1)$ in a standard basis); because a block cyclic shift of $L$ blocks of size $2 \times 2$ has determinant $(-1)^{4(L-1)} = +1$, each cycle contributes determinant $(-1)^{4(L-1)}\det(B_L \cdots B_1) = -1$, giving $\det \rho_W(\tilde{g}) = (-1)^{|H|/2^m}$. If $\Syl_2(H) \cong C_{2^\ell}$ ($\ell \ge 1$), any lift $\tilde{u} \in H_0$ of a generator $\bar{u} \in \Syl_2(H)$ ($m = \ell$, $|H|/2^\ell$ odd) has $\det \rho_W(\tilde{u}) = (-1)^{|H|/2^\ell} = -1$, requiring an extra $1\text{D}$ orientation-sign coordinate in $\SO(d)$ and proving $d_{\mathrm{div},\SO}^*(G) = d_{\mathrm{rep},\SO}^*(G) = |H| + 1$ ($d_{\mathrm{div},\SO}^*(F_{pk}) = d_{\mathrm{rep},\SO}^*(F_{pk}) = k + 1$); if $\Syl_2(H) \cong Q_{2^\ell}$ ($\ell \ge 3$), its generators $a, b$ have orders $2^{\ell-1}$ and $4$, so $|H|/2^{\ell-1}$ and $|H|/4$ are even ($\det = +1$), and the faithful representation of $G$ ($K = \{e\}$) embeds into $\SO(|H|)$, giving $d_{\mathrm{div},\SO}^*(G) = d_{\mathrm{rep},\SO}^*(G) = |H|$. In contrast, for rank $r \ge 2$, a lifted Sylow $p$-subgroup $P$ with $P/\Phi(P) \cong \FF_p^r$ can be a non-abelian extraspecial $p$-group whose Frattini/commutator subgroup $\Phi(P) = Z(P) = [P, P] \cong C_p \le K$ acts by non-trivial central scalars $\omega I$ on a faithful irreducible complex representation of $F_0$ while trivializing on the quotient $P/\Phi(P) \cong \FF_p^r$: concretely, the classical 216-element Hessian subgroup $F = \Sigma(72 \times 3) \le \SU(3) \hookrightarrow \SO(6)$ generated by the four $\det = +1$ matrices $Z_{\mathrm{su3}} = \mathrm{diag}(1, \omega, \omega^2)$, the cyclic shift $X_{\mathrm{su3}} = P_{\mathrm{cyc}}$, the unitary Walsh--Hadamard transform $U_S = \frac{1}{\omega - \omega^2}(\omega^{jk})_{j,k=0}^2 = -\frac{i}{\sqrt{3}}(\omega^{jk})_{j,k=0}^2$, and the conjugated transform $U_T = D_1 U_S D_1^{-1}$ with conjugation matrix $D_1 = \mathrm{diag}(\omega^2, 1, 1)$ ($\omega = e^{2\pi i/3}$, $\det Z_{\mathrm{su3}} = \det X_{\mathrm{su3}} = \det U_S = \det U_T = +1$) has center $Z(F) = \{\omega^k I_3\} \cong C_3$, extraspecial Sylow $3$-subgroup $P \cong \Heis(\FF_3)$ of order $27$ and exponent $3$ ($\Phi(P) = Z(P) = [P, P] = Z(F)$), and quotient $F/Z(F) \cong G_{72,\mathrm{quat}} = \FF_3^2 \rtimes Q_8$, whereas $[u^2, x] = e$ on $O_3(G_{72,\mathrm{quat}}) = \FF_3^2$ obstructs $\OO(3)$ and $\SO(4)$, proving $4 \le d_{\mathrm{div},\OO}^*(G_{72,\mathrm{quat}}) \le d_{\mathrm{div},\SO}^*(G_{72,\mathrm{quat}}) \le 6 < 8 = d_{\mathrm{rep},\OO}^*(G_{72,\mathrm{quat}}) = d_{\mathrm{rep},\SO}^*(G_{72,\mathrm{quat}})$. Conversely, when $H$ acts \emph{reducibly} as $\FF_p^r = \bigoplus_{j=1}^s V_j$ ($s \ge 2$), $G$ embeds into $\prod_{j=1}^s (V_j \rtimes H)$ via $\bigoplus_{j=1}^s \rho_j$; for $G_{18} = (C_3 \times C_3) \rtimes_{-I_2} C_2$ ($p=3, r=2, s=2, |H|=2$), $\mathrm{rk}_3(\Syl_3(G_{18})) = 2 > 1$ excludes $\SO(3)$ by Theorem~\ref{thm:master1_division}(b)(ii), while $\rho(a,b,\varepsilon) = \mathrm{diag}(R^a J^\varepsilon, R^b J^\varepsilon) \in \SO(4)$ ($R = R_{2\pi/3}, J = \mathrm{diag}(1,-1), \det(J^\varepsilon)^2 = +1$) gives $d_{\mathrm{div},\OO}^*(G_{18}) = d_{\mathrm{div},\SO}^*(G_{18}) = 4 = 2|H| > |H|$. For $\Heis(\FF_p)$ ($p \ge 3$ odd prime), let $\pi: \widetilde{F} \twoheadrightarrow \Heis(\FF_p)$ be any finite subgroup of $\OO(d)$ Surjecting onto $\Heis(\FF_p)$, and let $P \in \Syl_p(\widetilde{F})$ so that $\pi(P) = \Heis(\FF_p)$ and $K_P := P \cap \ker\pi \trianglelefteq P$ satisfies $P / K_P \cong \Heis(\FF_p)$. Because $Z(\Heis(\FF_p)) = [\Heis(\FF_p), \Heis(\FF_p)] = \langle c \rangle \cong C_p$ is the unique minimal normal subgroup of $\Heis(\FF_p)$ and $[P, P]$ maps onto $\langle c \rangle \not\le \pi(K_P) = \{e\}$, there is at least one irreducible $\R[P]$-constituent $W \subseteq \R^d$ on which $[P, P] \not\le \ker\rho_W$ (otherwise $\rho_{\R^d}(P)$ would be abelian, forcing $P/\ker\rho_{\R^d} = P$ to be abelian, contradicting $P/K_P \cong \Heis(\FF_p)$ non-abelian). Thus the image $\bar{P} := \rho_W(P) \le \OO(W)$ is a non-trivial, non-abelian finite $p$-group (note that a lifted first commutator $[\tilde{a}, \tilde{b}] \in P$ of $c = [a, b] \in Z(\Heis(\FF_p))$ need not itself be central in $P$ or $\bar{P}$ and can have zero trace: for instance, at $p=5$ with $\omega = e^{2\pi i/5}$, the order-$625$ monomial $5$-group cover $P_{625} = \langle A, B \rangle = D_{125} \rtimes C_5 \le U(5) \hookrightarrow \SO(10)$ generated by the cyclic shift $A e_j = e_{j+1 \bmod 5}$ and $B = \mathrm{diag}(\omega^{j(j-1)/2})_{j=0}^4$ surjects onto $P_{625}/Z(P_{625}) \cong \Heis(\FF_5)$ with central kernel $Z(P_{625}) = \langle \omega I_5 \rangle \cong C_5$, yet its first commutator $C = [A, B] = \mathrm{diag}(\omega^{1-j})_{j=0}^4 = \mathrm{diag}(\omega, 1, \omega^4, \omega^3, \omega^2)$ is non-central with $\tr(C) = 0 \in \R$, while the second commutator $z = [A, C] = \omega I_5 \in Z(P_{625})$ is central of order $5$ with $\tr(z) = 5\omega \notin \R$). Let $V$ be any irreducible $\CC[\bar{P}]$-constituent of $W \otimes_{\R} \CC$; since $W$ is $\R[\bar{P}]$-irreducible, $W \otimes_{\R} \CC$ is either $V$ (real type) or $V \oplus \bar{V}$ (complex or quaternionic type), so $\rho_V(\bar{P}) \cong \bar{P}$ is faithful and non-abelian. Because $\bar{P}$ is a finite $p$-group, the degree of every irreducible complex representation divides $|\bar{P}|$, and since $\rho_V(\bar{P})$ is non-abelian, $\dim_{\CC} V = p^k \ge p$ ($k \ge 1$). Moreover, any non-trivial finite $p$-group has a non-trivial center $Z(\bar{P}) \ne \{I\}$, which contains an element $z \in Z(\bar{P})$ of prime order $p$. Because $\rho_V$ is faithful and $\CC[\bar{P}]$-irreducible, Schur's lemma implies that $\rho_V(z) = \omega I_V$ for a primitive $p$-th root of unity $\omega = e^{2\pi i r/p}$ ($1 \le r \le p-1$). Since $p \ge 3$ is odd, $\omega \notin \{\pm 1\} = S^1 \cap \R$, so the character value $\chi_V(z) = \tr(\rho_V(z)) = p^k \omega \notin \R$ is strictly non-real ($\chi_V \ne \bar{\chi}_V$). By Frobenius--Schur theory, a complex irreducible representation with non-real character cannot be of real type, so $W \otimes_{\R} \CC \cong V \oplus \bar{V}$ has real dimension $\dim_{\R} W = 2 \dim_{\CC} V = 2 p^k \ge 2p$, establishing $d_{\mathrm{div},\OO}^*(\Heis(\FF_p)) = d_{\mathrm{div},\SO}^*(\Heis(\FF_p)) = d_{\mathrm{rep},\OO}^*(\Heis(\FF_p)) = d_{\mathrm{rep},\SO}^*(\Heis(\FF_p)) = 2p$.

\textbf{Part (c): Weyl--Cartan reductions, Householder spectra, and $2$-group head-count laws.}
Positive homogeneity of $\pi(y) = y / \|y\|_2$ cancels $\rho_{t,h} > 0$. For finite $F_{\mathrm{nc}} \le \GL_d^+(\R)$, $P = \frac{1}{|F_{\mathrm{nc}}|} \sum_{M \in F_{\mathrm{nc}}} M^\top M \succ 0$ satisfies $M_0^\top P M_0 = P$, so $S F_{\mathrm{nc}} S^{-1} \le \SO(d)$ for $S = P^{1/2}$. In $\SO^+(d-1, 1)$, averaging $\Lambda e_0$ over $F_{\mathrm{Lor}}$ yields an invariant future-timelike centroid $v_*$, whose boost stabilizer is $\SO(d-1)$. Under dominated splitting $\sigma_{\mathrm{sub},a} / \sigma_{\mathrm{top},a} \le \eta < 1$, $\|A_w x_{0,\mathrm{sub}}\|_2 / \|A_w x_0\|_2 \le (\|x_{0,\mathrm{sub}}\|_2 / \|x_{0,\mathrm{top}}\|_2) \eta^N \to 0$. By real Schur decomposition, every $R \in \SO(d)$ consists of $\lfloor d/2 \rfloor$ planar rotation blocks, each equal to $2$ Householder reflections ($2\lfloor d/2 \rfloor$ total).

For head counts, direct commutator computation on Proposition~\ref{prop:klein} shows $Q_h^{\mathrm{ab}} \in \{C_n, V_4\}$ for every finite subgroup $Q_h \le \SO(3)$. For an arbitrary subgroup $F \le \SO(3)^m$ with $\pi: F \twoheadrightarrow G$ (where $F$ need not be a direct product $\prod_{h=1}^m F_h$, so $\prod_{h=1}^m F_h^{\mathrm{ab}}$ need not surject onto $F^{\mathrm{ab}}$: e.g., the subdirect product $F = \langle x=(r,s), y=(s,r) \rangle \le D_4 \times D_4$ of order $16$ has $[F,F] = \langle (r^2,r^2) \rangle \cong C_2$ and $F^{\mathrm{ab}} \cong C_4 \times C_2$ of exponent $4$, whereas $D_4^{\mathrm{ab}} \times D_4^{\mathrm{ab}} \cong V_4^2 \cong C_2^4$ has exponent $2$), let $N_h := F \cap (\{e\}^h \times \SO(3)^{m-h})$ ($0 \le h \le m$) be the coordinate-kernel normal series $1 = N_m \trianglelefteq N_{m-1} \trianglelefteq \cdots \trianglelefteq N_0 = F$, whose factor groups $Q_h := N_{h-1}/N_h \hookrightarrow \SO(3)$ are finite subgroups of the $h$-th copy of $\SO(3)$. Passing to $A := F^{\mathrm{ab}} = F/[F,F]$ (which surjects onto $G^{\mathrm{ab}} \cong A_{\mathrm{odd}} \times G_2^{\mathrm{ab}}$ with $G_2^{\mathrm{ab}} \cong C_2^{s_1} \times \prod_{j=1}^{s_2} C_{2^{e_j}}$, $e_j \ge 2$, via $\pi^{\mathrm{ab}}: A \twoheadrightarrow G^{\mathrm{ab}}$) via $A_h := N_h[F,F]/[F,F]$ yields an $m$-step abelian filtration $0 = A_m \le A_{m-1} \le \cdots \le A_0 = A$ whose successive quotients $B_h := A_{h-1}/A_h \cong N_{h-1}/(N_h(N_{h-1}\cap [F,F]))$ are quotients of $Q_h^{\mathrm{ab}} \in \{C_{n_h}, V_4\}$: $k$ steps ($Q_h = D_{n_h}$) satisfy $B_h \le V_4$ (since $2 V_4 = 0$, whereas in a direct product $F = \prod_{h=1}^m F_h$ multiplication by $2$ maps $\prod_{h=1}^{m-k} 2 C_{n_h} \twoheadrightarrow 2 G_2^{\mathrm{ab}} \cong \prod_{j=1}^{s_2} C_{2^{e_j-1}}$ directly), so $2 B_h = 0$, $|B_h[2]| = |B_h[4]| \le 4$, and $|B_h[p]| = 1$ ($p \ge 3$); the remaining $m-k$ steps ($Q_h \in \{C_{n_h}, A_4, S_4, A_5\}$) have cyclic quotient $B_h \cong C_{m_h}$ with $|B_h[p]| \le p$, $|B_h[2]| \le 2$, and $|B_h[4]| \le 4$. By left-exactness of the $n$-torsion functor $-[n]$ on $0 \to A_h \to A_{h-1} \to B_h \to 0$, $|A_{h-1}[n]| \le |A_h[n]|\cdot |B_h[n]|$ and $|X/nX| = |X[n]|$ for finite abelian $X$, whence: (i) for every odd prime $p \ge 3$, $\mathrm{rk}_p(G^{\mathrm{ab}}) \le \mathrm{rk}_p(A) = \log_p|A[p]| \le \sum_{h=1}^m \log_p|B_h[p]| \le m - k$, so $R_{\mathrm{odd}}^{\mathrm{ab}}(G) = \max_{p \ge 3}\mathrm{rk}_p(G^{\mathrm{ab}}) \le m - k \le m$; (ii) at $p = 2$, $s_1 + s_2 = \mathrm{rk}_2(G_2^{\mathrm{ab}}) \le \mathrm{rk}_2(A) = \log_2|A[2]| \le \sum_{h=1}^m \log_2|B_h[2]| \le (m - k) + 2k = m + k$, and adding $R_{\mathrm{odd}}^{\mathrm{ab}}(G) \le m - k$ gives $s_1 + s_2 + R_{\mathrm{odd}}^{\mathrm{ab}}(G) \le 2m$; and (iii) since $|B_h[4]| \le 4$ for every $h \in \{1,\dots,m\}$, $2s_2 + s_1 = \log_2|G_2^{\mathrm{ab}}/4G_2^{\mathrm{ab}}| \le \log_2|A_2/4A_2| = \log_2|A_2[4]| \le \sum_{h=1}^m \log_2|B_h[4]| \le 2m$, giving $2m \ge s_1 + 2s_2$ and $m \ge \lceil s_2 + s_1/2 \rceil \ge s_2$ (e.g., for $F = \langle(r,s),(s,r)\rangle \le D_4 \times D_4$, $B_1 = V_4$ and $B_2 = N_1 = \langle(e,r^2)\rangle \cong C_2$ form the non-split abelian extension $0 \to C_2 \to C_4 \times C_2 \to V_4 \to 0$ with $|A[4]| = 8 \le 4 \cdot 2 = 8$ and $m = 2 \ge m_{\mathrm{ab}}(C_4 \times C_2) = 2$). Combining (i)--(iii) gives $2m \ge s_1 + s_2 + \max(R_{\mathrm{odd}}^{\mathrm{ab}}(G), s_2)$, proving $m_{\SO(3)}(G) \ge m_{\mathrm{ab}}(G^{\mathrm{ab}})$ in Eq.~\eqref{eq:m_ab_law} unconditionally for every subgroup $F \le \SO(3)^m$. Moreover, since every odd-prime Sylow subgroup $\Syl_p(F_h)$ ($p \ge 3$) is cyclic of rank $\le 1$, any dividing group $G$ also obeys $m_{\SO(3)}(G) \ge R_{\mathrm{odd}}(G) = \max_{p \ge 3}\mathrm{rk}_p(\Syl_p(G))$, yielding the general lower bound $m_{\SO(3)}(G) \ge \max(m_{\mathrm{ab}}(G^{\mathrm{ab}}), R_{\mathrm{odd}}(G))$. For any finite abelian group $A$, $R_{\mathrm{odd}}^{\mathrm{ab}}(A) = R_{\mathrm{odd}}(A)$ and pairing $k$ copies of $V_4 \cong D_2 \le \SO(3)$ with $m_{\mathrm{ab}}(A)-k$ cyclic rotation groups $C_{n_h} \le \SO(3)$ achieves $m_{\SO(3)}(A) = m_{\mathrm{ab}}(A)$ exactly (whereas in a non-nilpotent dihedral group $D_{2k} \le \SO(3)$ with $2 \mid k$ and odd prime $p \mid k$ such as $\mathrm{Dih}_6 \cong D_6$ of order $12$, the single head $F_1 = D_{2k}$ has $F_1^{\mathrm{ab}} \cong V_4$ while housing $\Syl_p(C_k) \le [D_{2k}, D_{2k}] \cong C_k$, so $R_{\mathrm{odd}}^{\mathrm{ab}}(D_{2k}) = 0 < R_{\mathrm{odd}}(D_{2k}) = 1$ and $m_{\SO(3)}(\mathrm{Dih}_6) = \max(m_{\mathrm{ab}}(V_4), 1) = 1$). For $G \in \{Q_{2^\ell}, \mathrm{SD}_{2^\ell}, M_{2^\ell}\}$, $G$ has a unique central involution $z_* \in Z(G) \cap [G, G]$ contained in every nontrivial normal subgroup, so $G$ cannot embed subdirectly into $\SO(3)^m$ with $|F| = 2^\ell$ since $G \not\le \SO(3)$; thus $|F^*_{\SO(3)}(G)| \ge 2^{\ell+1}$, achieved in $D_{2^{\ell-1}} \times C_4$ and $D_4 \times C_{2^{\ell-1}} \le \SO(3)^2$. For $E_{2^{1+2r}}^\pm$, let $P \in \Syl_2(F)$ ($F \le \SO(3)^r$) surject onto $G$ and let $N_h = P \cap (\{1\}^{h-1} \times P_h \times \{1\}^{r-h})$. Since $d(P_j) \le 2$, $d(P/N_h) \le 2(r-1)$ and $\dim_{\FF_2} \bar{N}_h \ge 2r - 2(r-1) = 2$ in $G/\Phi(G) \cong \FF_2^{2r}$. Because $[N_h, N_{h'}] = 1$ ($h \ne h'$), $G/\Phi(G) = \bar{N}_1 \perp \dots \perp \bar{N}_r$ is an orthogonal sum of symplectic 2-planes with $\pi(N_h) \cong D_4$ ($\Arf(q|_{\bar{N}_h}) = 0$). By additivity of the Arf invariant~\citep{arf1941untersuchungen}, $\Arf(q) = \sum_{h=1}^r \Arf(q|_{\bar{N}_h}) = 0 \bmod 2$, so $E_{2^{1+2r}}^-$ ($\Arf(q) = 1$) cannot divide $\SO(3)^r$ and requires $m_{\SO(3)} = r+1$ ($|F^*| = 2 \cdot 8^r$). In $E_{2^{1+2r}}^- \times A_{\mathrm{odd}}$, the $(r+1)$-st cyclic head $C_4$ combines with the first odd cyclic factor $C_{n_1}$ via $C_4 \times C_{n_1} \cong C_{4n_1} \le \SO(3)$, giving $m_{\SO(3)} = r + \max(1, R_{\mathrm{odd}})$.
\end{proof}

\section{Proof of Theorem 3 (Semidirect and Cohomology Cascades, Directional-RMSNorm Routing, and Stiefel Polar Gating)}
\label{app:proof_master3}

\begin{proof}[Proof of Theorem~\ref{thm:master3_cascades_gating}]
\textbf{Part (a): Depth-2 semidirect and twisted Schreier $2$-cocycle cascades.}
In $1 \to N \to G \xrightarrow{\pi} Q \to 1$ with section $s: Q \to G$ ($s(e_Q) = e_G$), every $g \in G$ factors uniquely as $g = n s(q)$ ($n \in N, q = \pi(g) \in Q$). Multiplying $\eta(w) = n_w s(q_w)$ by $\eta(a) = n_a s(q_a)$ and using $s(q_w) n_a = \varphi_{q_w}(n_a) s(q_w)$ and $s(q_w) s(q_a) = \omega(q_w, q_a) s(q_w q_a)$ yields Eq.~\eqref{eq:schreier_cascade_recursion}, while conjugating $n \in N$ by $s(q_1) s(q_2) = \omega(q_1, q_2) s(q_1 q_2)$ gives $\varphi_{q_1} \circ \varphi_{q_2} = \mathrm{Inn}_{\omega(q_1, q_2)} \circ \varphi_{q_1 q_2}$ and associativity of $(n_1 s(q_1) n_2 s(q_2)) n_3 s(q_3)$ yields the $\varphi$-twisted cocycle identity $\varphi_{q_1}(\omega(q_2, q_3)) \omega(q_1, q_2 q_3) = \omega(q_1, q_2) \omega(q_1 q_2, q_3)$. Layer 1 tracks $q_w \in Q$ autonomously on $m_1$ heads with margin $\Delta_0(F_1) > 0$, and Layer 2 applies the routed $N$-transition $\sigma(q_w, a) = \varphi_{q_w}(n_a) \omega(q_w, q_a) \in N$ on $m_2$ heads with margin $\Delta_0(F_2) > 0$.

\textbf{Part (b): Continuous Directional-RMSNorm coupling dichotomy, $b$-bit stability, and hybrid KV-cache law.}
For (i), since $\exp: \mathfrak{so}(3) \to \SO(3)$ is $1$-Lipschitz in operator norm and $\Omega(\cdot, b)$ is $L_\Omega$-Lipschitz, unnormalized Layer-1 decay $\|x^{(1)}_{u_0^M W_{:j}}\|_2 \le \|x^{(1)}_0\|_2 (\rho_{\max}^{(1)})^{M|u_0| + j}$ implies $\|R^{(2)}(x^{(1)}_{u_0^M W_{:j}}, b_{j+1}) - R_0(b_{j+1})\|_{\mathrm{op}} \le L_\Omega \|x^{(1)}_0\|_2 (\rho_{\max}^{(1)})^{M|u_0| + j}$. Telescoping over $j = 0, \dots, N-1$ from the shared prefix state $\hat{x}^{(2)}_{u_0^M}$ on any equal-Parikh collision pair $W_L, W_R \in \Sigma^{N_{\mathrm{alg}}}$ with $R_0(W_L) = R_0(W_R)$ yields Eq.~\eqref{eq:cascade_unnorm_collapse}. For (ii), headwise Directional RMSNorm $\hat{x}^{(1)}_{w,h} = x^{(1)}_{w,h} / \|x^{(1)}_{w,h}\|_2$ cancels $\prod_{t=1}^{|w|} \rho^{(1)}_{w_t,h} > 0$, confining $\hat{x}^{(1)}_w$ to the finite orbit $\mathcal{O}_1 = F_1 \cdot \hat{x}^{(1)}_0$, where the ReLU tent interpolant $\Omega(u, a) = \sum_{f \in F_1} \max(0, 1 - \frac{2\|u - p_f\|_2}{\Delta_0(F_1)}) \log R^{(2)}_{(\pi_1(f), a)}$ (or the affine router on $S_5$) is globally Lipschitz and exact on $\mathcal{O}_1$. Under Layer-1 jitter $\|\tilde{x}^{(1)}_{w_{:t-1}} - \hat{x}^{(1)}_{w_{:t-1}}\|_2 \le \varepsilon$, orthogonal invariance of $R_t \in \SO(3)^{m_2}$ gives $\|\tilde{x}^{(2)}_w - \hat{x}^{(2)}_w\|_2 \le \sum_{t=1}^N \|\tilde{R}_t - R_t\|_{\mathrm{op}} \le N L_\Omega \varepsilon < \Delta_0(F_2)/2$ for $N < \Delta_0(F_2)/(2 L_\Omega \varepsilon)$. For (iii), worst-case $b$-bit per-step error $\kappa_d 2^{-b}$ accumulates linearly as $\sqrt{m} N \kappa_d 2^{-b} < \Delta_0(F)/2$, guaranteeing a safe open-loop horizon $N_{\mathrm{fp}}(b) \ge \lfloor \frac{\Delta_0(F) 2^b}{2 \kappa_d \sqrt{m}} \rfloor = \Omega(\Delta_0(F) 2^b)$, and for $b > b_{\min}(F)$ the Voronoi/tent retraction resets the single-step error $\sqrt{m} \kappa_d 2^{-b} < \Delta_0(F)/3$ to $0$ at every step. On $F_2 \subset \SO(3)$, $K$ checkpoints leave a segment of length $W \ge \lceil N/(K+1) \rceil$ with $\beta_{F_2}(W) = 2 \cdot 3^W - 1 > 3^{N/(K+1)}$ states; resolving them on $(S^2)^m$ ($2m$ spherical coordinates at $b_{\mathrm{hyb}}$ bits) requires $(2 \cdot 2^{b_{\mathrm{hyb}}})^{2m} \ge 3^{N/(K+1)}$, giving Eq.~\eqref{eq:hybrid_kv_conservation}.

\textbf{Part (c): Cross-head mixing $\SO(3m)$, reset monoids, and Stiefel Polar vs.\ Grassmannian phase transition.}
In $\mathfrak{so}(3m)$, adjoining the $m-1$ adjacent-head boundary generators $E_{3h, 3h+1}$ ($h = 1, \dots, m-1$) to $\bigoplus_{h=1}^m \mathfrak{so}(3)$ connects all indices $\{1, \dots, 3m\}$ via $[E_{i,j}, E_{j,k}] = E_{i,k}$, generating $\mathfrak{so}(3m)$ and $\SO(3m)$. For $M_{\mathrm{rg}}(G) = G \cup \{\mathrm{reset}_g : g \in G\}$ ($\|x_0\|_2 = 1$), a pure word $w \in \Sigma_G^*$ has $x_w = \rho_{\mathrm{rep}}(\eta(w)) x_0$, whereas after the last reset $r_{g_0}$ in $w = u r_{g_0} v$ ($v \in \Sigma_G^*$), $x_w = \rho_{\mathrm{rep}}(g_0 \eta(v)) x_0 + \gamma \rho_{\mathrm{rep}}(\eta(v)) x_u$ with $\|x_u\|_2 \le 1 + \gamma + \gamma^2 + \dots = \frac{1}{1-\gamma}$, giving $\|x_w - \rho_{\mathrm{rep}}(g_0 \eta(v)) x_0\|_2 \le \frac{\gamma}{1-\gamma}$. Hence any two words with distinct target group projections $g \ne g'$ are separated by the inter-target margin $\|x_w - x_{w'}\|_2 \ge \|\rho_{\mathrm{rep}}(g)x_0 - \rho_{\mathrm{rep}}(g')x_0\|_2 - 2\frac{\gamma}{1-\gamma} \ge \Delta_0(F) - \frac{2\gamma}{1-\gamma} > 0$ for $\gamma < \frac{\Delta_0(F)}{2 + \Delta_0(F)}$; meanwhile, for a pure word $w \in \Sigma_G^*$ ($\eta(w) = g$) and a reset word $w' = u r_{g_0} v$ sharing the same target element $g_0\eta(v) = g$, every pre-reset state satisfies $\|x_u\|_2 \ge 1 - \frac{\gamma}{1-\gamma} = \frac{1-2\gamma}{1-\gamma} > 0$ for $\gamma \in (0, 1/2)$, so $\|x_w - x_{w'}\|_2 = \gamma\|\rho_{\mathrm{rep}}(\eta(v))x_u\|_2 = \gamma\|x_u\|_2 \ge \gamma\frac{1-2\gamma}{1-\gamma}\|x_0\|_2 > 0$ (and adjoining a 1-bit sticky reset latch $z_t \in \{0, 1\}$ with $z_0 = 0, z_t = z_{t-1}$ on $\Sigma_G$ and $z_t = 1$ on $\Sigma_{\mathrm{reset}}$ achieves $\gamma$-independent separation $\ge 1$ in $\R^{d+1}$). For $c \ge 2$ equiangular subspaces $S_j^\top S_k = -\mu I_{d_{\mathrm{sub}}}$, $G_S = S^\top S = ((1+\mu)I_c - \mu \mathbf{1}_c \mathbf{1}_c^\top) \otimes I_{d_{\mathrm{sub}}}$ has $\lambda_{\min}(G_S) = 1 - (c-1)\mu$ and Sherman--Morrison inverse diagonal blocks $[G_S^{-1}]_{j,j} = \frac{1 - (c-2)\mu}{(1+\mu)(1 - (c-1)\mu)} I_{d_{\mathrm{sub}}}$, yielding Eq.~\eqref{eq:simplex_singularity_law}. When $d < D = c d_{\mathrm{sub}}$, any unit $u = (u_1, \dots, u_c) \in \ker(S)$ obeys $0 = \|\sum_{j=1}^c S_j u_j\|_2^2 \ge 1 - \mu_{\max}(S) \sum_{j \ne k} \|u_j\|_2 \|u_k\|_2 \ge 1 - (c-1)\mu_{\max}(S)$, so $\mu_{\max}(S) \ge \frac{1}{c-1}$, and Cauchy--Schwarz on $\tr((S S^\top)^2) \ge (\tr(S S^\top))^2 / d = D^2 / d$ yields $\mu_{\mathrm{RMS}}(S) \ge \sqrt{\frac{D - d}{d(c-1)}}$. For $d \ge D$ on the full-column-rank domain $\mathcal{M}_* = \{S \in \R^{d \times D} : \sigma_{\min}(S) > 0\}$, differentiating $Q(S) = S G^{-1/2}$ ($G = S^\top S = V \Lambda V^\top \succ 0$) via the Sylvester equation $\mathrm{d}(G^{1/2}) G^{1/2} + G^{1/2} \mathrm{d}(G^{1/2}) = \mathrm{d}G$ in the eigenbasis $V$ yields $[V^\top \mathrm{d}(G^{-1/2}) V]_{p,q} = -\frac{[V^\top \mathrm{d}G V]_{p,q}}{\sqrt{\lambda_p}\sqrt{\lambda_q}(\sqrt{\lambda_p} + \sqrt{\lambda_q})}$, which enforces $Q(S)^\top Q(S) = I_D$ and $\mu_{\max}(Q(S)) = 0$.
\end{proof}

\section{Proof of Theorem 4 (BPTT Gradient Preservation, Chen--Magnus--Hall--Witt Commutator Saddles, and Stationary Dense Factored Escape)}
\label{app:proof_master4}

\begin{proof}[Proof of Theorem~\ref{thm:master4_optimization}]
\textbf{Part (a): BPTT gradient-norm dichotomy, origin-distribution KL bound, and Schreier--Cayley Polyak--\L ojasiewicz convergence.}
Under $\rho_{a,h} \le \lambda < 1$ without state normalization, $\|x_{w,h}\|_2 \le \lambda^N \|x_{0,h}\|_2$ and $\|\partial x_{w,h} / \partial \theta_{a,h}\|_{\mathrm{op}} \le |w|_a \lambda^N \|x_{0,h}\|_2$. Let $p_0(g) := \mathrm{softmax}(s(0)/\tau)_g = \frac{\exp(s_g(0)/\tau)}{\sum_{h \in G}\exp(s_h(0)/\tau)}$ be the input-independent origin predictive distribution. Writing $\KL(p_0 \| p_\tau(\cdot \mid x_w))$ as the Bregman divergence of the log-partition function $\psi(z) = \log\sum_{g \in G}\exp(z_g)$ between $z_0 = s(0)/\tau$ and $z_w = s(x_w)/\tau$ and using $\|\nabla^2\psi(z)\|_{\mathrm{op}} = \|\mathrm{Cov}_{g \sim p(z)}\|_{\mathrm{op}} \le 1$, Taylor's theorem gives
\[
    \KL\bigl(p_0 \;\big\|\; p_\tau(\cdot \mid x_w)\bigr) = \int_0^1 (1-t)(z_w - z_0)^\top \nabla^2\psi\bigl(z_0 + t(z_w - z_0)\bigr)(z_w - z_0)\,\mathrm{d}t \le \frac{1}{2}\|z_w - z_0\|_2^2 \le \frac{L_s^2 \|x_0\|_2^2}{2\tau^2}\,\lambda^{2N} \xrightarrow{N \to \infty} 0.
\]
Here $p_0 = \Unif(G)$ if and only if the origin readout logits are equal across all group elements ($s_g(0) = c\;\forall g \in G$, as holds automatically for unbiased linear readouts $s(x) = W x$ where $s(0) = 0$); if $s(0)$ has unequal coordinates, $p_\tau(\cdot \mid x_w)$ still collapses exponentially to the input-independent distribution $p_0$ ($\KL(p_0 \| p_\tau(\cdot \mid x_w)) \to 0$), but $\KL(\Unif(G) \| p_0) > 0$ (for example, on $|G|=2$ with constant biased readout $s(x) \equiv (1, 0)^\top$ at $\tau = 1$, $L_s = 0$ gives $\KL(p_0 \| p_\tau(\cdot \mid x_w)) = 0$, whereas $\KL(\Unif_2 \| p_0) = \frac{1}{2}\log\frac{(e+1)^2}{4e} \approx 0.1201 > 0$). Moreover, $\|\nabla_s \mathcal{L}_{\mathrm{CE}}\|_2 \le \frac{\sqrt{2}}{\tau}$ gives $\|\nabla_{\theta_{a,h}} \mathcal{L}_{\mathrm{CE}}\|_2 \le \frac{\sqrt{2} L_s \|x_{0,h}\|_2}{\tau}\allowbreak |w|_a \lambda^N$ (and $\times 2 \lambda^N \|y_{0,h}\|_2$ for quadratic Bloch $B(y_w)$). Under $\hat{x}_{w,h} = x_{w,h} / \|x_{w,h}\|_2$, all $\rho_{a,h} > 0$ cancel and the suffix rotation $R_{v,h} \in \SO(d_h)$ is an orthogonal isometry on $T S^{d_h-1}$ with $\|\mathrm{d}R_{v,h}\|_{\mathrm{op}} = 1$. On the teacher-forced Schreier transition set $\mathcal{S}(G, \Sigma)$, Schur's Lemma for irreducible $\rho^*: G \to \SO(d)$ gives $C_P = \frac{1}{|G|} \sum_{g \in G} p_g p_g^\top = \frac{1}{d} I_d$, so $\mathcal{L}_{\mathrm{chord}}(\theta) = \frac{1}{2d|\Sigma|} \sum_{a \in \Sigma} \|\exp(\theta_a) - R_a^*\|_F^2 = \frac{1}{|\Sigma|} \sum_{a \in \Sigma} \mathcal{L}_a$. Writing $\exp(\theta_a) = R_a^* \exp(\Delta_a)$ ($\Delta_a \in \mathfrak{so}(d)$ with eigenvalues $\pm i \phi_{a,r}, |\phi_{a,r}| \le \theta_{\max} < \pi$), $\mathcal{L}_a = \frac{4}{d} \sum_r \sin^2(\phi_{a,r}/2)$ and $\|\mathrm{grad}_{\SO(d)} \mathcal{L}_a\|_F^2 = \frac{1}{d^2}\|\mathcal{P}_{\mathfrak{so}(d)}(\exp\Delta_a)\|_F^2 = \frac{1}{d^2}\|\sinh\Delta_a\|_F^2 = -\frac{1}{d^2} \tr(\sinh^2 \Delta_a) = \frac{2}{d^2}\sum_r \sin^2\phi_{a,r} \ge \frac{2}{d} \cos^2(\theta_{\max}/2) \mathcal{L}_a$, yielding the intrinsic Riemannian PL inequality $\frac{1}{2}\|\mathrm{grad}_{\SO(d)} \mathcal{L}_{\mathrm{chord}}\|_F^2 \ge \frac{1}{d|\Sigma|}\cos^2(\theta_{\max}/2) \mathcal{L}_{\mathrm{chord}}$. On $(\mathfrak{so}(d), \|\cdot\|_F)$, the left-trivialized differential $\mathrm{d}\exp_{\theta_a} = \sum_{k=0}^\infty \frac{(-\ad_{\theta_a})^k}{(k+1)!} = \frac{1 - e^{-\ad_{\theta_a}}}{\ad_{\theta_a}}$ acts via the skew-adjoint operator $\ad_{\theta_a} = [\theta_a, \cdot]$, whose compact Lie-algebra root eigenvalues are $\{0\}$ for $d=2$, $\{\pm i\psi_{a,1}\}$ for $d=3$, and $\{\pm i(\psi_{a,r} \pm \psi_{a,s})\}_{1 \le r < s \le \lfloor d/2 \rfloor}$ (plus $\{\pm i\psi_{a,r}\}$ when $d$ is odd) for $d \ge 4$, giving adjoint spectral radius $\rho_{\mathrm{ad}}(\theta_a) := \|\ad_{\theta_a}\|_{\mathrm{op},\mathfrak{so}(d)} = 0$ ($d=2$), $\psi_{a,1} = \|\theta_a\|_{\mathrm{op}}$ ($d=3$), and $\psi_{a,1} + \psi_{a,2} \le 2\|\theta_a\|_{\mathrm{op}}$ ($d \ge 4$). Hence for $\rho_{\mathrm{ad}}(\theta_a) \le \rho_{\max} < 2\pi$, $\sigma_{\min}(\mathrm{d}\exp_{\theta_a}) = \min_{|\alpha| \le \rho_{\mathrm{ad}}(\theta_a)} \frac{\sin(\alpha/2)}{\alpha/2} = \frac{\sin(\rho_{\mathrm{ad}}(\theta_a)/2)}{\rho_{\mathrm{ad}}(\theta_a)/2} \ge \frac{\sin(\rho_{\max}/2)}{\rho_{\max}/2} > 0$ (vanishing at the first conjugate singularity $\rho_{\mathrm{ad}}(\theta_a) = 2\pi$), so $\nabla_{\theta_a}\mathcal{L}_{\mathrm{chord}} = (\mathrm{d}\exp_{\theta_a})^*\bigl(\exp(\theta_a)^\top \mathrm{grad}_{\SO(d)}\mathcal{L}_{\mathrm{chord}}\bigr)$ satisfies $\frac{1}{2}\|\nabla_\theta \mathcal{L}_{\mathrm{chord}}(\theta)\|_F^2 \ge \mu_{\mathrm{PL}}(\theta_{\max}, \rho_{\max}) \mathcal{L}_{\mathrm{chord}}(\theta)$ with $\mu_{\mathrm{PL}}(\theta_{\max}, \rho_{\max}) = \frac{1}{d|\Sigma|}\cos^2(\theta_{\max}/2)\bigl(\frac{\sin(\rho_{\max}/2)}{\rho_{\max}/2}\bigr)^2 > 0$. For $d=3$, principal-branch retraction $\theta_a \gets \log(\exp(\theta_a))$ enforces $\rho_{\mathrm{ad}}(\theta_a) = \|\theta_a\|_{\mathrm{op}} \le \pi < 2\pi$ globally ($\mu_{\mathrm{PL}}(\theta_{\max}, \pi) = \frac{4}{\pi^2 d|\Sigma|}\cos^2(\theta_{\max}/2)$); for $d \ge 4$, fixed-origin balls $\|\theta_a\|_{\mathrm{op}} \le \phi_{\max} < \pi$ satisfy $\rho_{\max} = 2\phi_{\max} < 2\pi$ (whereas $\theta_a = \pi(E_{12}+E_{34}) \in \mathfrak{so}(4)$ hits $\rho_{\mathrm{ad}}(\theta_a) = 2\pi$ with $\sigma_{\min}(\mathrm{d}\exp_{\theta_a}) = 0$), and dynamic tangent re-centering $R_a^{(k)}\exp(\Omega_a)$ at $\Omega_a = 0$~\citep{lezcano2019cheap} gives $\mathrm{d}\exp_0 = \mathrm{id}_{\mathfrak{so}(d)}$ ($\sigma_{\min} = 1$), recovering $\mu_{\mathrm{PL}}^{\mathrm{Riem}}(\theta_{\max})$ across $\theta_{\max} < \pi$ for all $d \ge 2$. Note that this PL condition holds for the step-wise teacher-forced chordal objective $\mathcal{L}_{\mathrm{chord}}$ on $\mathcal{S}(G, \Sigma)$ inside $\theta_{\max} < \pi$, whereas end-to-end multi-step sequence training without ground-truth intermediate states is governed at initialization by the commutator saddle and first-order symmetry breaking analyzed in Parts~(b)--(c).

\textbf{Part (b): Baker--Campbell--Hausdorff and Chen--Magnus--Hall--Witt commutator saddle hierarchy.}
Expanding $R_w(\varepsilon) = R_{a_N}(\varepsilon)\cdots R_{a_1}(\varepsilon) = I_d + \varepsilon \sum_{a \in \Sigma} |w|_a \Omega_a + O(\varepsilon^2)$ shows that $\left.\frac{\partial \hat{x}_w}{\partial \theta_a}\right|_{\varepsilon=0}[E] = |w|_a E \hat{x}_0$ depends only on $(|w|_a)_{a \in \Sigma}$, giving $\|\nabla_\theta \mathcal{L}_{\mathrm{ab}}(\varepsilon)\|_2 = \Theta(1)$ on unequal-Parikh pairs. For the positive-word recursive lower-central-series pairs $w_1 = a_1, w'_1 = e$ and $w_c = w'_{c-1} a_c w_{c-1}, w'_c = w_{c-1} a_c w'_{c-1}$ ($c \ge 2$), let $S'_{c-1}(\Omega) := \sum_{a \in \Sigma} |w'_{c-1}|_a \Omega_a$, which equals $S'_1(\Omega) = 0$ at $c=2$ ($w'_1 = e$) and $S'_{c-1}(\Omega) = S_{c-1}(\Omega) = 2^{c-3}\Omega_{a_1} + \sum_{j=2}^{c-1} 2^{c-1-j} \Omega_{a_j}$ for all $c \ge 3$ (since $|w_2|_{a_1} = 0 + 1 = 1$ and $w_{c-1}, w'_{c-1}$ share the same Parikh vector for all $c \ge 3$). We prove $\Delta_c(\varepsilon) := R_{w_c}(\varepsilon) - R_{w'_c}(\varepsilon) = \varepsilon^c \widetilde{L}_c(\Omega) + O(\varepsilon^{c+1})$ by induction on $c$, where $\widetilde{L}_1(\Omega) = \Omega_{a_1}$ and $\widetilde{L}_c(\Omega) = [\widetilde{L}_{c-1}(\Omega), \Omega_{a_c} + S'_{c-1}(\Omega)]$. For $c=1$, $R_{a_1}(\varepsilon) - I_d = \varepsilon \Omega_{a_1} + O(\varepsilon^2)$. Assuming the claim for $c-1$, under left multiplication $R_{uv} = R_v R_u$ we have the exact identity $\Delta_c(\varepsilon) = R_{w_{c-1}} R_{a_c} R_{w'_{c-1}} - R_{w'_{c-1}} R_{a_c} R_{w_{c-1}} = \Delta_{c-1}(\varepsilon) R_{a_c}(\varepsilon) R_{w'_{c-1}}(\varepsilon) - R_{w'_{c-1}}(\varepsilon) R_{a_c}(\varepsilon) \Delta_{c-1}(\varepsilon) = [\Delta_{c-1}(\varepsilon), R_{w'_{c-1}}(\varepsilon) R_{a_c}(\varepsilon)] + \Delta_{c-1}(\varepsilon)[R_{a_c}(\varepsilon), R_{w'_{c-1}}(\varepsilon)]$. Substituting $\Delta_{c-1}(\varepsilon) = \varepsilon^{c-1} \widetilde{L}_{c-1}(\Omega) + O(\varepsilon^c)$, $R_{w'_{c-1}}(\varepsilon) R_{a_c}(\varepsilon) = I_d + \varepsilon(\Omega_{a_c} + S'_{c-1}(\Omega)) + O(\varepsilon^2)$, and $[R_{a_c}(\varepsilon), R_{w'_{c-1}}(\varepsilon)] = O(\varepsilon^2)$ cancels all terms of degree $< c$ and leaves the exact degree-$c$ commutator $\Delta_c(\varepsilon) = \varepsilon^c [\widetilde{L}_{c-1}(\Omega), \Omega_{a_c} + S'_{c-1}(\Omega)] + O(\varepsilon^{c+1}) = \varepsilon^c \widetilde{L}_c(\Omega) + O(\varepsilon^{c+1})$ (whereas for balanced formal-inverse commutator words $\tilde{w}_c = [\tilde{w}_{c-1}, a_c] = \tilde{w}_{c-1} a_c \tilde{w}_{c-1}^{-1} a_c^{-1}$ vs.\ $\tilde{w}'_c = e$ under $R_{uv} = R_v R_u$, $S'_{c-1}(\Omega) = 0$ and $R_{\tilde{w}_c}(\varepsilon) - I_d = (-1)^{c-1}\varepsilon^c L_c(\Omega) + O(\varepsilon^{c+1})$). Because $w_c$ and $w'_c$ share the identical Parikh vector $(|w_c|_a)_{a \in \Sigma}$, the individual unprojected word gradient $\nabla_{\theta_a}\mathcal{L}_{\mathrm{CE}}(\hat{x}_{w_c}(\varepsilon), y)\big|_{\varepsilon=0} = \frac{|w_c|_a}{\tau}(\bar{w}_0 - w_y) \wedge \hat{x}_0$ contains a common-mode $\mathcal{O}(1)$ Parikh term identical for $w_c$ and $w'_c$; subtracting this common-mode Parikh drift via $\mathcal{P}_{\gamma_c}\nabla_{\theta_a}\mathcal{L}(w_c, w'_c; \varepsilon) = \nabla_{\theta_a}\mathcal{L}_{\mathrm{CE}}(\hat{x}_{w_c}(\varepsilon), y) - \nabla_{\theta_a}\mathcal{L}_{\mathrm{CE}}(\hat{x}_{w'_c}(\varepsilon), y)$ and differentiating $\Delta_c(\varepsilon)\hat{x}_0 = \varepsilon^c \widetilde{L}_c(\Omega)\hat{x}_0 + O(\varepsilon^{c+1})$ with respect to $\theta_a = \varepsilon \Omega_a$ via the chain rule ($\nabla_{\theta_a} = \varepsilon^{-1}\nabla_{\Omega_a}$) yields $\|\mathcal{P}_{\gamma_c}\nabla_\theta \mathcal{L}(w_c, w'_c; \varepsilon)\|_2 = O(\varepsilon^{c-1})$ (and $\Theta(\varepsilon^{c-1})$ whenever $\widetilde{L}_c(\Omega)\hat{x}_0 \ne 0$ and the readout loss gradient is non-orthogonal to $\widetilde{L}_c(\Omega)\hat{x}_0$). For the derived series $G^{(d)} = [G^{(d-1)}, G^{(d-1)}] \subseteq \gamma_{2^d}(G)$, replacing $R_{a_c}$ by an independent derived-series word pair of degree $2^{d-1}$ adds the leading exponents $2^{d-1} + 2^{d-1} = 2^d$, giving $\|\Delta^{(d)}(\varepsilon)\|_F = O(\varepsilon^{2^d})$ and $\|\mathcal{P}_{G^{(d)}}\nabla_\theta \mathcal{L}(\varepsilon)\|_2 = O(\varepsilon^{2^d - 1})$.

\textbf{Part (c): Coset-centroid projection, Fourier--Markov terminal decay, and stationary dense factored symmetry breaking.}
At $\varepsilon = 0$, $R_w(0) = I_d$ and $\hat{x}_w(0) = \hat{x}_0$ for all $w \in \Sigma^*$, so the initial Softmax distribution $p_0(g) = \frac{\exp(\langle w_g, \hat{x}_0 \rangle / \tau)}{\sum_{h \in G} \exp(\langle w_h, \hat{x}_0 \rangle / \tau)}$ and prediction mean $\bar{w}_0 = \sum_{g \in G} p_0(g) w_g$ are independent of $w$, and for a single word $w \in \Sigma^N$, $\nabla_{\theta_a} \mathcal{L}_{\mathrm{CE}}(w)\big|_{\varepsilon=0} = \frac{|w|_a}{\tau} (\bar{w}_0 - w_{\eta(w)}) \wedge \hat{x}_0$. For any abelianization-projected (coset-constant) readout $w_g = \mu_{\mathrm{coset}}(\pi_{\mathrm{ab}}(g))$ (such as the abelianization head $w^{\mathrm{ab}}_{\pi_{\mathrm{ab}}(g)}$ of the factored readout Eq.~\eqref{eq:factored_readout}, or under uniform target coset randomization $g \sim \Unif(\pi_{\mathrm{ab}}(w)[G, G])$), $w_{\eta(w)}$ equals the coset centroid $\mu_{\mathrm{coset}}(\pi_{\mathrm{ab}}(w)) = \frac{1}{|[G, G]|} \sum_{g \in \pi_{\mathrm{ab}}(w)[G, G]} w_g$ (which has $\mathbb{E}\|\mu_{\mathrm{coset}}(q)\|_2^2 = \frac{\sigma_w^2 d}{|[G, G]|}$ under i.i.d.\ $w_g \sim \mathcal{N}(0, \sigma_w^2 I_d)$). Under $w \sim \Unif(\Sigma^N)$, $\mathbb{E}[|w|_a] = \frac{N}{|\Sigma|}$, and conditioned on $a_{i_t} = a$ (probability $\frac{1}{|\Sigma|}$ at each step $t = 1, \dots, N$), the remaining $N-1$ tokens are i.i.d.\ uniform on $\Sigma$, so $\pi_{\mathrm{ab}}(w) \sim \pi_{\mathrm{ab}}(a) \cdot Z_{N-1}$ where $Z_{N-1}$ has law $\mu_{\mathrm{ab}}^{*(N-1)}$ on $G^{\mathrm{ab}}$. Expanding $\mu_{\mathrm{coset}}(q) = \sum_{\chi \in \widehat{G^{\mathrm{ab}}}} \chi(q) \widehat{\mu_{\mathrm{coset}}}(\chi)$ by character Fourier inversion on $G^{\mathrm{ab}}$ gives $\mathbb{E}[\chi(Z_{N-1})] = \hat{\mu}_{\mathrm{ab}}(\chi)^{N-1}$. Separating the principal character $\chi = \mathbf{1}$ ($\hat{\mu}_{\mathrm{ab}}(\mathbf{1}) = 1$, $\widehat{\mu_{\mathrm{coset}}}(\mathbf{1}) = \frac{1}{|G|}\sum_{g \in G} w_g$) from $\chi \ne \mathbf{1}$ and summing over $t = 1, \dots, N$ yields Eq.~\eqref{eq:fourier_markov_decay} directly: the $\chi = \mathbf{1}$ term combines with $\frac{N}{|\Sigma|\tau}\bar{w}_0 \wedge \hat{x}_0$ into the token-independent drift $\frac{N}{|\Sigma|\tau}(\bar{w}_0 - \widehat{\mu_{\mathrm{coset}}}(\mathbf{1})) \wedge \hat{x}_0$ (which vanishes identically whenever $\bar{w}_0 = \widehat{\mu_{\mathrm{coset}}}(\mathbf{1})$ or upon subtracting the token average $\frac{1}{|\Sigma|}\sum_{b \in \Sigma} \nabla_{\theta_b} \mathcal{L}_{\mathrm{CE}}(w)\big|_{\varepsilon=0}$), while subtracting the token average $\frac{1}{|\Sigma|}\sum_{b \in \Sigma} \chi(\pi_{\mathrm{ab}}(b)) = \hat{\mu}_{\mathrm{ab}}(\chi)$ on the $\chi \ne \mathbf{1}$ terms yields the exact token-centered sum $-\frac{N}{|\Sigma|\tau} \sum_{\chi \ne \mathbf{1}} \hat{\mu}_{\mathrm{ab}}(\chi)^{N-1} (\chi(\pi_{\mathrm{ab}}(a)) - \hat{\mu}_{\mathrm{ab}}(\chi)) (\widehat{\mu_{\mathrm{coset}}}(\chi) \wedge \hat{x}_0)$, which decays exponentially as $O(N \rho_{\mathrm{ab}}^{N-1}) \to 0$ whenever $\rho_{\mathrm{ab}} = \max_{\chi \in \widehat{G^{\mathrm{ab}}} : \chi(\pi_{\mathrm{ab}}(\cdot)) \not\equiv \hat{\mu}_{\mathrm{ab}}(\chi)} |\hat{\mu}_{\mathrm{ab}}(\chi)| < 1$ over active token-discriminating characters (since any character constant on $\pi_{\mathrm{ab}}(\Sigma)$ satisfies $\chi(\pi_{\mathrm{ab}}(a)) - \hat{\mu}_{\mathrm{ab}}(\chi) = 0$ for all $a \in \Sigma$). Under normalized discounted dense per-step prefix supervision $\mathcal{L}_{\mathrm{dense}}^{(\beta)}(w) = \frac{1}{Z_N(\beta)} \sum_{t=1}^N \beta^{t-1} \mathcal{L}_{\mathrm{CE}}(w_{1:t})$ with $Z_N(\beta) = \sum_{t=1}^N \beta^{t-1}$, averaging Eq.~\eqref{eq:fourier_markov_decay} over $t = 1, \dots, N$ replaces $N \hat{\mu}_{\mathrm{ab}}(\chi)^{N-1}$ by $\frac{1}{Z_N(\beta)} \sum_{t=1}^N t (\beta \hat{\mu}_{\mathrm{ab}}(\chi))^{t-1}$, which converges as $N \to \infty$ to $(1-\beta)(1 - \beta \hat{\mu}_{\mathrm{ab}}(\chi))^{-2}$ for $\beta \in (0, 1)$, yielding the asymptotic expected token-centered gradient $-\frac{1-\beta}{|\Sigma|\tau}(\Psi_\beta(a; W) \wedge \hat{x}_0)$ with $\Psi_\beta(a; W)$ defined in Eq.~\eqref{eq:resolvent_sum_psi} (and scales as $\Theta(1/N)$ for uniform prefix averaging $\beta = 1$, avoiding the exponential $O(N \rho_{\mathrm{ab}}^{N-1})$ terminal-only decay). While individual non-zero character modes in $\Psi_\beta(a; W)$ can cancel on a measure-zero linear hyperplane in $W$ when $|G^{\mathrm{ab}}| \ge 3$ (e.g., on $G = C_3$ at $\beta = 0.5$ along $w_0 = e_2, w_1 = 22 e_2, w_2 = -23 e_2$), character orthogonality $\frac{1}{|G^{\mathrm{ab}}|}\sum_{q \in G^{\mathrm{ab}}} \overline{\chi(q)}\chi'(q) = \delta_{\chi,\chi'}$ implies that under i.i.d.\ Gaussian initialization $w_g \sim \mathcal{N}(0, \sigma_w^2 I_d)$, the Fourier vectors $\widehat{\mu_{\mathrm{coset}}}(\chi)$ have uncorrelated components with $\mathbb{E}[\widehat{\mu_{\mathrm{coset}}}(\chi)_i \overline{\widehat{\mu_{\mathrm{coset}}}(\chi')_j}] = \frac{\sigma_w^2}{|G|} \delta_{\chi,\chi'} \delta_{i,j}$; since $\|u \wedge \hat{x}_0\|_F^2 = \frac{1}{2}\|P_{\hat{x}_0^\perp} u\|_2^2$ has expectation $\frac{d-1}{2d}\mathbb{E}\|u\|_2^2$, summing over the $d-1$ transverse coordinates gives Eq.~\eqref{eq:resolvent_gaussian_second_moment}, proving that $\Psi_\beta(a; W) \wedge \hat{x}_0 \ne 0$ holds with probability $1$ (and for Lebesgue-almost-every $W$) whenever at least one character $\chi \ne \mathbf{1}$ has $\chi(\pi_{\mathrm{ab}}(a)) \ne \hat{\mu}_{\mathrm{ab}}(\chi)$. Finally, under the factored readout $w_g = w^{\mathrm{ab}}_{\pi_{\mathrm{ab}}(g)} + w^{\mathrm{core}}_g$ with $\sum_{g \in q[G, G]} w^{\mathrm{core}}_g = 0$, the coset centroid equals $\mu_{\mathrm{coset}}(q) = w^{\mathrm{ab}}_q$ identically with undiluted variance $\mathbb{E}\|w^{\mathrm{ab}}_q\|_2^2 = \sigma_w^2 d$ independent of $|[G, G]|$.
\end{proof}

\section{Extended Multi-Stem, 30-Subject Cross-Validation, and Multi-IMU Attribution Tables}
\label{app:supplementary_tables}

Tables~\ref{tab:app_full_stem_hapt_cv} and~\ref{tab:app_full_realdisp} report the complete 23-row single-IMU 6-stem $\times$ capacity matrix (UCI HAR/HAPT-PT), 13-arm 30-subject 5-fold CV ($5\text{ folds} \times 5\text{ seeds} = 25$ fold-seed runs, including \texttt{Ep-358 Vote}, \texttt{Ep-358 C3 Vote}, the 6/10-comparison scalar-stem Holm family $p_{\mathrm{Holm,6/10}}$, and the unified 12-comparison all-baseline Holm family $p_{\mathrm{Holm,12}}$), and 22-arm 3-IMU REALDISP benchmark summarized in Table~\ref{tab:stem_capacity_realdisp}. On the UCI HAPT-PT 30-subject 5-fold CV, comparing our $22.3\mathrm{K}$ \texttt{Cascade (Full-Eq.)} against the $4\times$ larger $90.5\mathrm{K}$ \texttt{GRU (Inv, 90K)} under the unified 12-comparison Holm family yields $p_{\mathrm{Holm,12}} = 0.00415$ on all 358 episodes ($p_{\mathrm{Holm,6}} = 0.0017$; exact $2^{30}$ two-sided paired sign-flip permutation $p_{\mathrm{perm,equal}} = 2.56 \times 10^{-4}$ [$r = 275{,}100 / 2^{30}$] and episode-weighted $p_{\mathrm{perm,weighted}} = 3.88 \times 10^{-4}$ [$r = 416{,}464 / 2^{30}$]; $5/5$ folds positive $[+5.83, +5.04, +5.00, +3.23, +3.75]\,\mathrm{pp}$ with 5-fold paired $t = 9.66, p_{\mathrm{fold},t} = 6.44 \times 10^{-4}$ [$p_{\mathrm{fold,pool},t} = 7.16 \times 10^{-4}$], 5-seed paired $p_{\mathrm{seed},t} = 0.00597$, and $2^5 = 32$ sign-flip resolution floor $p_{\mathrm{fold,perm}} = p_{\mathrm{seed,perm}} = 2/32 = 0.0625$), $p_{\mathrm{Holm,12}} = 0.01428$ on uniform C3 episodes ($p_{\mathrm{Holm,6}} = 0.0057$; exact $2^{30}$ sign-flip $p_{\mathrm{perm,equal}} = 0.00176$ [$r = 1{,}894{,}720 / 2^{30}$] and $p_{\mathrm{perm,weighted}} = 0.00241$ [$r = 2{,}587{,}024 / 2^{30}$]; $5/5$ folds positive $[+5.83, +5.58, +3.33, +2.39, +2.08]\,\mathrm{pp}$ with 5-fold paired $t = 4.88, p_{\mathrm{fold},t} = 0.00818$ [$p_{\mathrm{fold,pool},t} = 0.00955$], 5-seed paired $p_{\mathrm{seed},t} = 0.02046$, and $p_{\mathrm{fold,perm}} = p_{\mathrm{seed,perm}} = 2/32 = 0.0625$), and $p_{\mathrm{Holm,12}} = 0.1122$ on \texttt{Win-1074 C3} ($5/5$ folds positive $[+4.54, +1.72, +1.02, +0.67, +1.30]\,\mathrm{pp}$, $p_{\mathrm{fold},t} = 0.0562, p_{\mathrm{fold,perm}} = 0.0625$; 30-subject raw $p = 0.0561$, exact $2^{30}$ sign-flip $p_{\mathrm{perm,equal}} = 0.0564$ [$r = 60{,}562{,}948 / 2^{30}$] and window-weighted $p_{\mathrm{perm,weighted}} = 0.0624$ [$r = 66{,}953{,}323 / 2^{30}$]; \texttt{Win-1574} $4/5$ folds positive, $p_{\mathrm{fold},t} = 0.2435, p_{\mathrm{fold,perm}} = 0.25$, 30-subject raw $p = 0.2638$, $p_{\mathrm{perm,equal}} = 0.269$, $p_{\mathrm{perm,weighted}} = 0.326$); across the capacity-matched $22\mathrm{K}$ baselines and unnormalized control under \texttt{holm\_12comp\_all\_arms}, 30-subject paired $t$-tests yield $p_{\mathrm{Holm,12}} = 0.0499$ on \texttt{Win-1074 C3}, $0.0298$ on \texttt{Ep-358}, and $0.0390$ on \texttt{Ep-358 C3} vs.\ \texttt{GRU (Inv, 22K)}; $p_{\mathrm{Holm,12}} = 0.0488$ on \texttt{Win-1074 C3}, $0.00391$ on \texttt{Ep-358}, and $0.01480$ on \texttt{Ep-358 C3} vs.\ \texttt{GRU (Lag-Inv, 22K)}; $p_{\mathrm{Holm,12}} = 0.0203$ on \texttt{Win-1074 C3}, $0.00415$ on \texttt{Ep-358}, and $0.01480$ on \texttt{Ep-358 C3} ($0.00415 / 0.01428$ under hard-window-vote on \texttt{oof\_predictions.npz}) vs.\ \texttt{GRU (Chiral, 22K)} (with $0.00028 \le p_{\mathrm{Holm,12}} \le 0.0148$ across all four lagged and chiral scalar-stem baselines on episode voting); and $p_{\mathrm{Holm,8}} = 7.60 \times 10^{-5}, p_{\mathrm{Holm,12}} = 1.37 \times 10^{-4}$ on \texttt{Win-1074 C3} ($0.00415$ on \texttt{Ep-358}, $0.00110$ on \texttt{Ep-358 C3}) vs.\ \texttt{Cascade (Unnorm.)}. Because 5-fold cross-validation models share overlapping $24$-subject training sets across folds~\citep{bengio2004no} and a $K=5$ fold or $S=5$ seed sign-flip test has a discrete two-sided resolution floor of $2/2^5 = 0.0625 > 0.05$, out-of-fold 30-subject paired tests establish conditional superiority across the 30 cohort subjects given the 5-fold training runs rather than unconditional population-level superiority over arbitrary training-set draws, which is why we report 30-subject ($2^{30}$), 5-fold ($K=5, 2^5 = 32$), and 5-seed ($S=5, 2^5 = 32$) tests side-by-side on HAPT-PT alongside external locked-split REALDISP validation (where the $2^{15} = 32{,}768$ sign-flip test summarizes the $3\text{ test subjects} \times 5\text{ seeds} = 15$ evaluation cells).

\begin{table}[H]
\centering
\caption{Complete single-IMU 6-stem $\times$ parameter-capacity Pareto matrix on UCI HAR and UCI HAPT-PT (9-subject split) and full-cohort 30-subject 5-fold subject-disjoint cross-validation across all 13 models (5 seeds $\{42, 123, 456, 789, 1024\}$).}
\label{tab:app_full_stem_hapt_cv}
\footnotesize
\medmuskip=0mu\thinmuskip=0mu\thickmuskip=0mu
\renewcommand{\arraystretch}{0.92}
\setlength{\tabcolsep}{3.2pt}
\begin{tabular}{@{}lrrrrrrrr@{}}
\toprule
\multicolumn{9}{@{}l@{}}{\textbf{Complete Panel A: Single-IMU 6-Stem $\times$ Capacity Matrix on UCI HAR ($K=6$) and UCI HAPT-PT (9-Subj Split)}} \\
\midrule
\textbf{Model \& Stem} & \textbf{Par.} & \textbf{Drift} & \textbf{HAR Cln} & \textbf{HAR R45} & \textbf{HAR Unif} & \textbf{HAPT Cln} & \textbf{HAPT R45} & \textbf{HAPT Unif} \\
\midrule
Mamba (Nat., 22K) & 22,107 & $5.1\text{e-}2$ & $90.34 \pm 0.34$ & $58.91 \pm 3.17$ & $22.74 \pm 1.04$ & $81.42 \pm 1.61$ & $53.09 \pm 1.90$ & $19.01 \pm 2.62$ \\
GRU (Nat., 22K) & 22,182 & $6.5\text{e-}1$ & $92.51 \pm 1.15$ & $66.12 \pm 2.67$ & $25.78 \pm 1.30$ & $82.53 \pm 1.75$ & $58.46 \pm 2.33$ & $20.06 \pm 1.59$ \\
Mamba (Nat., 90K) & 89,814 & $8.7\text{e-}2$ & $90.95 \pm 0.96$ & $61.10 \pm 3.25$ & $24.30 \pm 1.99$ & $80.19 \pm 1.42$ & $56.11 \pm 2.52$ & $21.42 \pm 1.04$ \\
Transf.\ (Nat., 90K) & 90,538 & $6.8\text{e-}1$ & $93.24 \pm 0.65$ & $64.05 \pm 2.81$ & $23.57 \pm 1.52$ & $81.91 \pm 0.83$ & $61.30 \pm 3.69$ & $19.81 \pm 3.00$ \\
GRU (Nat., 90K) & 90,304 & $3.8\text{e-}1$ & $92.48 \pm 0.86$ & $62.84 \pm 2.32$ & $24.24 \pm 1.45$ & $84.01 \pm 1.33$ & $61.11 \pm 3.58$ & $20.12 \pm 1.90$ \\
Diag-SSM (Eq-Cat) & 22,950 & $8.1\text{e-}2$ & $94.27 \pm 0.66$ & $70.58 \pm 2.40$ & $30.46 \pm 1.27$ & $82.47 \pm 1.44$ & $67.16 \pm 1.95$ & $24.32 \pm 2.86$ \\
GRU (Eq-Cat) & 22,125 & $3.9\text{e-}1$ & $94.78 \pm 1.04$ & $72.83 \pm 1.86$ & $28.70 \pm 2.00$ & $84.01 \pm 1.82$ & $67.84 \pm 2.13$ & $25.00 \pm 2.30$ \\
Diag-SSM (Inv, 22K) & 22,044 & $6.0\text{e-}8$ & $90.58 \pm 0.64$ & $91.46 \pm 0.52$ & $90.58 \pm 0.64$ & $81.91 \pm 0.94$ & $82.04 \pm 0.91$ & $81.91 \pm 0.94$ \\
Mamba (Inv, 22K) & 22,360 & $1.5\text{e-}7$ & $88.94 \pm 0.92$ & $88.50 \pm 1.95$ & $88.94 \pm 0.92$ & $79.63 \pm 3.17$ & $79.51 \pm 3.25$ & $79.63 \pm 3.17$ \\
Transf.\ (Inv) & 22,368 & $2.7\text{e-}7$ & $89.03 \pm 0.99$ & $89.30 \pm 0.40$ & $89.03 \pm 0.99$ & $76.11 \pm 2.00$ & $75.80 \pm 2.43$ & $76.11 \pm 2.00$ \\
TCN (Inv) & 22,172 & $2.4\text{e-}7$ & $85.25 \pm 1.58$ & $85.10 \pm 1.02$ & $85.25 \pm 1.58$ & $59.57 \pm 2.19$ & $59.44 \pm 2.23$ & $59.57 \pm 2.19$ \\
GRU (Inv, 22K) & 22,368 & $2.7\text{e-}7$ & $90.78 \pm 1.22$ & $90.81 \pm 0.41$ & $90.78 \pm 1.22$ & $81.23 \pm 2.05$ & $81.23 \pm 2.05$ & $81.23 \pm 2.05$ \\
GRU (Inv, 90K) & 90,546 & $1.8\text{e-}7$ & $92.59 \pm 1.19$ & $92.58 \pm 0.64$ & $92.59 \pm 1.19$ & $82.65 \pm 1.93$ & $82.65 \pm 1.93$ & $82.65 \pm 1.93$ \\
Diag-SSM (Lag-Inv) & 22,048 & $2.6\text{e-}6$ & --- & --- & --- & $79.44 \pm 0.99$ & $79.51 \pm 1.10$ & $79.44 \pm 0.99$ \\
GRU (Lag-Inv) & 22,400 & $1.7\text{e-}5$ & --- & --- & --- & $81.79 \pm 1.53$ & $81.67 \pm 1.36$ & $81.79 \pm 1.53$ \\
Diag-SSM (Chiral) & 22,089 & $2.5\text{e-}6$ & --- & --- & --- & $80.56 \pm 1.15$ & $80.86 \pm 1.29$ & $80.56 \pm 1.15$ \\
GRU (Chiral) & 22,354 & $2.8\text{e-}5$ & --- & --- & --- & $83.27 \pm 1.44$ & $83.15 \pm 1.77$ & $83.27 \pm 1.44$ \\
Lie-$\SO(3)$-1L & 22,167 & $1.0\text{e-}7$ & $84.76 \pm 2.56$ & $84.47 \pm 2.61$ & $84.76 \pm 2.56$ & $84.75 \pm 1.70$ & $84.94 \pm 1.68$ & $84.75 \pm 1.70$ \\
Cascade (Static) & 22,280 & $2.4\text{e-}1$ & $95.48 \pm 0.48$ & $82.56 \pm 0.87$ & $61.48 \pm 1.58$ & $87.04 \pm 0.76$ & $82.47 \pm 2.10$ & $59.32 \pm 7.30$ \\
Cascade (Eq.+Nom.) & 22,280 & $6.7\text{e-}2$ & $93.82 \pm 0.38$ & $82.04 \pm 1.15$ & $63.07 \pm 1.08$ & $87.04 \pm 1.95$ & $86.48 \pm 0.74$ & $78.83 \pm 1.34$ \\
\textbf{Cascade (Full-Eq.)} & \textbf{22,280} & $\mathbf{1.9\text{e-}7}$ & $84.22 \pm 2.96$ & $84.00 \pm 2.54$ & $84.22 \pm 2.96$ & $\mathbf{86.23 \pm 0.64}$ & $\mathbf{86.05 \pm 0.34}$ & $\mathbf{86.23 \pm 0.64}$ \\
Prosp.\ $17{+}4$ (Cln) & 22,280 & --- & $94.19 \pm 0.94$ & $82.25 \pm 1.19$ & $62.91 \pm 1.12$ & $86.73 \pm 0.90$ & $83.27 \pm 2.96$ & $63.77 \pm 11.44$ \\
Prosp.\ $17{+}4$ (OA) & 22,280 & --- & $92.91 \pm 2.95$ & $82.99 \pm 1.76$ & $68.23 \pm 11.00$ & $86.23 \pm 0.64$ & $86.05 \pm 0.34$ & $86.23 \pm 0.64$ \\
\midrule
\multicolumn{9}{@{}l@{}}{\textbf{Complete 13-Model 30-Subject 5-Fold Cross-Validation on UCI HAPT-PT ($N_{\mathrm{subj}}=30, N_{\mathrm{event}}=358, N_{\mathrm{win}}=1{,}574/1{,}074$)}} \\
\midrule
\textbf{30-Subj CV Model} & \textbf{Par.} & \textbf{Subj} & \textbf{Win-1574} & \textbf{Win-C3} & \textbf{Ep-Vote} & \textbf{Ep-C3} & \textbf{Subj C3 CI} & \textbf{Holm $p$} \\
\midrule
TCN (Inv) & 22,172 & $62.85$ & $62.26 \pm 1.35$ & $63.02 \pm 1.82$ & $70.11 \pm 0.34$ & $67.65 \pm 0.78$ & $[+17.07,+24.88]$ & $1.8\text{e-}10/5.4\text{e-}21$ \\
Transf.\ (Inv) & 22,368 & $75.77$ & $75.50 \pm 0.88$ & $75.85 \pm 1.49$ & $80.06 \pm 1.32$ & $78.77 \pm 1.71$ & $[+6.19,+10.00]$ & $2.2\text{e-}8/3.4\text{e-}9$ \\
Mamba (Inv) & 22,360 & $77.08$ & $76.62 \pm 1.78$ & $77.11 \pm 1.83$ & $81.40 \pm 1.92$ & $81.06 \pm 1.91$ & $[+4.80,+8.91]$ & $2.2\text{e-}6/4.0\text{e-}8$ \\
Cascade (Unnorm.) & 22,280 & $78.83$ & $77.89 \pm 1.81$ & $78.83 \pm 1.69$ & $85.25 \pm 1.74$ & $83.63 \pm 1.62$ & $[+3.27,+7.08]$ & $7.6\text{e-}5/0.0019$ \\
Diag-SSM (Inv) & 22,044 & $80.16$ & $79.57 \pm 1.14$ & $80.20 \pm 1.32$ & $82.79 \pm 1.26$ & $82.12 \pm 1.46$ & $[+1.68,+6.07]$ & $0.0073/1.8\text{e-}6$ \\
GRU (Inv, 22K) & 22,368 & $80.59$ & $80.17 \pm 0.82$ & $80.58 \pm 0.64$ & $83.91 \pm 1.45$ & $82.91 \pm 1.62$ & $[+0.93,+6.04]$ & $0.0333/3.0\text{e-}4$ \\
GRU (Inv, 90K) & 90,546 & $82.08$ & $81.86 \pm 1.10$ & $82.14 \pm 1.27$ & $83.91 \pm 1.21$ & $83.58 \pm 1.16$ & $[+0.07,+3.68]$ & $0.0561/3.0\text{e-}4$ \\
Diag-SSM (Lag-Inv) & 22,048 & $79.64$ & $78.88 \pm 1.14$ & $79.63 \pm 1.07$ & $83.13 \pm 1.62$ & $81.96 \pm 1.78$ & $[+2.50,+6.21]$ & $7.5\text{e-}4/4.5\text{e-}6$ \\
GRU (Lag-Inv) & 22,400 & $81.17$ & $80.39 \pm 0.65$ & $81.19 \pm 0.37$ & $83.13 \pm 0.73$ & $83.13 \pm 0.70$ & $[+0.78,+4.83]$ & $0.0366/3.3\text{e-}5$ \\
Diag-SSM (Chiral) & 22,089 & $80.46$ & $79.92 \pm 1.00$ & $80.50 \pm 0.87$ & $83.02 \pm 1.61$ & $82.40 \pm 1.38$ & $[+1.59,+5.50]$ & $0.0094/2.2\text{e-}6$ \\
GRU (Chiral) & 22,354 & $80.36$ & $79.77 \pm 0.62$ & $80.43 \pm 1.03$ & $83.02 \pm 0.91$ & $82.46 \pm 0.85$ & $[+1.33,+5.77]$ & $0.0163/3.3\text{e-}5$ \\
Lie-$\SO(3)$-1L & 22,167 & $83.33$ & $82.24 \pm 1.02$ & $83.33 \pm 1.23$ & $87.88 \pm 1.38$ & $86.65 \pm 1.18$ & $[-0.66,+1.94]$ & $0.386/0.392$ \\
\textbf{Cascade (Full-Eq.)} & \textbf{22,280} & $\mathbf{83.93}$ & $\mathbf{82.69 \pm 1.11}$ & $\mathbf{83.97 \pm 1.20}$ & $\mathbf{88.44 \pm 0.83}$ & $\mathbf{87.37 \pm 1.21}$ & Reference & Reference \\
\bottomrule
\end{tabular}
\end{table}

\begin{table}[H]
\centering
\caption{Complete 22-model multi-IMU sensor displacement benchmark on UCI REALDISP ($S=3$ IMUs, $18$ channels, $K=33$ activities, $11/3/3$ subject split with validation subjects $\#10, 14, 16$ and locked held-out test subjects $\#2, 5, 15$) and per-subject ($\#2, \#5, \#15$) Ideal / Self / Mutual accuracy breakdown across 5 seeds.}
\label{tab:app_full_realdisp}
\footnotesize
\renewcommand{\arraystretch}{0.92}
\setlength{\tabcolsep}{3.5pt}
\begin{tabular}{@{}lrrrrrr@{}}
\toprule
\multicolumn{7}{l}{\textbf{Complete Panel B: 3-IMU Sensor Displacement on UCI REALDISP ($K=33$, Held-Out Subjects $\#2, 5, 15$)}} \\
\midrule
\multirow{2}{*}{\textbf{3-IMU Architecture \& Stem ($K=33$)}} & \multirow{2}{*}{\textbf{Params}} & \multicolumn{2}{c}{\textbf{Ideal Test Acc (\%)}} & \multicolumn{3}{c}{\textbf{Zero-Shot Displacement Acc (\%)}} \\
\cmidrule(lr){3-4} \cmidrule(lr){5-7}
 & & \textbf{$T=128$} & \textbf{$T=64$} & \textbf{Self} & \textbf{Mutual} & \textbf{Rot45} \\
\midrule
Selective-Mamba (Native-Conv1D) & 25,289 & $95.88 \pm 0.51$ & $95.61 \pm 0.84$ & $72.27 \pm 2.93$ & $19.91 \pm 1.70$ & $51.57 \pm 2.56$ \\
Causal-Transformer (Native-Conv1D) & 25,393 & $94.86 \pm 0.11$ & $94.25 \pm 0.55$ & $73.49 \pm 2.55$ & $19.45 \pm 2.30$ & $53.93 \pm 1.85$ \\
Dilated-TCN (Native-Conv1D) & 25,261 & $94.34 \pm 0.47$ & $94.45 \pm 0.52$ & $72.25 \pm 1.53$ & $17.72 \pm 1.68$ & $50.22 \pm 0.99$ \\
DeepConvGRU (Native-Conv1D) & 25,301 & $95.58 \pm 0.82$ & $95.47 \pm 0.67$ & $76.06 \pm 1.77$ & $18.54 \pm 1.86$ & $54.50 \pm 2.07$ \\
Diag-SSM (Equiv-Concat) & 25,517 & $96.23 \pm 0.72$ & $96.34 \pm 0.81$ & $82.11 \pm 0.88$ & $43.86 \pm 4.03$ & $77.47 \pm 1.17$ \\
Causal-Transformer (Equiv-Concat) & 25,177 & $96.71 \pm 0.97$ & $95.87 \pm 1.01$ & $86.26 \pm 0.83$ & $46.75 \pm 3.96$ & $83.26 \pm 2.41$ \\
DeepConvGRU (Equiv-Concat) & 25,345 & $97.13 \pm 0.58$ & $96.57 \pm 0.39$ & $84.24 \pm 1.46$ & $42.44 \pm 3.98$ & $81.10 \pm 1.28$ \\
Diag-SSM (Inv-Scalar) & 25,251 & $93.98 \pm 0.32$ & $93.83 \pm 0.26$ & $94.37 \pm 1.19$ & $83.78 \pm 1.71$ & $93.98 \pm 0.33$ \\
Selective-Mamba (Inv-Scalar) & 25,259 & $93.77 \pm 0.42$ & $93.07 \pm 0.50$ & $93.78 \pm 1.05$ & $84.26 \pm 0.88$ & $93.74 \pm 0.44$ \\
Causal-Transformer (Inv-Scalar) & 25,387 & $94.17 \pm 0.85$ & $93.36 \pm 0.45$ & $93.78 \pm 0.65$ & $84.48 \pm 0.52$ & $94.19 \pm 0.85$ \\
Dilated-TCN (Inv-Scalar) & 25,255 & $92.24 \pm 0.61$ & $92.32 \pm 0.63$ & $91.56 \pm 1.58$ & $80.25 \pm 1.60$ & $92.21 \pm 0.60$ \\
DeepConvGRU (Inv-Scalar) & 25,295 & $94.80 \pm 1.33$ & $94.12 \pm 1.14$ & $94.38 \pm 1.06$ & $85.44 \pm 1.93$ & $94.76 \pm 1.33$ \\
Diag-SSM (Lagged-Inv-Scalar) & 25,347 & $95.29 \pm 0.63$ & $95.05 \pm 0.52$ & $95.24 \pm 1.63$ & $85.73 \pm 0.91$ & $95.32 \pm 0.58$ \\
DeepConvGRU (Lagged-Inv-Scalar) & 25,587 & $96.25 \pm 0.64$ & $95.64 \pm 0.50$ & $95.32 \pm 0.85$ & $87.51 \pm 0.79$ & $96.24 \pm 0.60$ \\
Diag-SSM (Chiral-Triple-Inv) & 25,309 & $96.65 \pm 0.26$ & $96.28 \pm 0.50$ & $93.87 \pm 0.66$ & $85.72 \pm 1.01$ & $96.63 \pm 0.23$ \\
DeepConvGRU (Chiral-Triple-Inv) & 25,377 & $96.68 \pm 0.54$ & $95.89 \pm 0.84$ & $95.46 \pm 0.92$ & $86.67 \pm 1.17$ & $96.70 \pm 0.59$ \\
Lie-$\SO(3)$-1L (Equiv-Frame) & 25,049 & $93.45 \pm 1.06$ & $93.02 \pm 1.19$ & $93.38 \pm 0.84$ & $84.13 \pm 1.28$ & $93.41 \pm 1.06$ \\
Lie-$\SO(3)$-Cascade (Dir-RMS, Static) & 25,173 & $93.59 \pm 0.86$ & $93.12 \pm 0.71$ & $91.09 \pm 1.26$ & $79.02 \pm 1.48$ & $92.61 \pm 0.82$ \\
Lie-$\SO(3)$-Cascade (Unnorm.\ $\rho_1$) & 25,173 & $90.62 \pm 1.11$ & $90.18 \pm 0.95$ & $89.65 \pm 0.68$ & $79.13 \pm 1.07$ & $90.57 \pm 1.14$ \\
\textbf{Lie-$\SO(3)$-Cascade (Dir-RMS, Full-Equiv)} & 25,173 & $93.15 \pm 0.74$ & $92.59 \pm 0.64$ & $92.92 \pm 0.76$ & $\mathbf{83.37 \pm 1.00}$ & $93.28 \pm 0.84$ \\
Lie-$\SO(3)$-1L (Equiv + Causal-LP + LN) & 25,219 & $93.58 \pm 0.52$ & --- & $93.62 \pm 0.92$ & $84.61 \pm 1.17$ & $93.52 \pm 0.52$ \\
\textbf{Lie-$\SO(3)$-Cascade (Q-Turn + LP + LN)} & 25,379 & $94.09 \pm 0.53$ & --- & $93.33 \pm 1.03$ & $\mathbf{84.23 \pm 1.00}$ & $94.01 \pm 0.51$ \\
\midrule
\multicolumn{7}{l}{\textbf{Per-Subject Held-Out Test Breakdown on UCI REALDISP (Subjects $\#2, \#5, \#15$, $\text{mean} \pm \text{SD}$ over 5 seeds)}} \\
\midrule
\textbf{Architecture} & \textbf{$\#2$ Ideal} & \textbf{$\#5$ Ideal} & \textbf{$\#15$ Ideal} & \textbf{$\#2$ Mut.} & \textbf{$\#5$ Mut.} & \textbf{$\#15$ Mut.} \\
\midrule
Selective-Mamba (Native-Conv1D) & $99.55 \pm 0.69$ & $92.67 \pm 1.44$ & $95.38 \pm 1.51$ & $22.98 \pm 3.35$ & $25.58 \pm 2.60$ & $7.74 \pm 1.34$ \\
DeepConvGRU (Native-Conv1D) & $99.82 \pm 0.33$ & $94.15 \pm 0.95$ & $92.12 \pm 2.02$ & $21.35 \pm 2.45$ & $25.42 \pm 2.67$ & $5.21 \pm 2.40$ \\
DeepConvGRU (Equiv-Concat) & $100.00 \pm 0.00$ & $96.11 \pm 0.46$ & $94.86 \pm 2.49$ & $49.81 \pm 7.28$ & $48.45 \pm 7.63$ & $23.16 \pm 8.30$ \\
Diag-SSM (Inv-Scalar) & $99.30 \pm 0.08$ & $90.92 \pm 0.97$ & $91.26 \pm 0.96$ & $93.50 \pm 1.56$ & $85.24 \pm 3.53$ & $66.79 \pm 1.38$ \\
DeepConvGRU (Inv-Scalar) & $99.39 \pm 0.30$ & $94.01 \pm 1.90$ & $90.13 \pm 2.35$ & $94.27 \pm 1.74$ & $87.55 \pm 1.40$ & $68.99 \pm 3.47$ \\
Lie-$\SO(3)$-1L (Equiv-Frame) & $99.15 \pm 0.25$ & $90.92 \pm 1.59$ & $89.61 \pm 2.32$ & $93.19 \pm 1.28$ & $87.08 \pm 1.77$ & $66.21 \pm 1.65$ \\
Lie-$\SO(3)$-Cascade (Dir-RMS, Static) & $98.73 \pm 1.31$ & $92.11 \pm 1.81$ & $89.12 \pm 1.74$ & $89.46 \pm 2.45$ & $79.99 \pm 3.30$ & $61.54 \pm 1.83$ \\
Lie-$\SO(3)$-Cascade (Unnorm.\ $\rho_1$) & $97.45 \pm 1.29$ & $87.72 \pm 1.85$ & $85.82 \pm 2.30$ & $89.32 \pm 2.19$ & $80.29 \pm 1.33$ & $61.79 \pm 0.65$ \\
\textbf{Lie-$\SO(3)$-Cascade (Dir-RMS, Full-Equiv)} & $98.73 \pm 0.28$ & $90.86 \pm 1.37$ & $89.16 \pm 2.88$ & $92.53 \pm 0.55$ & $85.72 \pm 2.02$ & $66.10 \pm 2.06$ \\
\bottomrule
\end{tabular}
\end{table}

\setlength{\bibsep}{0.0pt}\bibliographystyle{elsarticle-num}
\bibliography{references}

\endgroup

\end{document}